\documentclass{article}

\usepackage{microtype}
\usepackage{graphicx}
\usepackage{subfigure}
\usepackage{booktabs} %

\usepackage{hyperref}

\usepackage[accepted]{icml2021_style/icml2021}

\usepackage[dvipsnames]{xcolor}

\usepackage{tikz}
\usetikzlibrary{calc,patterns,decorations.pathreplacing,shapes.multipart}

\usepackage{amsmath, latexsym}
\usepackage{amsfonts}
\usepackage{mathtools}
\usepackage{ntheorem}
\usepackage{dsfont}
\usepackage[dvipsnames]{xcolor}
\usepackage[colorinlistoftodos]{todonotes}
\usepackage{booktabs}
\usepackage{xfrac}
\usepackage{bbm}

\usepackage[most]{tcolorbox}
\usepackage{xparse}
\usepackage{lipsum}
\usepackage{changepage}
\usepackage{enumitem}
\usepackage{pifont}

\newtheorem{theorem}{Theorem}[section]
\newtheorem{lemma}[theorem]{Lemma}
\newtheorem{remark}{Remark}[section]
\newenvironment{proof}{\paragraph{Proof:}}{\hfill$\square$}
\newtheorem{theo}{Theorem}[section]
\newtheorem{definition}{Definition}[section]

\newtheorem{proposition}{Proposition}[section]

\newcommand{\xmark}{\ding{55}}

\newtheorem{innercustomthm}{Theorem}
\newenvironment{customthm}[1]
  {\renewcommand\theinnercustomthm{#1}\innercustomthm}
  {\endinnercustomthm}

\newtheorem{innercustomcor}{Corollary}
\newenvironment{customcor}[1]
  {\renewcommand\theinnercustomcor{#1}\innercustomcor}
  {\endinnercustomcor}

\newcommand{\tph}{\tau}

\newcommand{\qedwhite}{\hfill \ensuremath{\Box}}

\newcommand\cut[1]{}

\usepackage[percent]{overpic}
\newcommand{\norm}[1]{\left\lVert#1\right\rVert}

\newcommand{\squishlist}{
   \begin{list}{$\bullet$}
    { \setlength{\itemsep}{0pt}      \setlength{\parsep}{3pt}
      \setlength{\topsep}{3pt}       \setlength{\partopsep}{0pt}
      \setlength{\leftmargin}{1.5em} \setlength{\labelwidth}{1em}
      \setlength{\labelsep}{0.5em} } }

\newcommand{\squishlisttwo}{
   \begin{list}{$\bullet$}
    { \setlength{\itemsep}{0pt}    \setlength{\parsep}{0pt}
      \setlength{\topsep}{0pt}     \setlength{\partopsep}{0pt}
      \setlength{\leftmargin}{2em} \setlength{\labelwidth}{1.5em}
      \setlength{\labelsep}{0.5em} } }

\newcommand{\squishend}{
    \end{list}  }

\newtheorem{cor}{Corollary}[section]

\DeclareMathOperator{\Tr}{Tr}

\newcommand{\myvec}[1]{\mathbf{#1}}

\newcommand{\myvecsym}[1]{\boldsymbol{#1}}

\newcommand{\vphi}{\myvecsym{\phi}}

\newcommand{\vtheta}{\myvecsym{\theta}}

\newcommand{\vxi}{\myvecsym{\xi}}

\newcommand{\vx}{\myvec{x}}

\newcommand{\vz}{\myvec{z}}

\newcommand{\be}{\begin{equation}}
\newcommand{\ee}{\end{equation}}
\newcommand{\bea}{\begin{eqnarray}}
\newcommand{\eea}{\end{eqnarray}}
\newcommand{\beaa}{\begin{eqnarray*}}
\newcommand{\eeaa}{\end{eqnarray*}}

\DeclareMathAlphabet{\mathpzc}{OT1}{pzc}{m}{n}

\graphicspath{{../figures/}}

\icmltitlerunning{Discretization Drift in Two-Player Games}

\begin{document}

\twocolumn[
\icmltitle{Discretization Drift in Two-Player Games}

\begin{icmlauthorlist}
\icmlauthor{Mihaela Rosca}{dm,ucl}
\icmlauthor{Yan Wu}{dm}
\icmlauthor{Benoit Dherin}{google}
\icmlauthor{David G.T. Barrett}{dm}
\end{icmlauthorlist}

\icmlaffiliation{dm}{DeepMind, London, UK}
\icmlaffiliation{ucl}{Center for Artificial Intelligence, University College London}
\icmlaffiliation{google}{Google, Dublin, Ireland}

\icmlcorrespondingauthor{Mihaela Rosca}{mihaelacr@deepmind.com}

\icmlkeywords{Discretization Drift, Optimization, GANs, ICML}

\vskip 0.3in
]

\printAffiliationsAndNotice{\icmlEqualContribution} %

\addtocontents{toc}{\protect\setcounter{tocdepth}{0}}

\begin{abstract}
Gradient-based methods for two-player games produce rich dynamics that can solve challenging problems, yet can be difficult to stabilize and understand. Part of this complexity originates from the discrete update steps given by simultaneous or alternating gradient descent, which causes each player to drift away from the continuous gradient flow
 ---  a phenomenon we call discretization drift. Using backward error analysis, we derive modified continuous dynamical systems that closely follow the discrete dynamics. These modified dynamics provide an insight into the notorious challenges associated with zero-sum games, including Generative Adversarial Networks. In particular, we identify distinct components of the discretization drift that can alter performance and in some cases destabilize the game. Finally, quantifying discretization drift allows us to identify regularizers that explicitly cancel harmful forms of drift or strengthen beneficial forms of drift, and thus improve performance of GAN training.
\end{abstract}

\section{Introduction}
\label{introduction}

The fusion of deep learning with two-player games has produced a wealth of breakthroughs in recent years, from Generative Adversarial Networks (GANs)~\cite{goodfellow2014generative} through to model-based reinforcement learning~\cite{sutton2018reinforcement,rajeswaran2020game}. Gradient descent methods are widely used across these settings, partly because these algorithms scale well to high-dimensional models and large datasets. However, much of the recent progress in our theoretical understanding of two-player differentiable games builds upon the analysis of continuous differential equations that model the dynamics of training~\cite{singh2000nash,heusel2017gans,nagarajan2017gradient}, leading to a gap between theory and practice.
Our aim is to take a step forward in our understanding of two player games by finding continuous systems which better match the gradient descent updates used in practice.

Our work builds upon~\citet{igr}, who use backward error analysis to quantify the discretization drift induced by using gradient descent in supervised learning. We extend their work and use backward error analysis to understand the impact of discretization in the training dynamics of two-player games.
More specifically, we quantify the \textit{Discretization Drift} (DD), the difference between the solutions of the original ODEs defining the game and the discrete steps of the numerical scheme used to approximate them. To do so, we construct modified continuous systems that closely follow the discrete updates.
While in supervised learning DD has a beneficial regularization effect~\citep{igr}, we find that the interaction between players in DD can have a destabilizing effect in adversarial games.

\textbf{Contributions}: Our primary contribution, Theorems \ref{thm:2_players_sim} and \ref{thm:alt}, provides the continuous modified systems which quantify the discretization drift in simultaneous and alternating gradient descent for general two-player differentiable games. Both theorems are novel in their scope and generality, as well as their application toward understanding the effect of the discretization drift in two-player games parametrized with neural networks. Theorems \ref{thm:2_players_sim} and \ref{thm:alt} allow us to use dynamical system analysis to describe GD without ignoring discretization drift, which we then use to:
\begin{itemize}
\item Provide new stability analysis tools (Section \ref{section:stability}). %
\item Motivate explicit regularization methods which drastically improve the performance of simultaneous gradient descent in GAN training (Section \ref{section:explicit_regularization}).
\item Pinpoint optimal regularization coefficients and shed new light on existing explicit regularizers (Table \ref{tab:methods_comp}).
\item Pinpoint the best performing learning rate ratios for alternating updates (Sections \ref{sec:common_payoff} and \ref{section:learning_rate_ratios}).
\item Explain previously observed but unexplained phenomena such as the difference in performance and stability between alternating and simultaneous updates in GAN training (Section~\ref{section:zero_sum}).
\end{itemize}

\section{Background}

{\bf Backward Error Analysis:} Backward error analysis was devised to study the long-term error incurred by following an ODE numerical solver instead of an exact ODE solution \cite{backward_lifespan,GNI}.
The general idea is to find a modified version of the original ODE that follows the steps of the numerical solver exactly.
Recently, ~\citet{igr} used this technique to uncover a form of DD, called \emph{implicit gradient regularization}, arising in supervised learning for models trained with gradient descent. They showed that for a model with parameters  $\vtheta$ and loss $L(\vtheta)$ optimized with gradient descent $\vtheta_t = \vtheta_{t-1} - h\nabla_{\vtheta} L(\vtheta)$, the first order modified equation is $\dot{\vtheta} =-\nabla_{\vtheta} \tilde L(\vtheta)$, with modified loss
\begin{align}
\tilde L(\vtheta) = L(\vtheta) + \frac h4 \norm{\nabla_{\vtheta} L(\vtheta)}^2
\label{eq:sup_learning_igr}
\end{align}
This shows that there is a hidden implicit regularization effect, dependent on learning rate $h$ that biases learning toward flat regions, where test errors are typically smaller~\citep{igr}.

{\bf Two-player games:} A well developed strategy for understanding two-player games in gradient-based learning is to analyze the continuous dynamics of the game~\cite{singh2000nash,heusel2017gans}. Tools from dynamical systems theory have been used to explain convergence behaviors and to improve training using modifications of learning objectives or learning rules~\citep{nagarajan2017gradient,balduzzi2018mechanics,mazumdar2019finding}. Many of the insights from continuous two-player games apply to games that are trained with discrete updates. However, discrepancies are often observed between the continuous analysis and discrete training~\citep{mescheder2018training}. In this work, we extend backward error analysis from the supervised learning setting -- a special case of a one-player game -- to the more general two-player game setting, thereby bridging this gap between the continuous systems that are often analyzed in theory and the discrete numerical methods used in practice.

\section{Discretization drift}
\label{sec:general_theorems}

Throughout the paper we will denote by $\vphi \in \mathbb{R}^m$ and $\vtheta \in \mathbb{R}^n$ the row vectors representing the parameters of the first and second player, respectively.
The players update functions will be denoted correspondingly by $f(\vphi, \vtheta): \mathbb{R}^m \times \mathbb{R}^n \rightarrow \mathbb{R}^m$ and by
$g(\vphi, \vtheta): \mathbb{R}^m \times \mathbb{R}^n \rightarrow \mathbb{R}^n$.
In this setting, the partial derivative $\nabla_{\vtheta} f(\vphi, \vtheta)$ is the $n\times m$ matrix
$\left(\partial_{\theta_i}f_j\right)_{ij}$ with $i = 1, \dots, n$ and  $j = 1, \dots, m$.

We aim to understand the impact of discretizing the ODEs
\begin{align}
 \dot{\vphi} &=  f( \vphi, \vtheta)   \label{eq:ode1}, \\
 \dot{\vtheta}  &= g( \vphi, \vtheta) \label{eq:ode2},
\end{align}
using either simultaneous  or alternating Euler updates. We derive a modified continuous system of the form:
\begin{align}
 \dot{\vphi} &=  f( \vphi, \vtheta)  + h f_1( \vphi, \vtheta) \\
 \dot{\vtheta}  &= g( \vphi, \vtheta) + h g_1( \vphi, \vtheta)
\end{align}
that closely follows the discrete Euler update steps; the local error between a discrete update and the modified system is of order $\mathcal O(h^3)$ instead of $\mathcal O(h^2)$ as is the case for the original continuous system given by Equations \eqref{eq:ode1} and \eqref{eq:ode2}.
If we can neglect errors of order $\mathcal O(h^3)$, the terms $f_1$ and $g_1$ above characterize the DD of the discrete scheme, which can be used to help us understand the impact of DD. For example, it allows us to compare the use of simultaneous and alternating Euler updates, by comparing the dynamics of their associated modified systems as characterized by the DD terms $f_1$ and $g_1$.

We can specialize these modified equations using $f = -\nabla_\phi L_1$ and $g = -\nabla_\theta L_2$, where $L_1$ and $L_2$ are the loss functions for the two players. We will use this setting later to investigate the modified dynamics of \textit{simultaneous} or \textit{alternating gradient descent}. We can further specialize the form of these updates for common-payoff games ($L_1 = L_2  = E$) and zero-sum games ($L_1 = -E$, $L_2 = E$).

\begin{figure}[t]
\begin{tikzpicture}[every text node part/.style={align=center,inner sep=0,outer sep=0}][overlay]
\coordinate (theta_t_minus_1) at (0,0);
\coordinate (theta_t) at (2.2,-1);
\coordinate (theta_t_plus_1) at (4.4, -0.5);

\node(draw) at ($(theta_t_minus_1) + (+0.3,-0.32)$) {$\theta_{t-1}$};
\node(draw) at ($(theta_t) + (-0.2,-0.2)$) {$\theta_{t}$};
\node(draw) at ($(theta_t_plus_1) + (-0.05,-0.35)$) {$\theta_{t+1}$};

\coordinate (first_time_transition) at ($(theta_t_minus_1) + (+0.55,1.6)$);
\coordinate (second_time_transition) at ($(first_time_transition) + (3,0)$);

\node(draw) at (first_time_transition) {$t -1 \rightarrow t $};
\node(draw) at (second_time_transition) {$t \rightarrow t + 1$};

\coordinate (cont_theta_t) at (2.2,0.8);
\coordinate (cont_theta_t_plus_1) at (4.4, 1.05);

\coordinate (mod_cont_theta_t) at (2.2, -0.6);
\coordinate (mod_cont_theta_t_plus_1) at (4.4, -0.2);

\draw [NavyBlue,thick,dashed] (theta_t_minus_1) -- (theta_t);
\draw [NavyBlue,thick,dashed] (theta_t) -- (theta_t_plus_1);

\draw [OliveGreen,thick]  (theta_t_minus_1) to[out=50,in=180] node[midway,above] {$\dot{\theta}$} (cont_theta_t);
\draw [OliveGreen,thick]  (theta_t) to[out=50,in=180] node[midway,above]{$\dot{\theta}$} (cont_theta_t_plus_1);

\draw [Plum,thick]  (theta_t_minus_1) to[out=50,in=180]  node[near end,above] {$\dot{\tilde{\theta}}$} (mod_cont_theta_t);
\draw [Plum,thick]  (theta_t) to[out=30,in=170] node[near end,above] {$\dot{\tilde{\theta}}$} (mod_cont_theta_t_plus_1);

\draw [black,thick,dashed] ($(theta_t) + (0,-0.5)$) -- ($(cont_theta_t) + (0,0.7)$);

\draw [
    thick,
    decoration={
        brace,
        mirror,
        raise=0.1cm
    },
    decorate
] (theta_t_plus_1) -- (mod_cont_theta_t_plus_1)
node [pos=0.5,anchor=west,xshift=0.15cm,yshift=-0.2cm] {\scriptsize Modified ODE error \\ \scriptsize $\mathcal{O}(h^3)$};

\draw [
    thick,
    decoration={
        brace,
        mirror,
        raise=2.3cm
    },
    decorate
] (theta_t_plus_1) -- (cont_theta_t_plus_1)
node [pos=0.5,anchor=west,xshift=2.35cm,yshift=-0.2cm] {\scriptsize ODE error \\ \scriptsize $\mathcal{O}(h^2)$};
\vspace{-2em}
\end{tikzpicture}
 \caption{Visualization of our approach, given by backward error analysis. For each player (we show only $\vtheta$ for simplicity), we find the modified ODE $\dot{\tilde{\vtheta}}$ which captures the change in parameters introduced by the discrete updates with an error of $\mathcal{O}(h^3)$. The modified ODE follows the discrete update more closely than the original ODE $\dot{\vtheta}$, which has an error of $\mathcal{O}(h^2)$.}
  \label{fig:idd_graphic}
\end{figure}
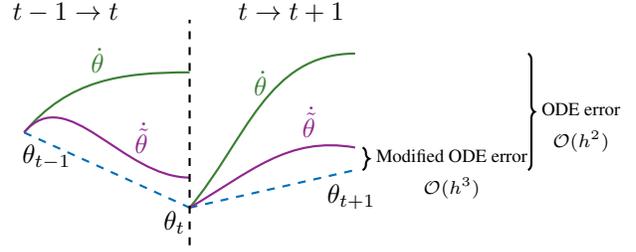

\subsection{DD for simultaneous Euler updates}

{The \textit{simultaneous Euler updates} with learning rates $\alpha h$ and $\lambda h$ respectively are given by
\begin{align}
\vphi_t &= \vphi_{t-1} + \alpha h f(\vtheta_{t-1}, \vphi_{t-1} ) \label{eq:simup1}\\
\vtheta_{t} &= \vtheta_{t-1} + \lambda h g(\vtheta_{t-1}, \vphi_{t-1}) \label{eq:simup2}
\end{align}
\begin{theo} \label{thm:2_players_sim} The discrete \emph{simultaneous} Euler updates
in \eqref{eq:simup1} and \eqref{eq:simup2} follow the continuous system
\label{thm:sim}
\begin{align*}
\dot{\vphi} &= f - \frac{\alpha h}{2} \left(f \nabla_{\vphi} f + g \nabla_{\vtheta} f \right) \\
\dot{\vtheta} &= g - \frac{\lambda h}{2} \left(f \nabla_{\vphi} g + g \nabla_{\vtheta} g \right)
\end{align*}
with an error of size $\mathcal O(h^3)$  after one update step.
\end{theo}

\textit{Remark}: A special case of Theorem~\ref{thm:2_players_sim} for zero-sum games with equal learning rates can be found in~\citet{lu2021resolution}.

\subsection{DD for alternating Euler updates}

For \textit{alternating Euler updates}, the players take turns to update their parameters, and can perform multiple updates each. We denote the number of alternating updates of the first  player (resp. second player) by $m$ (resp. $k$).
We scale the learning rates by the number of updates, leading to the following updates $ \vphi_{t} := \vphi_{m,t} $ and $\vtheta_{t} := \vtheta_{k, t}$ where
\begin{align}
 \vphi_{i, t} &=  \vphi_{i-1, t} + \frac{\alpha h}{m}   f(\vphi_{i-1,t}, \vtheta_{t-1}) , \hspace{1em} i = 1 \dots m, \label{eq:altup1} \\
 \vtheta_{j, t} &=  \vtheta_{j-1, t} + \frac{\lambda h}{k} g(\vphi_{m, t}, \vtheta_{j-1, t}), \hspace{1em} j = 1 \dots k. \label{eq:altup2}
\end{align}
\begin{theo} The discrete \emph{alternating} Euler updates in \eqref{eq:altup1} and \eqref{eq:altup2} follow the continuous system
\label{thm:alt}
\begin{align*}
\dot{\vphi} &= f - \frac{\alpha h}{2} \left(\frac{1}{m}f \nabla_{\vphi} f + g \nabla_{\vtheta} f \right) \\
\dot{\vtheta} &= g - \frac{\lambda h}{2} \left((1- \frac {2\alpha} {\lambda})f \nabla_{\vphi} g + \frac{1}{k} g \nabla_{\vtheta} g \right)
\end{align*}
with an error of size $\mathcal O(h^3)$ after one update step.
\end{theo}

\textit{Remark}: Equilibria of the original continuous systems (i.e., points where $f = \myvec{0}$ and $g=\myvec{0}$) remain equilibria of the modified continuous systems.

\begin{definition}
\normalfont{The discretization drift for each player has two terms: one term containing a player's own update function only - terms we will call \textit{self terms} - and a term that also contains the other player's update function - which we will call \textit{interaction terms}.}
\end{definition}

\subsection{Sketch of the proofs}

Following backward error analysis, the idea is to modify the original continuous system by adding corrections in powers of the learning rate: $\tilde f = f + h f_1 + h^2 f_2 + \cdots$ and $\tilde g = g + h g_1 + h^2 g_2 + \cdots$, where for simplicity in this proof sketch, we use the same learning rate for both players (detailed derivations can be found in the Supplementary Material). We want to find corrections $f_i$, $g_i$ such that the modified system $\dot \vphi = \tilde f$ and $\dot \vtheta = \tilde g$ follows the discrete update steps exactly. To do that we proceed in three steps:

{\bf Step 1:} We expand the numerical scheme to find a relationship between $\vphi_t$ and $\vphi_{t-1}$ and $\vtheta_t$ and $\vtheta_{t-1}$ up to order $\mathcal{O}(h^2)$.
In the case of the simultaneous Euler updates this does not require any change to Equations~\eqref{eq:simup1} and~\eqref{eq:simup2}, while for alternating updates we have to expand the intermediate steps of the integrator using Taylor series.

{\bf Step 2:} We compute the Taylor's series of the modified equations solution, yielding:
\begin{align*}
\tilde \vphi(h)   &=  \vphi_{t-1}   + h f + h^2 (f_1 + \frac 12 (f\nabla_\phi f + g\nabla_\theta f)) + \mathcal{O}(h^3)\\
\tilde \vtheta(h) &=  \vtheta_{t-1} +  h g +  h^2 (g_1 + \frac 12 (f\nabla_\phi g + g\nabla_\theta g)) + \mathcal{O}(h^3),
\end{align*}
where all the $f$'s, $g$'s, and their derivatives are evaluated at ($\vphi_{t-1}$, $\vtheta_{t-1}$).

{\bf Step 3:} We match the terms of equal power in $h$ so that the solution of modified equations coincides with the discrete update after one step.
 This amounts to finding the corrections, $f_i$'s and $g_i$'s, so that all the terms of order higher than $\mathcal{O}(h)$ in the modified equation solution above will vanish; this yields the first order drift terms $f_1$ and $g_1$ in terms of $f$, $g$, and their derivatives. For simultaneous updates we obtain $f_1 =  - \tfrac{1}{2} (f\nabla_{\vphi} f + g\nabla_{\vtheta} f)$ and $g_1 = - \tfrac{1}{2}(f\nabla_{\vphi} g + g\nabla_{\vtheta} g)$, and by construction, we have obtained the modified truncated system which follows the discrete updates exactly up to order $\mathcal O(h^3)$, leading to Theorem~\ref{thm:sim}.

\emph{Remark:} The modified equations in Theorems \ref{thm:2_players_sim} and \ref{thm:alt} closely follow the discrete updates only for learning rates where errors of size $\mathcal O(h^3)$ can be neglected. Beyond this, higher order corrections are likely to contribute to the DD.

\subsection{Visualizing trajectories}
To illustrate the effect of DD in two-player games, we use a simple example adapted from~\citet{balduzzi2018mechanics}:
\begin{align}
\dot{\phi} = f(\phi, \theta) =  - \epsilon_1 \phi  + \theta; \hspace{1em} \dot{\theta} = g(\phi, \theta) = \epsilon_2 \theta  - \phi
\label{eq:simple_odes}
\end{align}
In Figure~\ref{fig:igr_dynamics_toy} we validate our theory empirically by visualizing the trajectories of the discrete
Euler steps for simultaneous and alternating updates, and show that they closely match the trajectories of the
corresponding modified continuous systems that we have derived. To visualize the trajectories of the original unmodified continuous system, we use Runge-Kutta 4 (a fourth-order numerical integrator that has no DD up to $\mathcal{O}(h^5)$ in the case where the same learning rates are used for the two players --- see \citet{backward_lifespan} and Supplementary Material for proofs). We will use Runge-Kutta 4 as a baseline throughout the paper.

\begin{figure}[t]
  \centering
  \includegraphics[width=\columnwidth]{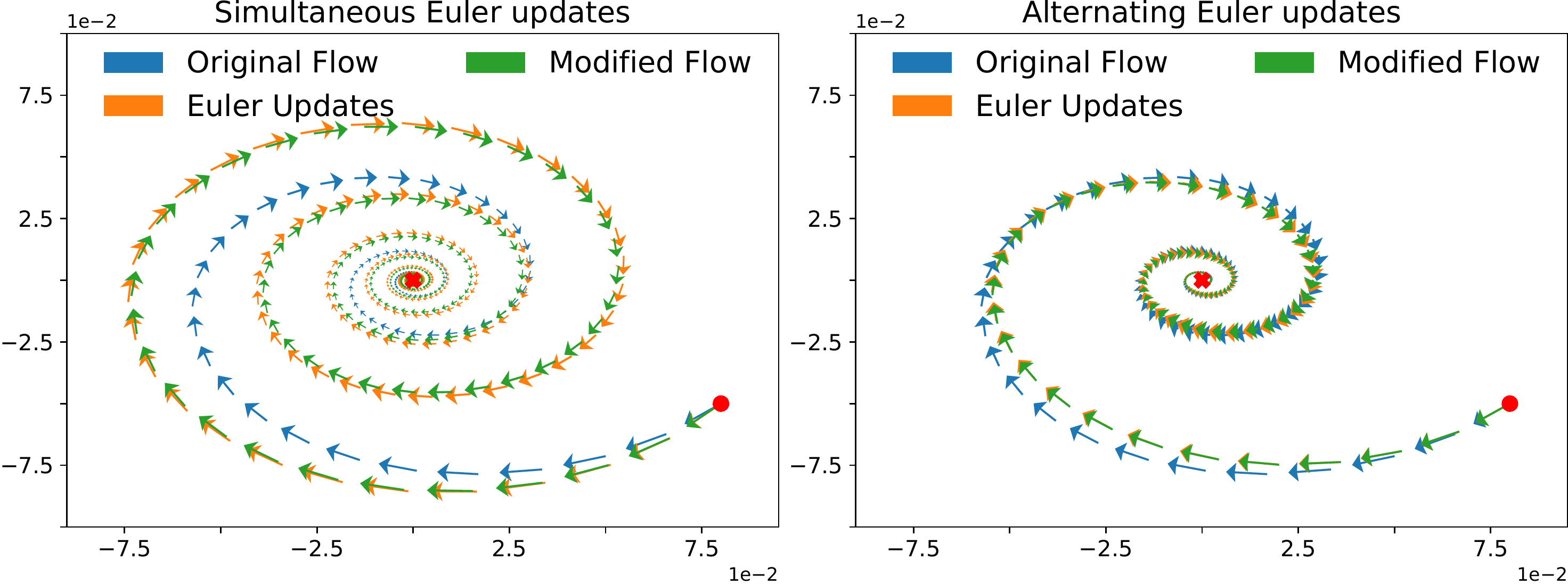}
  \caption{The modified continuous flow captures the effect of DD in simultaneous and alternating Euler updates.}
  \label{fig:igr_dynamics_toy}
\end{figure}

\section{The stability of DD}
\label{section:stability}

The long-term behavior of gradient-based training can be characterized by the stability of its equilibria.
Using stability analysis of a continuous system to understand discrete dynamics in two-player games has a fruitful history; however, prior work ignored discretization drift, since they analyse the stability of the original game ODEs.
 Using the modified ODEs given by backward error analysis provides two benefits here: 1) they account for the discretization drift, and 2) they provide \textit{different} ODEs for simultaneous and for alternating updates, capturing the specificities of the two optimizers.

\begin{figure}[t]
  \centering
  \includegraphics[width=0.95\columnwidth]{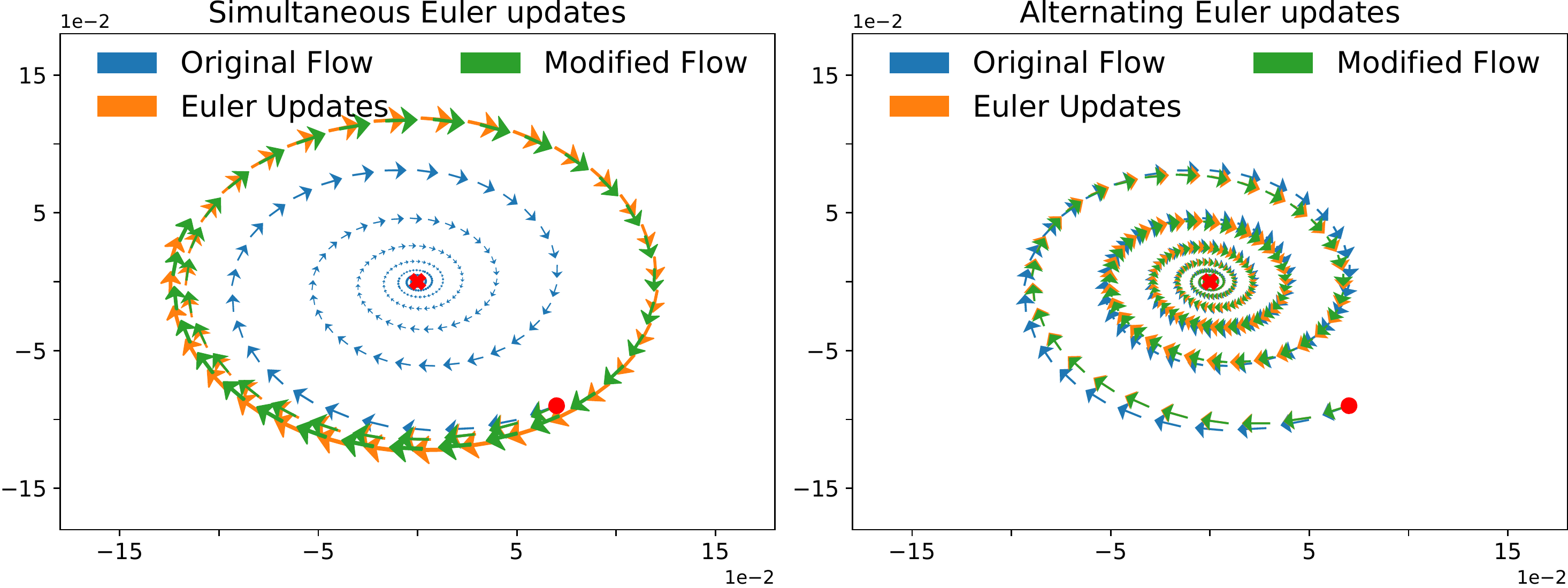}
  \caption{Discretization drift can change the stability of a game. $\epsilon_1 = \epsilon_2 = 0.09$ and a learning rate $0.2$.}
  \label{fig:igr_divergence}
\end{figure}

The modified ODE approach gives us a method to analyze the stability of the discrete updates: 1) Choose the system and the update type -- simultaneous or alternating -- to be analyzed; 2) write the modified ODEs for the chosen system as given by Theorems~\ref{thm:sim} and~\ref{thm:alt}; 3) write the corresponding modified Jacobian, and evaluate it at an equilibrium; 4) determine the stability of the equilibrium by computing the eigenvalues of the modified Jacobian.

Steps 1 and 2 are easy, and step 4 is required in the stability analysis of any continuous system. For step 3, we provide a general form of the modified Jacobian at the equilibrium point of the original system, where $f=\myvec{0}$ and $g=\myvec{0}$:
\begin{gather}
 \widetilde{J}
 =
  \begin{bmatrix}
    \nabla_{\vphi}\widetilde{f} & \nabla_{\vtheta}\widetilde{f} \\
    \nabla_{\vphi}\widetilde{g}  &
   \nabla_{\vtheta}\widetilde{g}
   \end{bmatrix}
   = J - \frac{h}{2} K
\end{gather}
where $J$ is the unmodified Jacobian and $K$ is a matrix that depends on the update type. For simultaneous updates (see Supplementary Material for alternating updates) $K=$
\begin{gather*}
  \begin{bmatrix}
    \alpha ( \nabla_{\vphi} f)^2 +  \alpha \nabla_{\vphi}g\nabla_{\theta}f   & \alpha \nabla_{\vtheta}f\nabla_{\vphi}f +  \alpha \nabla_{\vtheta}g\nabla_{\vtheta}f\\
   \lambda \nabla_{\vphi}g\nabla_{\vtheta}g +  \lambda \nabla_{\vphi}f\nabla_{\vphi}g&
   \lambda( \nabla_{\vtheta} g)^2 +  \lambda \nabla_{\vtheta}f\nabla_{\vphi}g
   \end{bmatrix}.
\end{gather*}

Using the method above we show (in the Supplementary Material), that, in two-player games, the drift can change a stable equilibrium into an unstable one, which is not the case for supervised learning~\citep{igr}. For simultaneous updates with equal learning rates, this recovers a result of~\citet{daskalakis2018limit} derived in the context of zero-sum games. We show this by example: consider the game given by the system of ODEs in Equation~\eqref{eq:simple_odes} with $\epsilon_1 = \epsilon_2 = 0.09$. The stability analysis of its modified ODEs for the simultaneous Euler updates shows they diverge when $\alpha h = \lambda h =0.2$. We use the same example to illustrate the difference in behavior between simultaneous and alternating updates: the stability analysis shows the modified ODEs for alternating updates converge to a stable equilibrium. In both cases, the results obtained using the stability analysis of the modified ODEs is consistent with empirical outcomes obtained by following the corresponding discrete updates, as shown in Figure~\ref{fig:igr_divergence}; this would not have been the case had we used the original system to do stability analysis, which would have always predicted convergence to an equilibrium.

The modified ODEs help us bridge the gap between theory and practice: they allow us to extend the reach of stability analysis to a wider range of techniques used for training, such as alternating gradient descent. We hope the method we provide will be used in the context of GANs, to expand prior work such as that of~\citet{nagarajan2017gradient} to alternating updates. However, the modified ODEs are not without limitations: they ignore discretization errors smaller than $\mathcal{O}(h^3)$, and thus they are not equivalent to the discrete updates; methods that directly assess the convergence of discrete updates (e.g. ~\citet{mescheder2017numerics}) remain an indispensable tool for understanding discrete systems.

\section{Common-payoff games}
\label{sec:common_payoff}

When the players share a common loss, as in \emph{common-payoff games}, we recover supervised learning with a single loss $E$, but with the extra-freedom of training the weights corresponding to different parts of the model with possibly different learning rates and update strategies (see for instnace \citet{imagenet_in_minutes} where a per-layer learning rate is used to obtain extreme training speedups at equal levels of test accuracy).
A special case occurs when the two players with equal learning rates ($\alpha = \lambda$) perform simultaneous gradient descent. In this case, both modified losses exactly recover Equation~\eqref{eq:sup_learning_igr}. \citet{igr} argue that minimizing the loss-gradient norm, in this case, has a beneficial effect.

In this section, we instead focus on alternating gradient descent. We partition a neural network into two parts, corresponding to two sets of parameters, $\vphi$ for the parameters closer to the input and $\vtheta$ for the parameters closer to the output. This procedure freezes one part of the network while training the other part and alternating between the two parts.
This scenario may arise in a distributed training setting, as a form of block coordinate descent.
In the case of common-payoff games, we have the following as a special case of Theorem~\ref{thm:alt} by substituting $f = -\nabla_{\vphi} E$ and $g = -\nabla_{\vtheta} E$:

\begin{cor} In a two-player common-payoff game with common loss $E$, \textit{alternating} gradient descent -- as described in equations~\eqref{eq:altup1} and ~\eqref{eq:altup2} - with one update per player follows a gradient flow given by the modified losses
\begin{align}
\tilde L_1&= E + \frac{\alpha h}{4} \left(\norm{\nabla_{\vphi} E}^2  + \norm{\nabla_{\theta} E}^2\right) \\
\tilde L_2&=  E + \frac{\lambda h}{4} \left((1 - \frac {2 \alpha} {\lambda}) \norm{\nabla_{\vphi} E}^2  + \norm{\nabla_{\theta} E}^2\right)
\label{eq:cooperating_zero_sum}
\end{align}
with an error of size $\mathcal O(h^3)$  after one update step.
\end{cor}

The term $(1 - \frac {2 \alpha} {\lambda}) \norm{\nabla_{\vphi} E}^2$ in Eq.~\eqref{eq:cooperating_zero_sum} is negative when the learning rates are equal, seeking to maximize the gradient norm of player $\vphi$. According to \citet{igr}, we expect less stable training and worse performance in this case.
This prediction is verified in Figure~\ref{fig:alternating_updates_supervised_learning}, where we compare simultaneous and alternating gradient descent for a MLP trained on MNIST with a common learning rate.

\begin{figure}[t]
  \includegraphics[width=0.48\columnwidth]{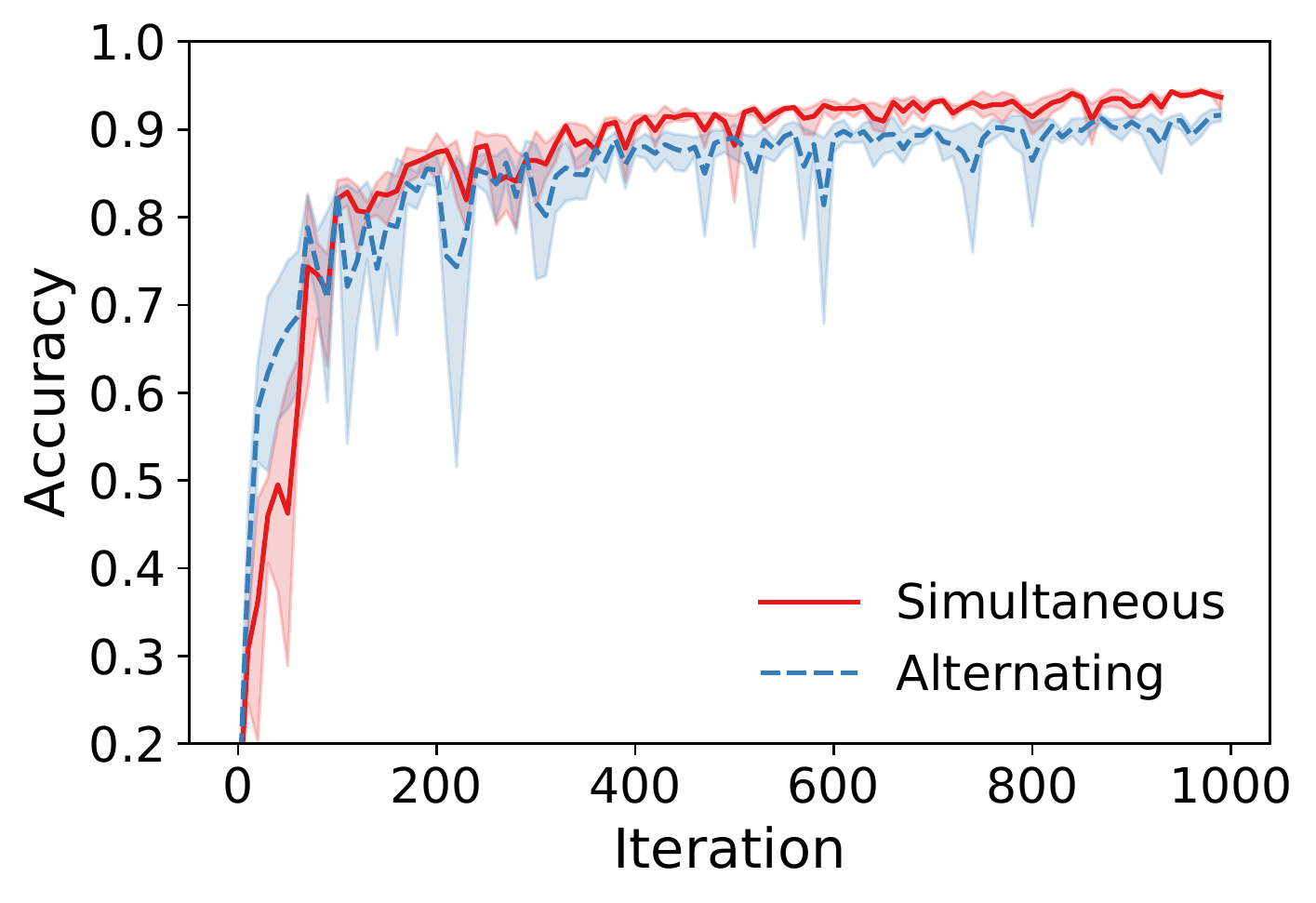}
  \includegraphics[width=0.48\columnwidth]{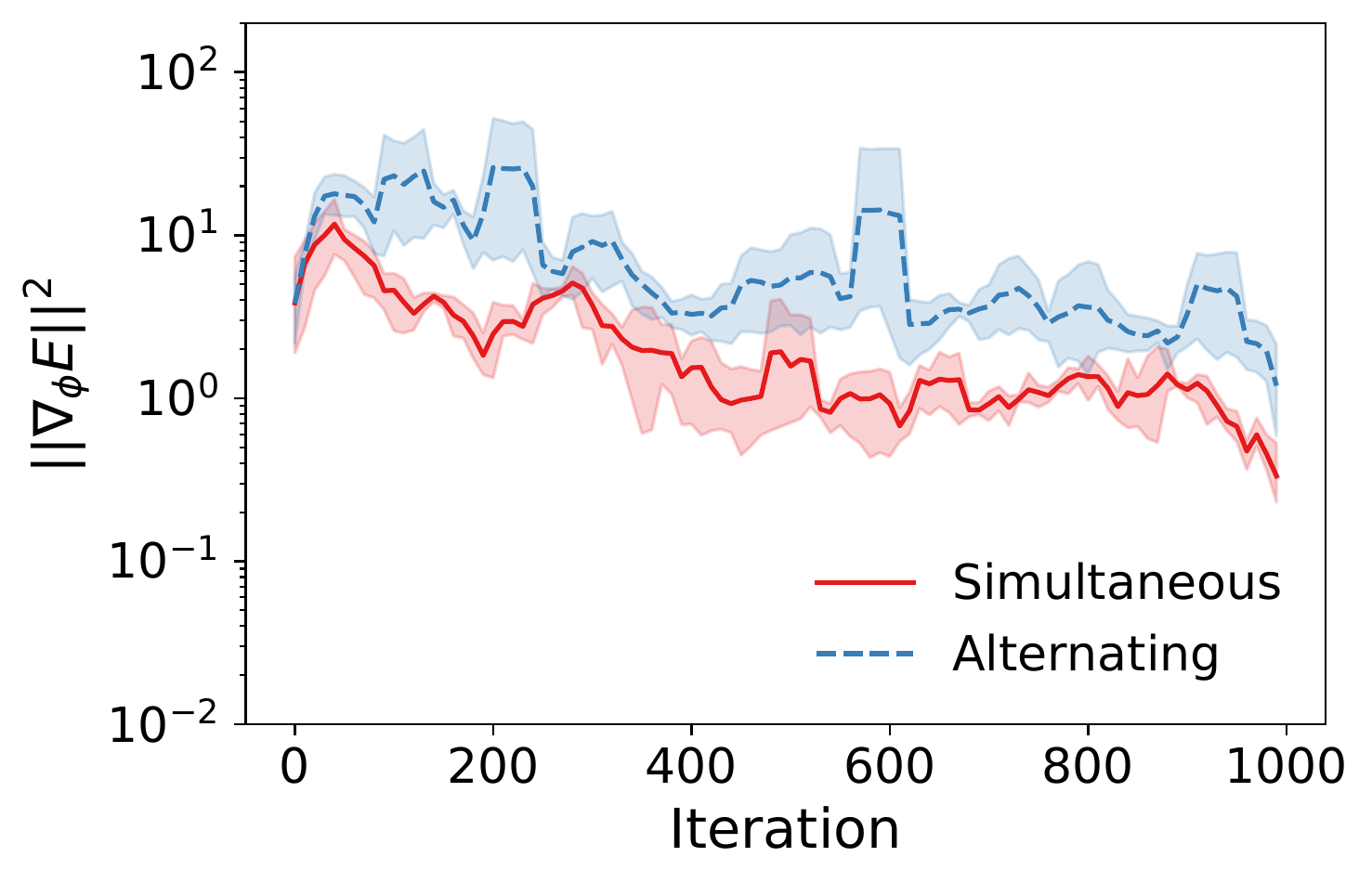}
  \caption{In common-payoff games alternating updates lead to higher gradient norms and unstable training.}
  \label{fig:alternating_updates_supervised_learning}
\end{figure}

\section{Analysis of zero-sum games}
\label{section:zero_sum}

We now study zero-sum games, where the two adversarial players have opposite losses $ - L_1 = L_2 = E$.
We substitute the updates $f = \nabla_{\vphi} E$ and $g = -\nabla_{\vtheta} E$ in the Theorems in Section~\ref{sec:general_theorems} and obtain:

\begin{cor} In a zero-sum two-player differentiable game, \textit{simultaneous} gradient descent updates - as described in equations~\eqref{eq:simup1} and ~\eqref{eq:simup2} - follows a gradient flow given by the modified losses
\label{cor:zs-sim}
\begin{align}
\tilde L_1&= - E + \frac{\alpha h}{4} \norm{\nabla_{\vphi} E}^2  -  \frac{\alpha h}{4} \norm{\nabla_{\theta} E}^2
\label{eq:min_min_simultaneous_zero_sum1}, \\
\tilde L_2&= E -  \frac{\lambda h}{4} \norm{\nabla_{\vphi} E}^2  +  \frac{\lambda h}{4} \norm{\nabla_{\theta} E}^2,
\label{eq:min_min_simultaneous_zero_sum2}
\end{align}
with an error of size $\mathcal O(h^3)$ after one update step.
\end{cor}
\textit{Remark}: Discretization drift preserves the adversarial structure of the game,
with the \textit{interaction terms} maximizing the gradient norm of the opposite player, while the \textit{self terms} are minimizing the player's own gradient norm.

\textit{Remark}: The modified losses of zero-sum games trained with simultaneous gradient descent with different learning rates are no longer zero-sum.
\begin{cor} In a zero-sum two-player differentiable game, \textit{alternating} gradient descent - as described in equations~\eqref{eq:altup1} and ~\eqref{eq:altup2} - follows a gradient flow given by the modified losses
\label{cor:zs-alt}
\begin{align}
\tilde L_1&= -E + \frac{\alpha h}{4m} \norm{\nabla_{\vphi} E}^2  - \frac{\alpha h}{4} \norm{\nabla_{\theta} E}^2
\label{eq:min_min_alt_zero_sum1}
\\
\tilde L_2&= E - \frac{\lambda h}{4} (1 - \frac{2 \alpha}{\lambda}) \norm{\nabla_{\vphi} E}^2  + \frac{\lambda h}{4k}\norm{\nabla_{\theta} E}^2
\label{eq:min_min_alt_zero_sum2}
\end{align}
with an error of size $\mathcal O(h^3)$ after one update step.
\end{cor}

\textit{Remark}: The modified losses of zero-sum games trained with alternating gradient descent are not zero sum.

\textit{Remark}: Since $(1 - \frac{2 \alpha}{\lambda})  < 1$ in the alternating case there is always less weight on the term encouraging maximizing $\norm{\nabla_{\vphi} E}^2$ compared to the simultaneous case under the same learning rates. For
$\frac{\alpha}{\lambda} > \frac{1}{2}$ both players minimize $\norm{\nabla_{\vphi} E}^2$.

\subsection{Dirac-GAN: an illustrative example}\label{sec:DiracGAN}

\citet{mescheder2018training} introduce the Dirac-GAN as an example to illustrate the often complex training dynamics in zero-sum games. We follow this example to provide an intuitive understanding of DD.
The generative adversarial network~\citep{goodfellow2014generative} is an example of two-player game that has been successfully used for distribution learning. The
generator $G$ with parameters $\vtheta$ learns a mapping from samples of the latent distribution $\vz \sim p(\vz)$ to the data space, while the discriminator $D$ with parameters $\vphi$
learns to distinguish these samples from data.
Dirac-GAN aims to learn a Dirac delta distribution with mass at zero;
the generator is modeling a Dirac with parameter $\theta$: $G(z) = \theta$ and the discriminator is a linear model on the input with
parameter $\phi$: $D_{\phi}(x) = \phi x$. This results in the zero-sum game given by:
\begin{align}
 E = l(\theta \phi) + l(0)
\label{eq:dirac_gan}
\end{align}
where $l$ depends on the GAN formulation used ($l(z) = - \log (1 + e^{-z})$ for instance).
The unique equilibrium point is $\theta = \phi = 0$. All proofs for this section are in the Supplementary Material, together with visualizations.

\textit{Reconciling discrete and continuous updates in Dirac-GAN}: The continuous dynamics induced by the gradient field from Equation~\eqref{eq:dirac_gan} preserve $\theta^2 + \phi^2$, while for simultaneous gradient descent $\theta^2 + \phi^2$ increases with each update~\citep{mescheder2018training}; different conclusions are reached when analyzing the dynamics of the original continuous system versus the discrete updates.
 We show that the modified ODEs given by equations~\eqref{eq:min_min_simultaneous_zero_sum1} and~\eqref{eq:min_min_simultaneous_zero_sum2} resolve this discrepancy, since simultaneous gradient descent \emph{modifies} the continuous dynamics to increase $\theta^2 + \phi^2$, leading to consistent conclusions from the modified continuous and discrete dynamics.

\textit{Explicit regularization stabilizes Dirac-GAN:} To counteract the instability in the Dirac-GAN induced by the interaction terms we can add explicit regularization with the same functional form: $L_1  = - E + u{\| \nabla_{\theta} E\|}^2$ and $L_2  = E + \nu {\| \nabla_{\phi} E\|}^2$  where $u, \nu$ are of $\mathcal{O}(h)$. We find that the modified Jacobian for this modified system with explicit regularization is negative definite if  $u > h\alpha /4$ and $\nu > h\lambda /4$, so the system is asymptotically stable and converges to the optimum. Notably, by quantifying DD, we are able to find the regularization coefficients which guarantee convergence and show that they depend on learning rates.

\section{Experimental analysis of GANs} \label{section:experimental}

To understand the effect of DD on more complex adversarial games, we analyze GANs trained for image generation on the CIFAR10 dataset. We follow the model architectures from Spectral-Normalized GANs~\citep{miyato2018spectral}.
Both players have millions of parameters. We employ the original GAN formulation,
where the discriminator is a binary classifier trained to classify data from the dataset as real and model samples as fake; the generator tries to generate samples that the discriminator considers as real. This can be formulated as a zero-sum game:
\begin{align*}
E = \mathbb{E}_{p^*(\vx)} \log D_{\vphi}(\vx) + \mathbb{E}_{p_{\vtheta}(\vz)} \log( 1 - D_{\vphi}(G_{\theta}(\vz))
\end{align*}

When it comes to gradient descent, GAN practitioners often use alternating, not simultaneous updates: the discriminator is updated first, followed by the generator. However, recent work shows higher-order numerical integrators can work well with simultaneous updates~\citep{odegan}. We will show that DD can be seen as the culprit behind some of the challenges in simultaneous gradient descent in zero-sum GANs, indicating ways to improve training performance in this setting.
For a clear presentation of the effects of DD, we employ a minimalist training setup. Instead of using popular adaptive optimizers such as Adam~\citep{kingma2014adam}, we train all the models with vanilla stochastic gradient descent,
 without momentum or variance reduction methods.

We use the Inception Score (IS)~\citep{improved_techniques_gans}
for evaluation.
Our training curves contain a horizontal line at the Inception Score of $7.5$, obtained with the same architectures we use, but with the Adam optimizer (the score reported by~\citet{miyato2018spectral} is $7.42$). \textit{All learning curves correspond to the best $10\%$ of models} for the corresponding training setting from a sweep over learning rates and 5 seeds - for a discussion of variance across seeds and corresponding plots, as well as best top $20\%$ and $30\%$ of models see the Supplementary Material. We also report box plots showing the performance quantiles across all hyperparameter and seeds, together with the top $10\%$ of models. For SGD we use learning rates $\{ 5 \times 10^{-2}, 1\times 10^{-2}, 5 \times 10^{-3}, 1 \times 10^{-3}\}$ for each player; for Adam, we use learning rates  $\{ 1 \times 10^{-4}, 2\times 10^{-4}, 3 \times 10^{-4}, 4 \times 10^{-4}\}$, which have been widely used in the literature~\citep{miyato2018spectral}. When comparing to Runge-Kutta (RK4) to assess the effect of DD we always use the same learning rates for both players. We present results using additional losses, via LS-GAN~\citep{ls_gan}, and report FID results ~\citep{heusel2017gans} and full experimental details in the Supplementary Material. The code associated with this work can be found at \url{https://github.com/deepmind/deepmind-research/dd_two_player_games}.

\subsection{Does DD affect training?}
We start our experimental analysis by showing the effect of DD on zero-sum games. We compare simultaneous gradient descent, alternating gradient descent and Runge-Kutta 4 updates,
 since they follow different continuous dynamics given by the modified equations we have derived up to $\mathcal{O}(h^3)$ error. Figure~\ref{fig:idd_zero_sum_games} shows simultaneous gradient descent performs substantially worse than alternating updates. When compared to Runge-Kutta 4, which has a DD error of $\mathcal{O}(h^5)$ when the two players have equal learning rates, we see that Runge-Kutta performs better, and that removing the drift improves training. Multiple updates also affect training, either positively or negatively, depending on learning rates - see Supplementary Material.

 \begin{figure}[t]
 \centering
  \includegraphics[width=0.48\columnwidth]{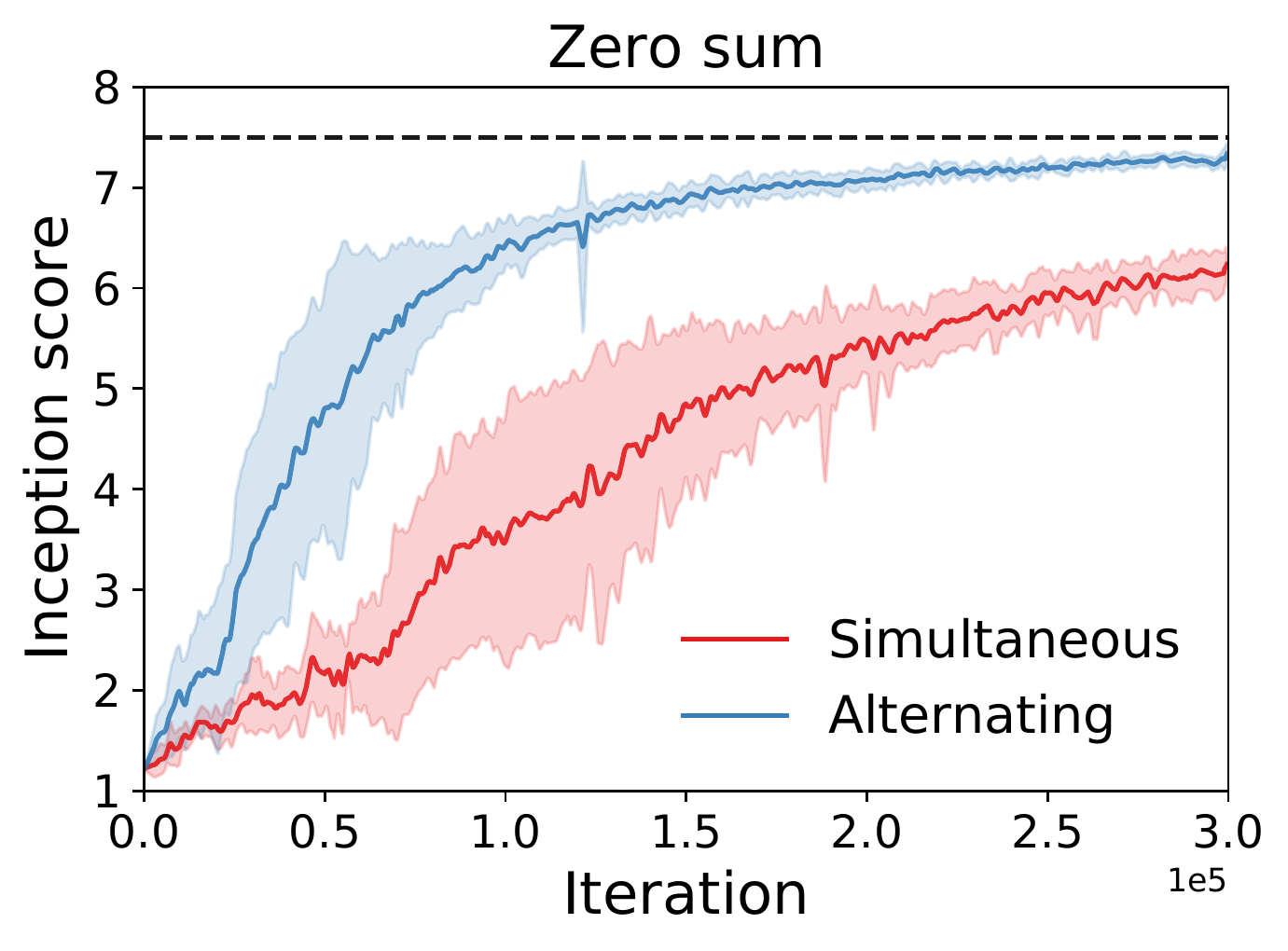}
  \includegraphics[width=0.48\columnwidth]{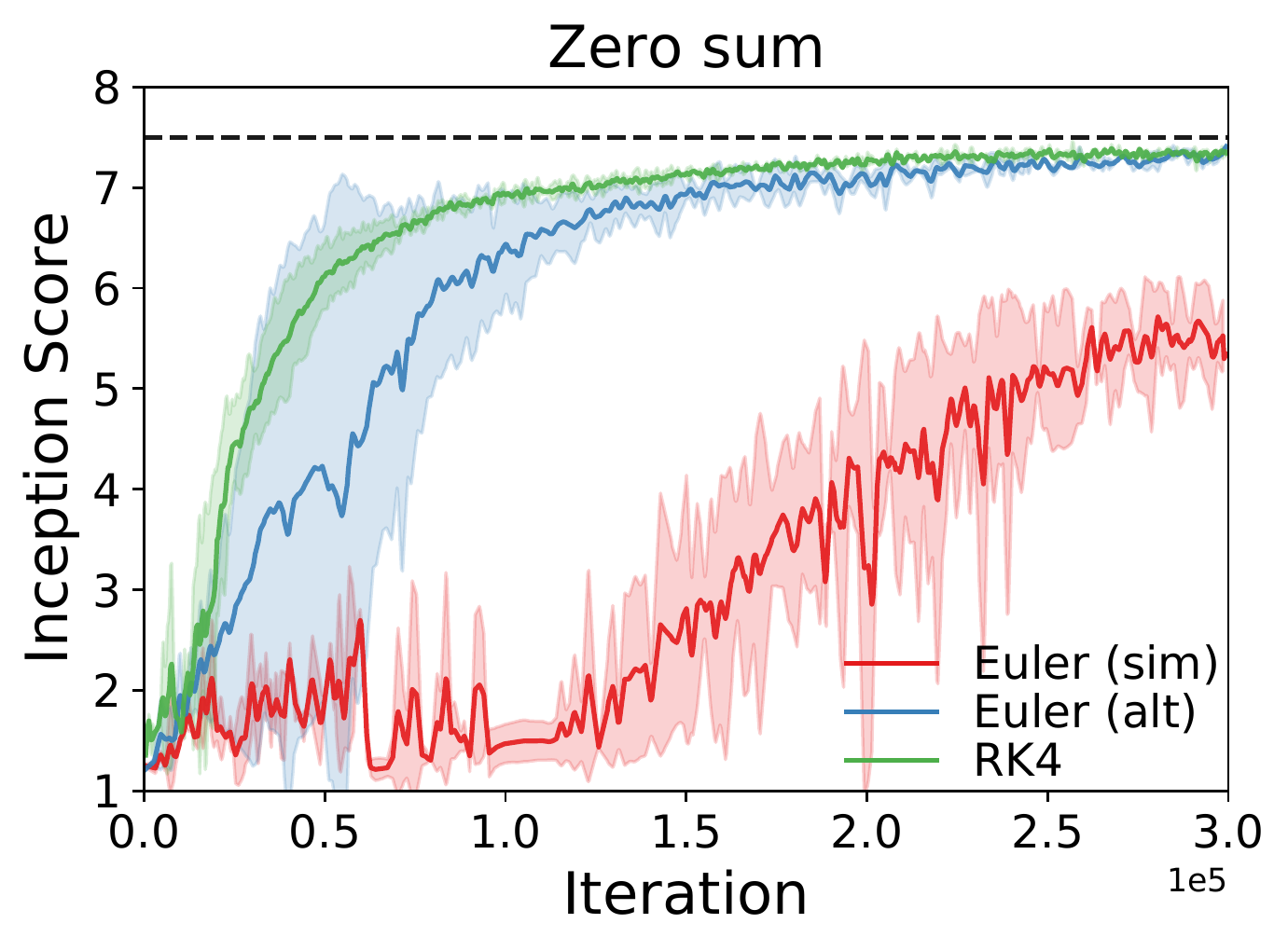}
    \caption{The effect of DD on zero-sum games. Alternating updates perform better (left), and with equal learning rates (right), RK4 has $\mathcal{O}(h^5)$ drift and performs better.}
   \label{fig:idd_zero_sum_games}
\end{figure}

\subsection{The importance of learning rates in DD}
\label{section:learning_rate_ratios}

While in simultaneous updates (Eqs~\eqref{eq:min_min_simultaneous_zero_sum1} and~\eqref{eq:min_min_simultaneous_zero_sum2}) the interaction terms of both players maximize the gradient norm of the other player, alternating gradient descent (Eqs~\eqref{eq:min_min_alt_zero_sum1} and~\eqref{eq:min_min_alt_zero_sum2}) exhibits less pressure on the second player (generator) to maximize the norm of the first player (discriminator). In alternating updates, when the ratio between the discriminator and generator learning rates exceeds $0.5$, both players are encouraged to minimize the discriminator's gradient norm. To understand the effect of learning rate ratios in training, we performed a sweep where the discriminator learning rates $\alpha$ are sampled uniformly between $[0.001, 0.01]$, and the learning rate ratios $\alpha / \lambda$ are in $\{0.1, 0.2, 0.5, 1, 2, 5\}$, with results shown in Figure~\ref{fig:sim_vs_alt_zero_sum_learning_rates}. Learning rate ratios greater than $0.5$ perform best for alternating updates,
and enjoy a substantial increase in performance compared to simultaneous updates.

\begin{figure}[t]
  \includegraphics[width=0.48\columnwidth]{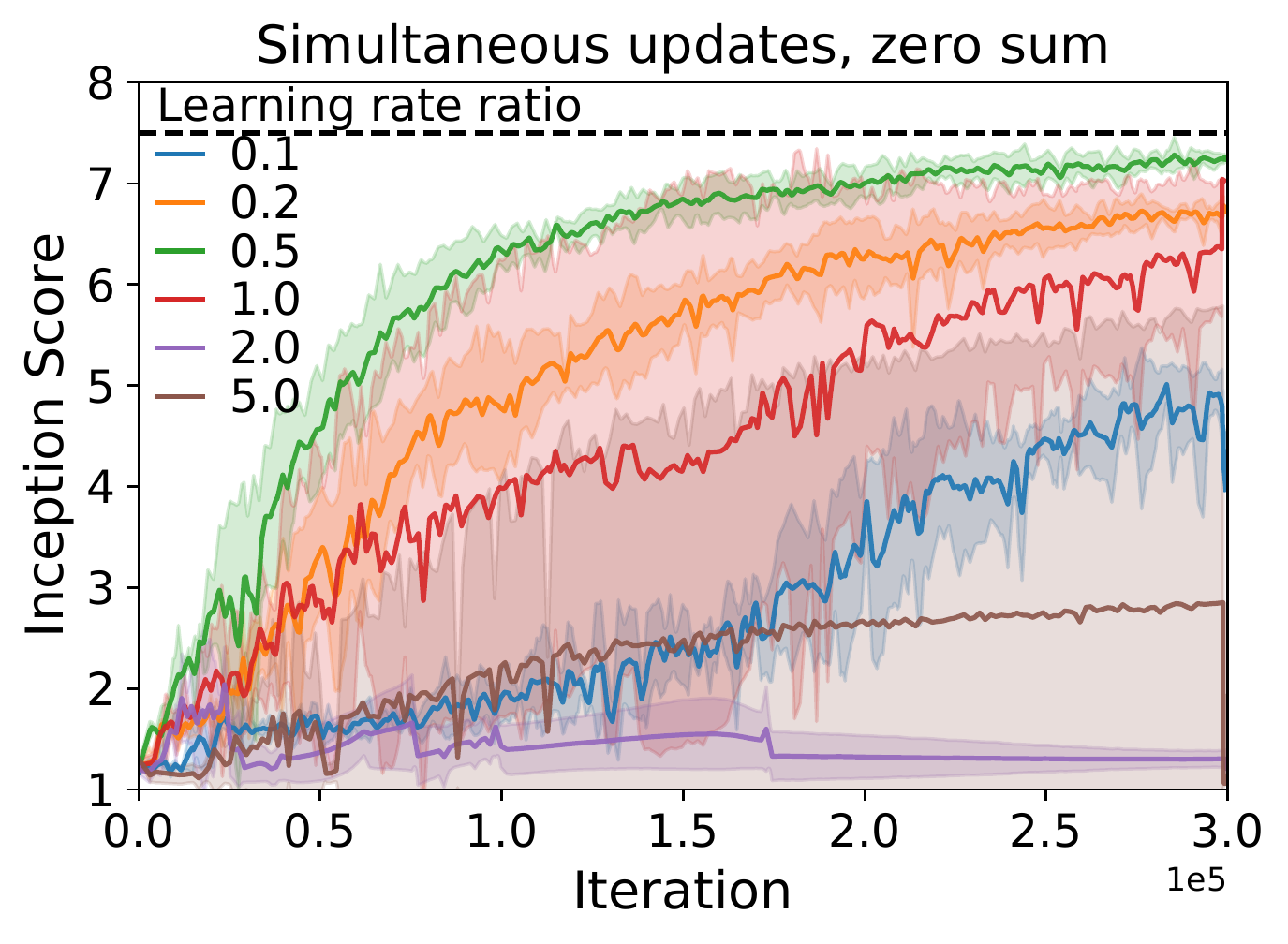}
  \includegraphics[width=0.48\columnwidth]{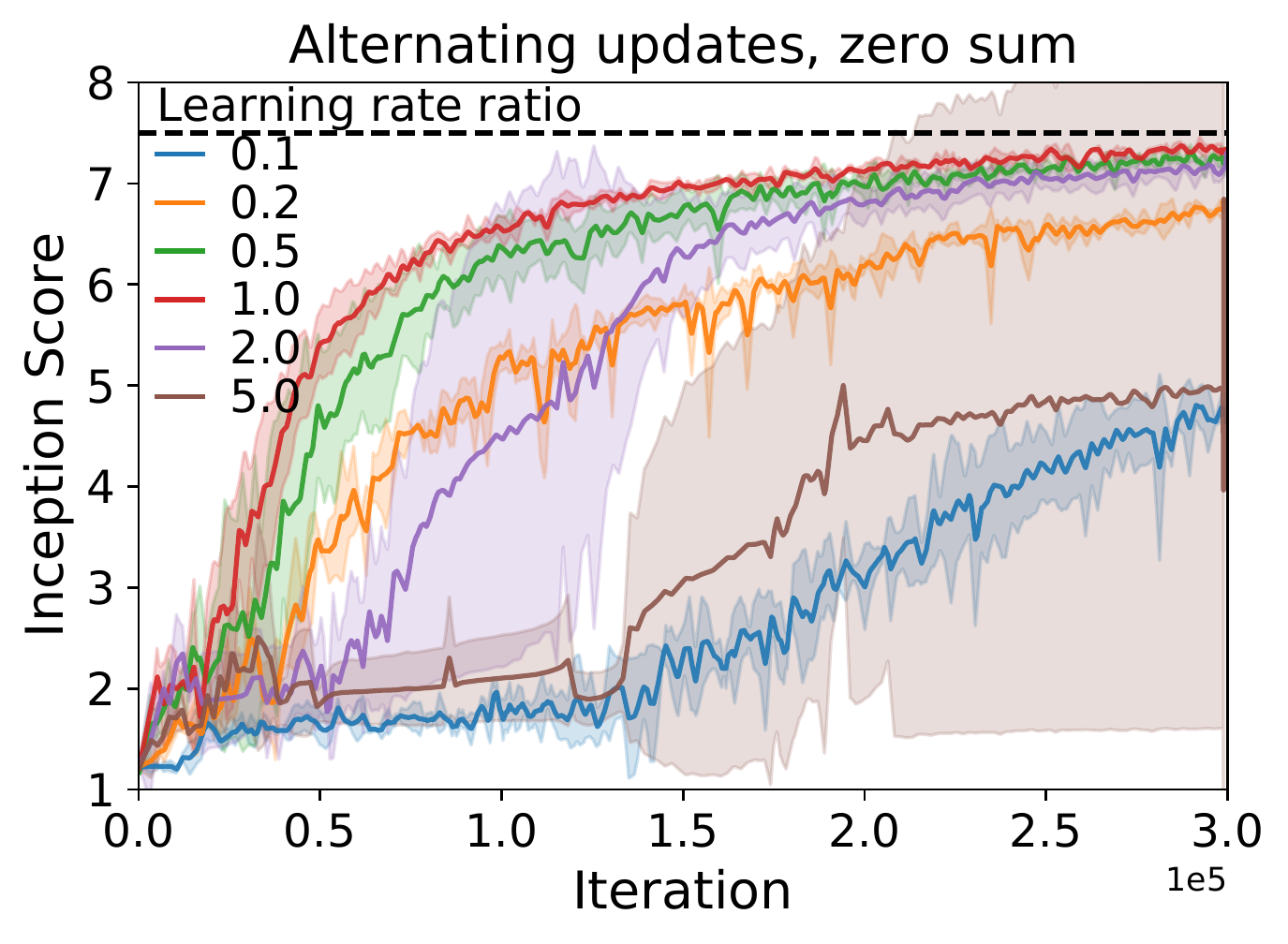}
   \caption{Alternating gradient descent performs better for learning rate ratios which reduce the adversarial nature of DD. The same learning rate ratios show no advantage in the simultaneous case.}
   \label{fig:sim_vs_alt_zero_sum_learning_rates}
\end{figure}

\subsection{Improving performance by explicit regularization}
\label{section:explicit_regularization}

We investigate whether \textit{canceling the interaction terms} between the two players can improve training stability and performance in zero-sum games trained using simultaneous gradient descent. We train models using the losses:
\begin{align}
L_1 &= - E + c_1 \norm{\nabla_{\vtheta} E}^2 \\
L_2 &= E + c_2 \norm{\nabla_{\vphi} E}^2
\label{eq:explicit_regularization_cancel_interaction}
\end{align}

If $c_1, c_2$ are $\mathcal{O}(h)$ we can ignore the DD from these regularization terms, since their effect on DD will be of order $\mathcal{O}(h^3)$.
We can set coefficients to be the negative of the coefficients present in DD, namely $c_1 = \alpha h/4$ and $c_2 = \lambda h/4$; we thus cancel the interaction terms which maximize the gradient norm of the other player, while keeping the self terms, which minimize the player's own gradient norm.
We show results in Figure~\ref{fig:sgd_adam_sim_comparison}: canceling the interaction terms leads to substantial improvement compared to SGD, obtains the same peak performance as Adam (though requires more training iterations) and recovers the performance of Runge-Kutta 4 (Figure~\ref{fig:idd_zero_sum_games}). Unlike Adam, we do not observe a decrease in performance when trained for longer but report higher variance in performance across seeds - see Supplementary Material.

\begin{figure}[t]
 \centering
  \begin{subfigure}{
  \includegraphics[width=0.45\columnwidth]{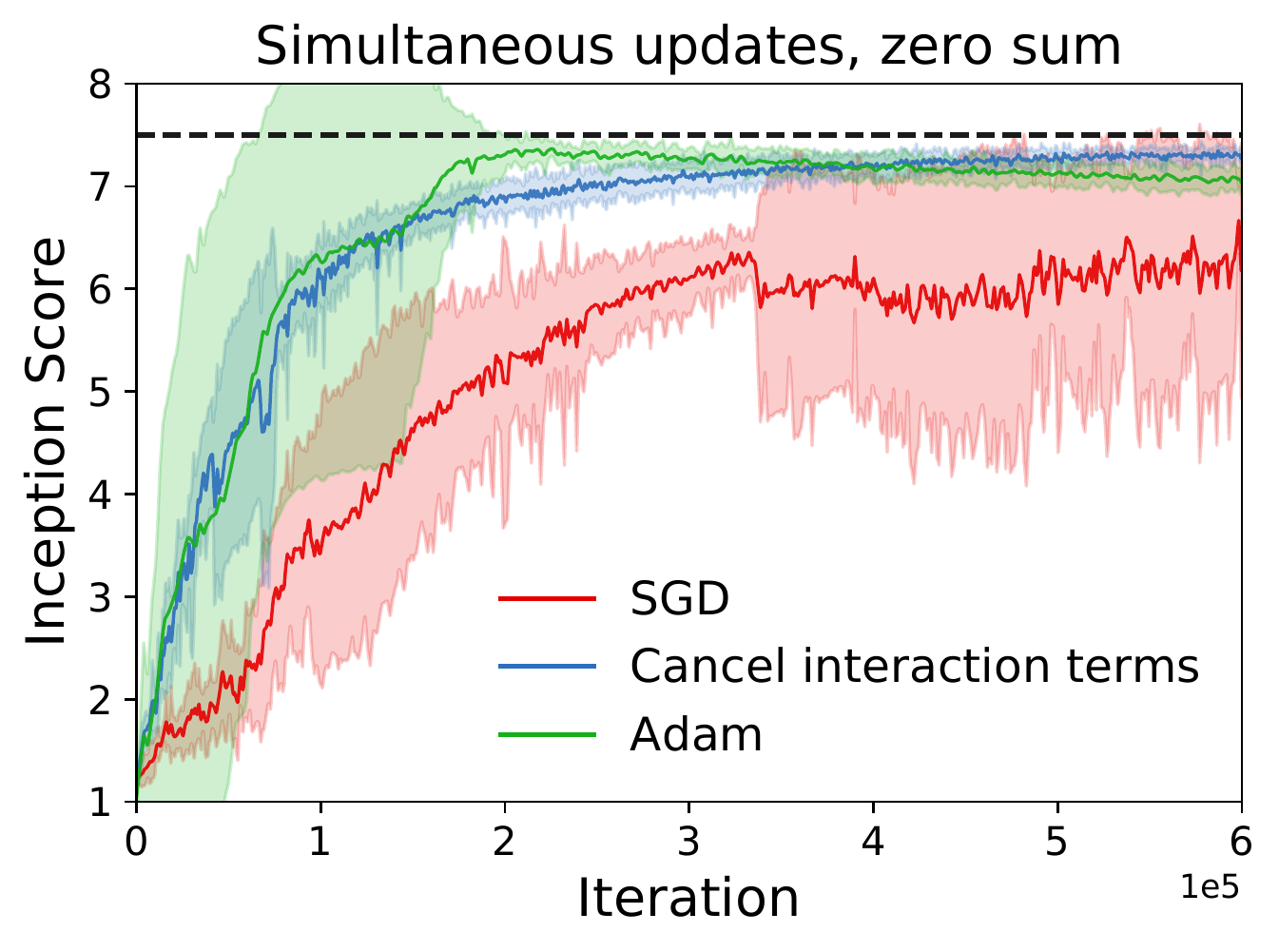}
  \label{fig:sgd_comp_main_paper}
} \end{subfigure}
 \begin{subfigure}{
  \includegraphics[width=0.45\columnwidth]{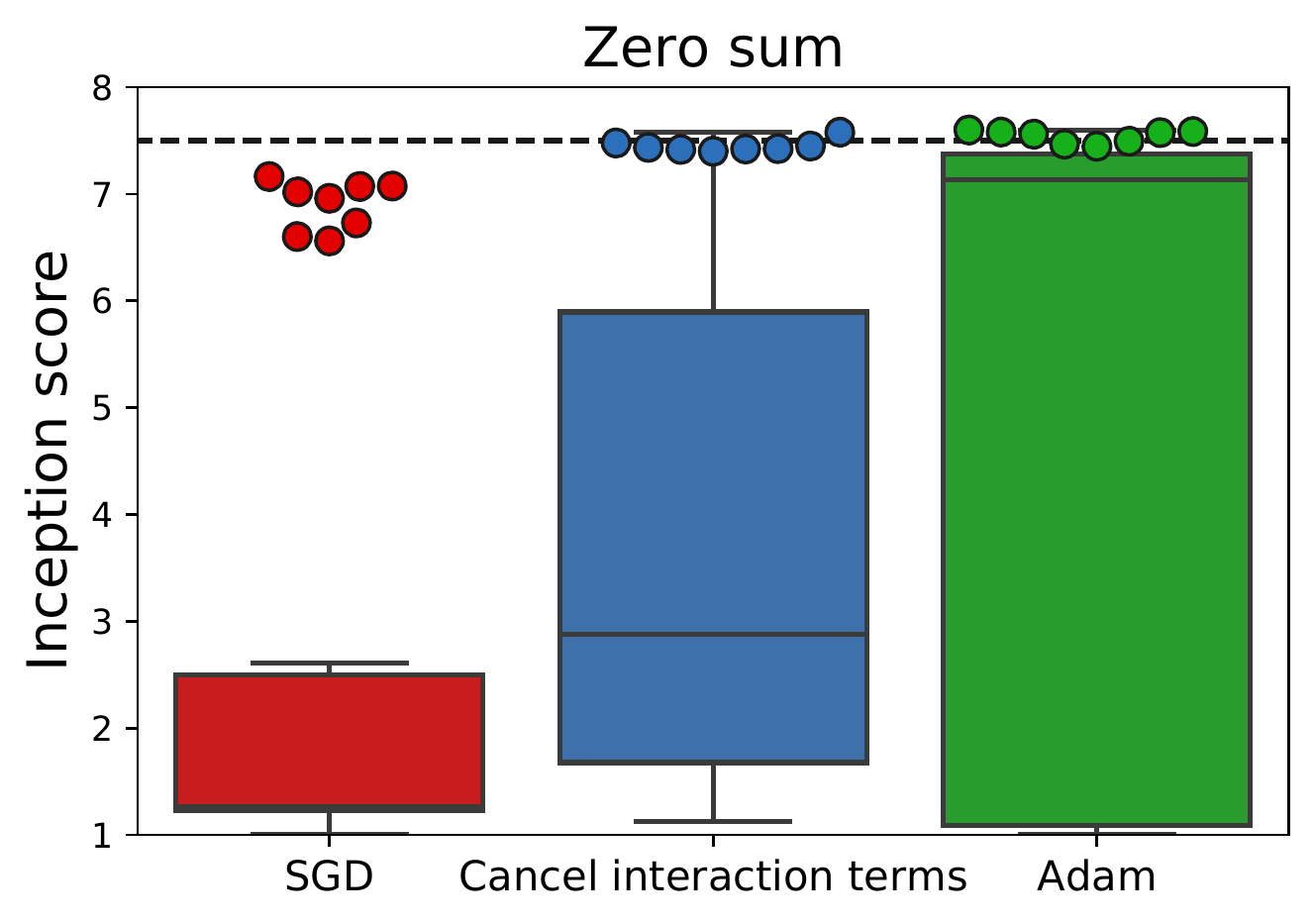}
  \label{fig:cancel_explicit_alt}
} \end{subfigure}
  \caption{Simultaneous updates: Explicit regularization canceling the interaction terms of DD improves performance, both for the best performing models (left) and across a sweep (right).}
  \label{fig:sgd_adam_sim_comparison}
\end{figure}

\begin{figure}[t]
 \centering
  \begin{subfigure}[SGA comparison.]{
  \includegraphics[width=0.465\columnwidth]{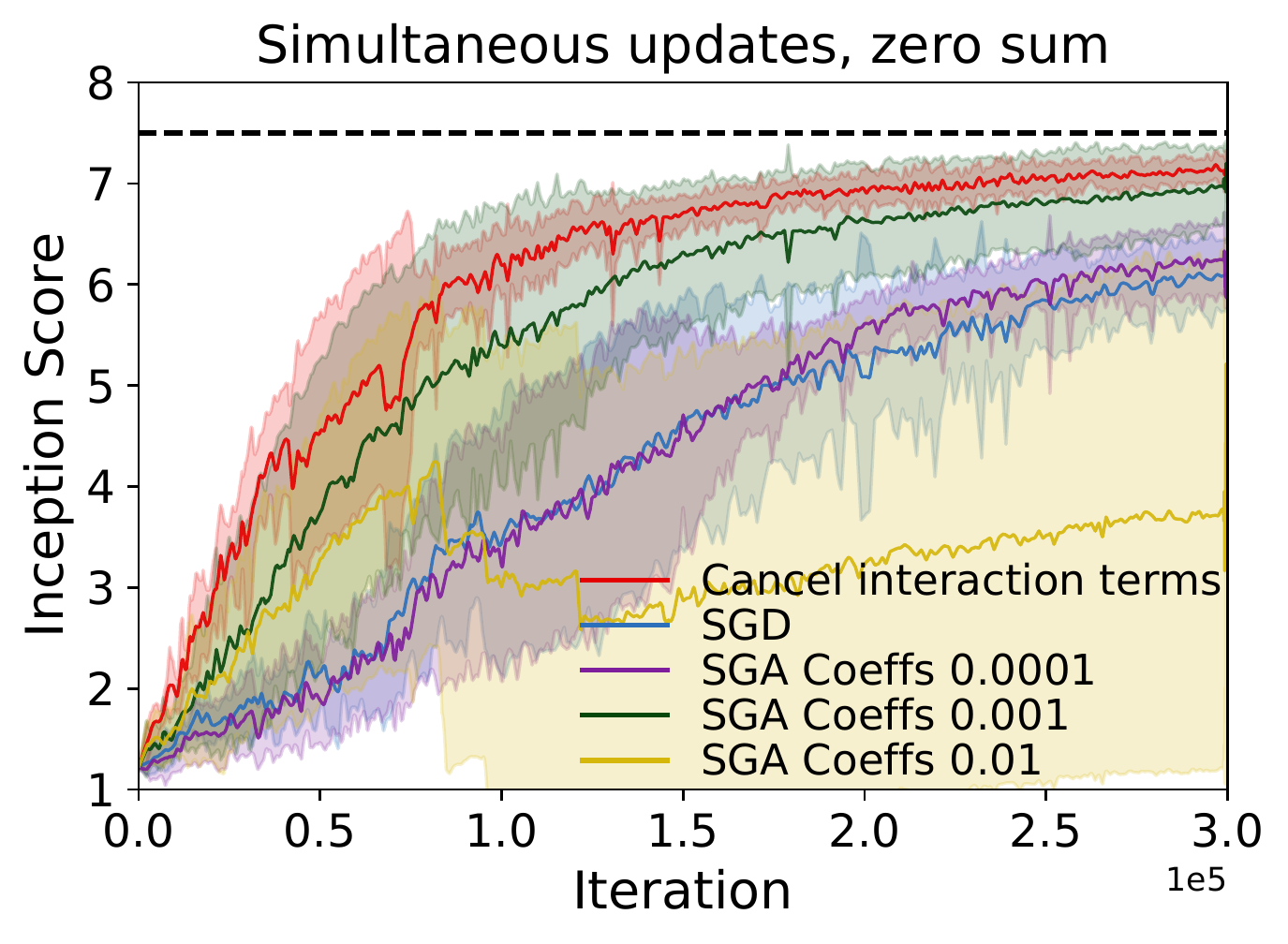}
  \label{fig:sga_comp}
} \end{subfigure}
 \begin{subfigure}[CO comparison.]{
 \includegraphics[width=0.465\columnwidth]{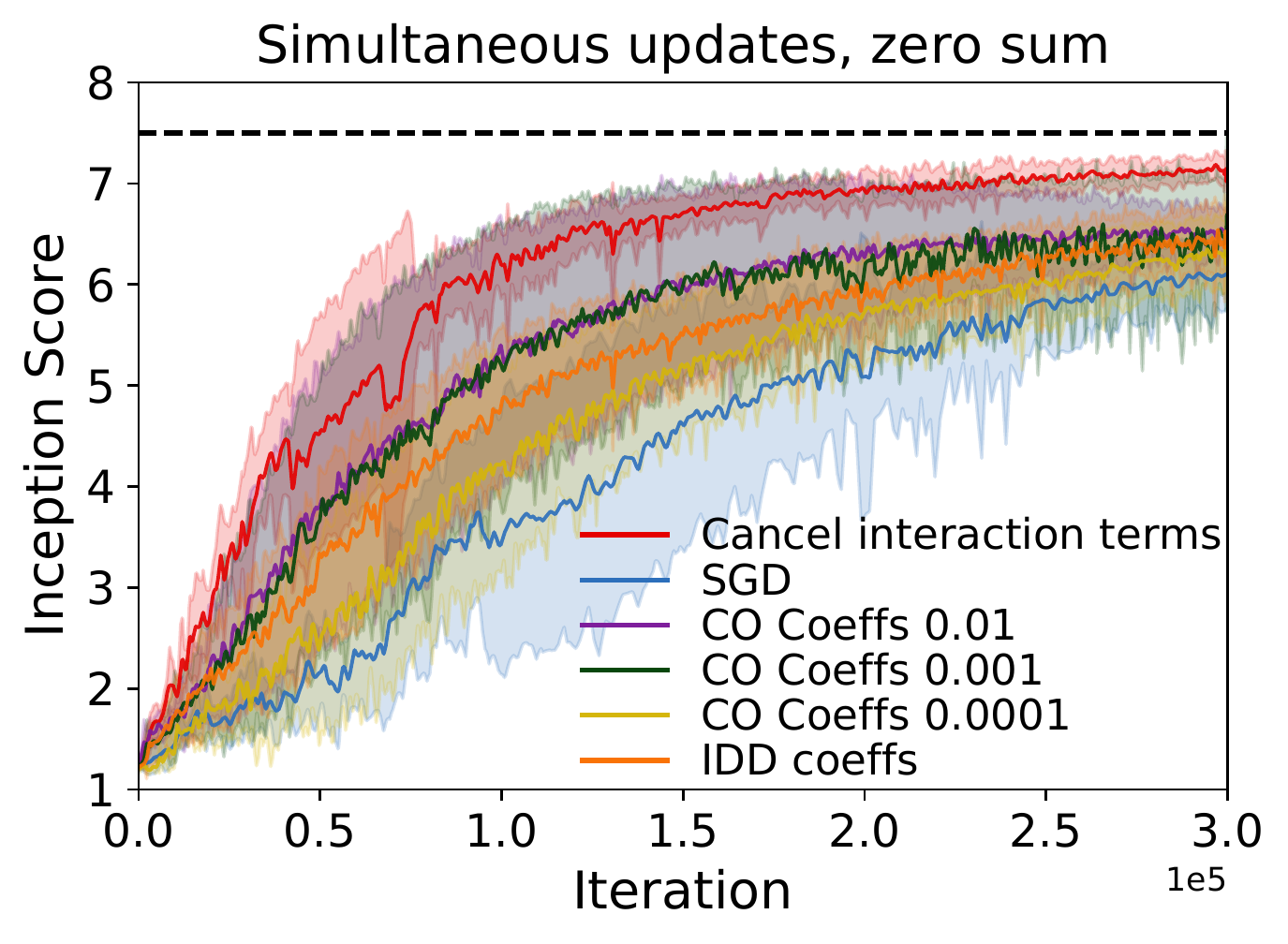}
  \label{fig:co_comp}
} \end{subfigure}
  \caption{Comparison with Symplectic Gradient Adjustment (SGA) and Consensus Optimization (CO): DD motivates  the form of explicit regularization and provides the optimal regularization coefficients, without requiring a larger hyperparameter sweep.}
  \label{fig:sga_consensus_opt_main_paper}
\end{figure}

\textit{Connection with Symplectic Gradient Adjustment (SGA)}: \citet{balduzzi2018mechanics} proposed SGA to improve the dynamics of gradient-based method for games, by counter-acting the rotational force of the vector field. Adjusting the gradient field can be viewed as modifying the losses as in Equation~\eqref{eq:explicit_regularization_cancel_interaction}; the modification from SGA cancels the interaction terms we identified. However, it is unclear whether the fixed coefficients of $c_1 = c_2 = \frac{1}{2}$ used by SGA are always optimal while our analysis shows the strength of DD changes with the exact discretization scheme, such as the step size. Indeed, as our experimental results in Figure~\ref{fig:sga_comp} show, adjusting the coefficient in SGA strongly affects training.

\textit{Connection with Consensus Optimization (CO)}: \citet{mescheder2017numerics} analyze the discrete dynamics of gradient descent in zero-sum games and prove that, under certain assumptions, adding explicit regularization that encourages the players to minimize the gradient norm of both players guarantees convergence to a local Nash equilibrium. Their approach includes canceling the interaction terms, but also requires \textit{strengthening the self terms}, using losses:
\begin{align}
L_1 &= - E + c_1 \norm{\nabla_{\vtheta} E}^2 + s_1 \norm{\nabla_{\vphi} E}^2 \\
L_2 &= E + s_2 \norm{\nabla_{\vtheta} E}^2 + c_2 \norm{\nabla_{\vphi} E}^2,
\label{eq:explicit_regularization_self_terms}
\end{align}
where they use $s_1 = s_2 = c_1 = c_2 = \gamma$ where $\gamma$ is a hyperparameter. In order to understand the effect of the self and interaction terms, we compare to CO, as well as a similar approach where we use coefficients proportional to the drift, namely $s_1 = \alpha h/4$ and $s_2 = \lambda h/4$; this effectively doubles the strength of the self terms in DD. We show results in Figure~\ref{fig:co_comp}. We first notice that CO can improve results over vanilla SGD. However, similarly to what we observed with SGA, the regularization coefficient is important and thus requires a hyperparameter sweep, unlike our approach which uses the coefficients provided by the DD. We further notice that strengthening the norms using the DD coefficients can improve training, but performs worse compared to only canceling the interaction terms. This shows the importance of finding the right training regime, and that strengthening the self terms does not always improve performance.

\textit{Alternating updates}: We perform the same exercise for alternating updates, where $c_1 = \alpha h/4$ and $c_2 = \lambda h / 4 ( 1 - \frac{2 \alpha}{\lambda}) $. We also study the performance obtained by only canceling the discriminator interaction term, since when $\alpha/\lambda > 0.5$ the generator interaction term minimizes, rather than maximizes, the discriminator gradient norm and thus the generator interaction term might not have a strong destabilizing force.
We observe that adding explicit regularization mainly brings the benefit of reduced variance when canceling the discriminator interaction term (Figure~\ref{fig:sgd_adam_sim_alt_updates}). As for simultaneous updates, we find that knowing the form of DD guides us to a choice of explicit regularization: for alternating updates canceling both interaction terms can hurt training, but the form of the modified losses suggests that we should only cancel the discriminator interaction term, with which we can obtain some gains.

\begin{figure}[t]
 \centering
  \begin{subfigure}{
  \includegraphics[width=0.45\columnwidth]{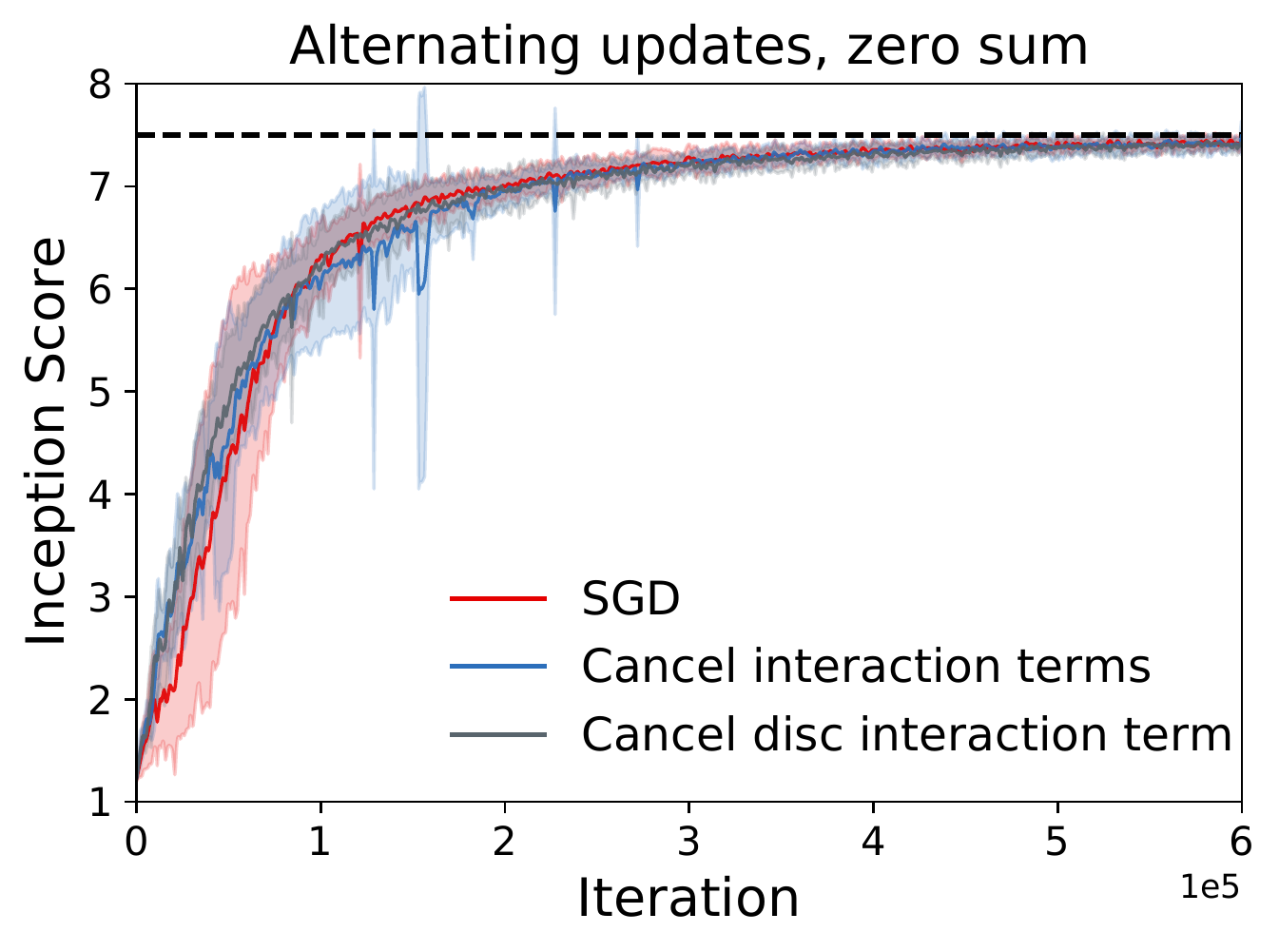}
  \label{fig:ch2}
} \end{subfigure}
 \begin{subfigure}{
  \includegraphics[width=0.45\columnwidth]{
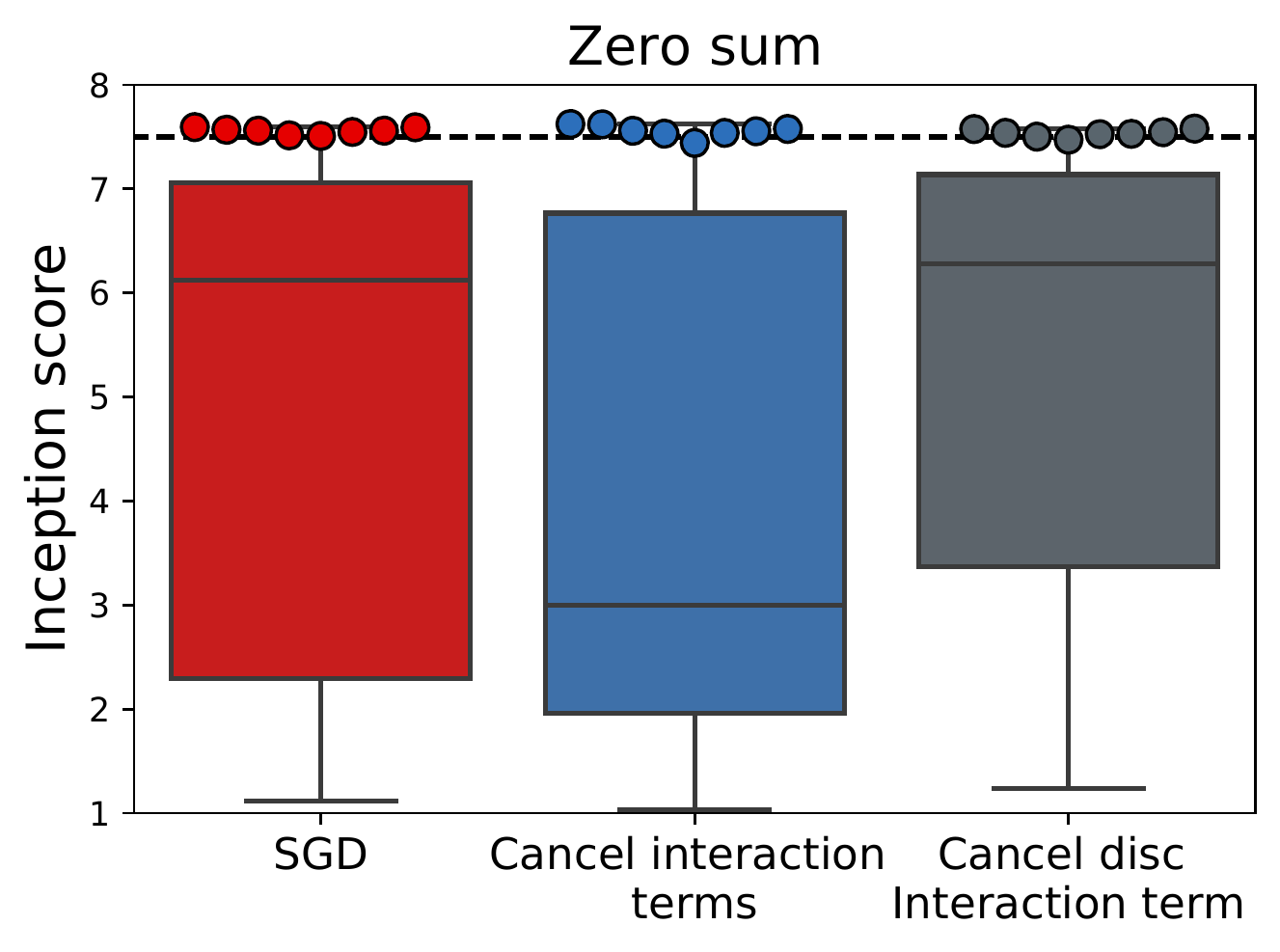}
  \label{fig:ch1}
} \end{subfigure}
  \caption{Alternating updates: DD tells us which explicit regularization terms to use; canceling only the discriminator interaction term improves sensitivity across hyperparameters (right).}
  \label{fig:sgd_adam_sim_alt_updates}
\end{figure}

\textit{Does canceling the interaction terms help for every choice of learning rates?} The substantial performance and stability improvement we observe applies to the performance obtained across a learning rate sweep. For individual learning rate choices however, canceling the interaction terms is not guaranteed to improve learning.

\subsection{Extension to non-zero-sum GANs}

Finally, we extend our analysis to GANs with the non-saturating loss for the generator $E_G = -\log D_\phi(G_\theta(\vz))$ introduced by \citet{goodfellow2014generative}, while keeping the discriminator loss unchanged as $E_D = \mathbb{E}_{p^*(\vx)} \log D_{\vphi}(\vx) + \mathbb{E}_{p_{\vtheta}(\vz)} \log( 1 - D_{\vphi}(G_{\theta}(\vz))$.
In contrast with the dynamics from zero-sum GANs we analyzed earlier, changing from simultaneous to alternating updates results in little change in performance - as can be seen in Figure~\ref{fig:non_saturating}. Despite having the same adversarial structure and the same discriminator loss, changing the generator loss changes the relative performance of the different discrete update schemes.
Since the effect of DD strongly depends on the game, we recommend analyzing the performance of discrete numerical schemes on a case by case basis. Indeed, while for general two-player games we cannot always write modified losses as for the zero-sum case -- see the Supplementary Material for a discussion -- we can use Theorems~\ref{thm:sim} and ~\ref{thm:alt} to understand the effect of the drift for specific choices of loss functions.

\begin{figure}[t]
  \centering
  \includegraphics[width=0.48\columnwidth]{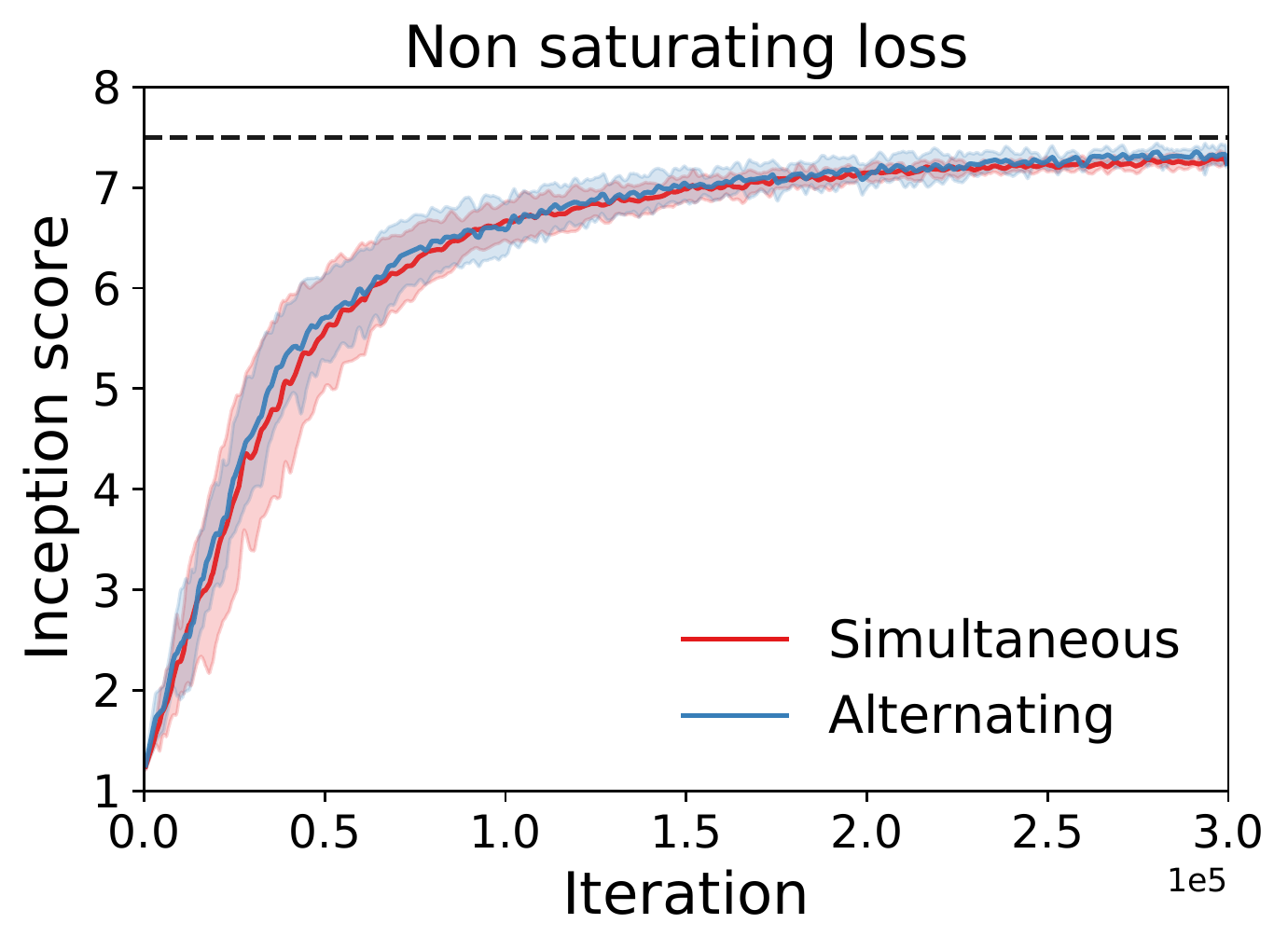}
  \includegraphics[width=0.48\columnwidth]{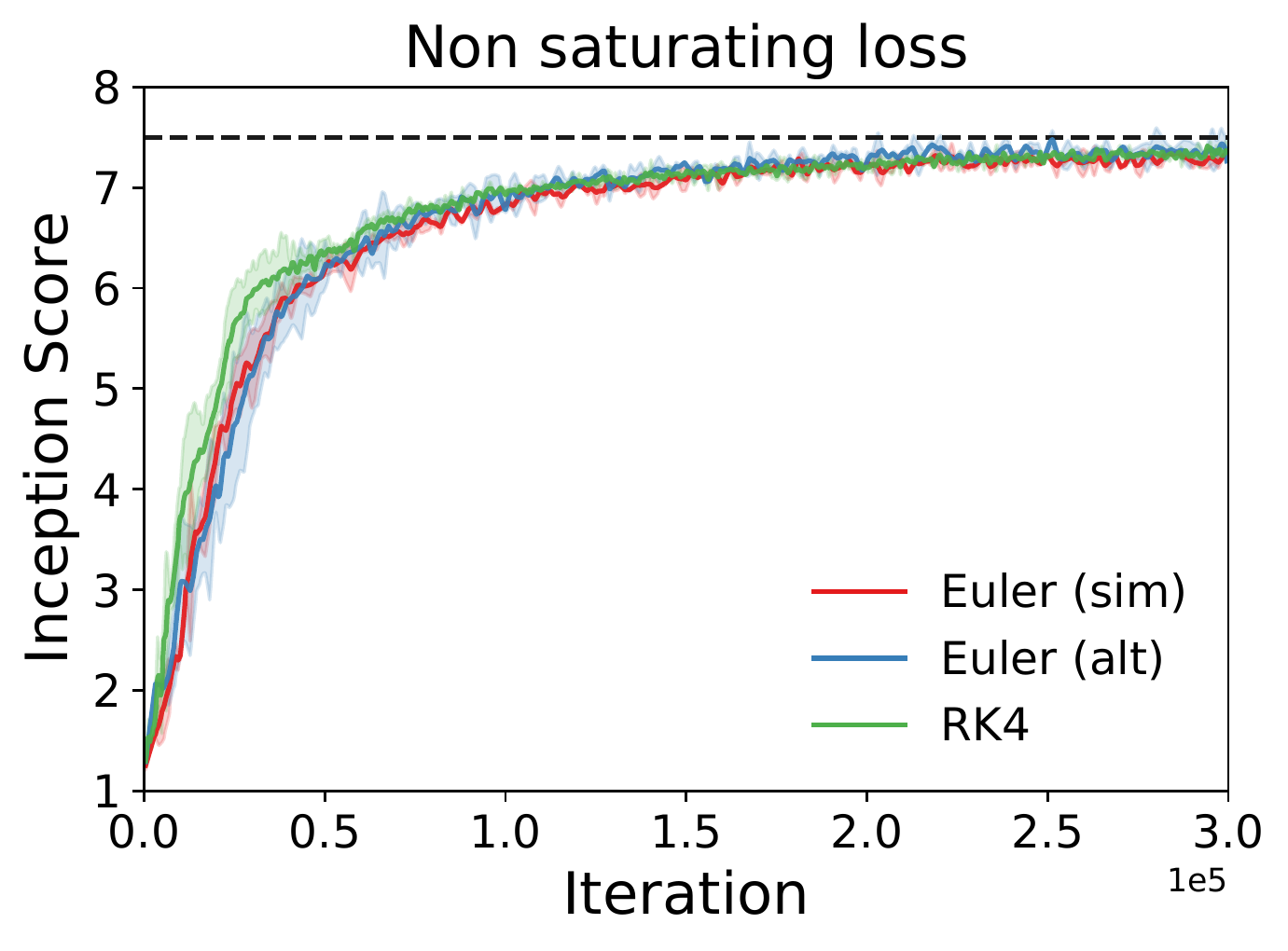}
  \caption{The effect of DD depends on the game: its effect is less strong for the non-saturating loss.}
  \label{fig:non_saturating}
\end{figure}

\section{Related work}

\begin{table*}[h]
\centering
\begin{tabular}{ |c|c|c|c| }
 \hline
  & Explicit & First player ($\vphi$) & Second player ($\vtheta$)\\
 \hline
 \hline
  DD (S) & \xmark & $\frac{\alpha h}{4}\norm{\nabla_{\vphi} E}^2 - \frac{\alpha h}{4}\norm{\nabla_{\vtheta} E}^2$  & $- \frac{\lambda h}{4}\norm{\nabla_{\vphi} E}^2 + \frac{\lambda h}{4}\norm{\nabla_{\vtheta} E}^2$ \\
 DD (A) & \xmark & $\frac {\alpha h}{4m}\norm{\nabla_{\vphi} E}^2 - \frac{\alpha h}{4}\norm{\nabla_{\vtheta} E}^2$ & $ \frac{(2 \alpha - \lambda) h}{4}\norm{\nabla_{\vphi} E}^2 + \frac{\lambda h}{4k}\norm{\nabla_{\vtheta} E}^2$  \\
  \hline
  Cancel DD interaction terms (S) & \checkmark &  $\frac{\alpha h}{4}\norm{\nabla_{\vtheta} E}^2$  & $ \frac{\lambda h}{4}\norm{\nabla_{\vphi} E}^2$ \\
 Cancel DD interaction terms (A) & \checkmark &  $\frac{\alpha h}{4}\norm{\nabla_{\vtheta} E}^2$  & $ -\frac{(2 \alpha - \lambda) h}{4}\norm{\nabla_{\vphi} E}^2$ \\
 SGA (S)  & \checkmark &$ \frac{1}{2} \norm{\nabla_{\vtheta} E}^2$ &  $ \frac{1}{2} \norm{\nabla_{\vphi} E}^2$ \\
 Consensus optimization (S) &\checkmark & $\eta \norm{\nabla_{\vphi} E}^2 + \eta \norm{\nabla_{\vtheta} E}^2$ & $\eta \norm{\nabla_{\vphi} E}^2 + \eta \norm{\nabla_{\vtheta} E}^2$ \\
Locally stable GAN (S) & \checkmark & \text{\sffamily X} &  $\eta \norm{\nabla_{\vphi} E}^2$ \\
 ODE-GAN (S)& \checkmark& $ \eta \norm{\nabla_{\vtheta} E}^2$ &  \text{\sffamily X} \\
 \hline
\end{tabular}
\caption{Comparing DD with explicit regularization methods in zero-sum games. SGA (without alignment, see the Supplementary Material and \citealt{balduzzi2018mechanics}), Consensus Optimization~\citep{mescheder2017numerics}, Locally Stable GAN~\citep{nagarajan2017gradient}, ODE-GAN~\citep{odegan}. We assume a $\min_{\vtheta} \min_{\vphi}$ game with learning rates $\alpha h$ and $\lambda h$  and number of updates $m$ and $k$, respectively. S and A denote simultaneous and alternating updates.}
\label{tab:methods_comp}
\end{table*}

{\bf Backward error analysis:} There have been a number of recent works applying backward error analysis to machine learning in the one-player case (supervised learning), which can potentially be extended to the two-players settings. In particular, \citet{igr_sgd} extended the analysis of~\citet{igr} to stochastic gradient descent.
\citet{symmetry} used modified gradient flow equations to show that discrete updates break certain conservation laws present in the gradient flows of deep learning models. \citet{franca2020} compare momentum and Nesterov acceleration using their modified equations. \citet{dissipative} used backward error analysis to help devise optimizers with a control on their stability and convergence rates. \citet{stochastic_adaptive_sgd} used modified equation techniques in the context of stochastic differential equations to devise optimizers with adaptive learning rates.
Other recent works such as \citet{grad_reg_avi} and \citet{fisher_explosion} noticed the strong impact of implicit gradient regularization in the training of over-parametrized models with SGD using different approaches. These works provide evidence that DD has a significant impact in deep learning, and that backward error analysis is a powerful tool to quantify it.
At last, note that a special case of Theorem~\ref{thm:2_players_sim} for zero-sum games with equal learning rates can be found in~\citet{lu2021resolution}.

{\bf Two-player games:} As one of the best-known examples at the interface of game theory and deep learning, GANs have been powered by gradient-based optimizers as other deep neural networks. The idea of an implicit regularization effect induced by simultaneous gradient descent in GAN training was first discussed in~\citet{schafer2019implicit}; the authors show empirically that the implicit regularization induced by gradient descent can have a positive effect on GAN training, and take a game theoretic approach to strengthening it using competitive gradient descent~\citep{competitive_sgd}. By quantifying the DD induced by gradient descent using backward error analysis, we shed additional light on the regularization effect that discrete updates have on GAN training, and show it has both beneficial and detrimental components (as also shown in ~\citet{daskalakis2018limit} with different methods in the case of simultaneous GD). Moreover, we show that the explicit regularization inspired by DD results in drastically improved performance.
The form of the modified losses we have derived are related to explicit regularization for GANs,
which has been one of the most efficient methods for stabilizing GANs as well as other games. Some of the regularizers constrain the complexity of the players~\citep{miyato2018spectral,gulrajani2017improved,biggan}, while others modify the dynamics for better convergence properties~\citep{odegan,mescheder2018training,nagarajan2017gradient,balduzzi2018mechanics,wang2019solving,mazumdar2019finding}.
Our approach is orthogonal in that, with backward analysis, we start from understanding the most basic gradient steps, underpinning any further modifications of the losses or gradient fields.
Importantly, we discovered a relationship between learning rates and the underlying regularization.
Since merely canceling the effect of DD is insufficient in practice (as also observed by \citet{odegan}), our approach complements regularizers that are explicitly designed to improve convergence.
We leave to future research to further study other regularizers in combination with our analysis. We summarize some of these regularizers in Table~\ref{tab:methods_comp}.

 To our knowledge, this is the first work towards finding continuous systems which better match the gradient descent updates used in two player games. Studying discrete algorithms using continuous systems has a rich history in optimization~\citep{su2016differential,wibisono2016variational}. Recently,
 models that directly parametrize differential equations demonstrated additional potential from bridging the discrete and continuous perspectives~\citep{chen2018neural,grathwohl2018ffjord}.

\section{Discussion}
We have shown that using modified continuous systems to quantify the discretization drift induced by gradient descent updates can help bridge the gap between the discrete optimization dynamics used in practice and the analysis of the original systems through continuous methods. This allowed us to cast a new light on the stability and performance of games trained using gradient descent,
and guided us towards explicit regularization strategies inspired by these modified systems.
We note however that DD merely modifies a game's original dynamics, and that the DD terms alone can not fully characterize a discrete scheme. In this sense, our approach only complements works analyzing the underlying game \citep{shoham2008multiagent}. Also, it is worth noting that our analysis is valid for learning rates small enough that errors of size $\mathcal O(h^3)$ can be neglected.

We have focused our empirical analysis on a few classes of two-player games, but the effect of DD will be relevant for all games trained using gradient descent. Our method can be expanded beyond gradient descent to other optimization methods such as Adam~\citep{kingma2014adam}, as well as to the stochastic setting as shown in~\citet{igr_sgd}.
We hope that our analysis of gradient descent provides %
a useful building block to further the understanding and performance of two-player games.

\textbf{Acknowledgments}. We would like to thank the reviewers and area chair for useful suggestions, and Alex Goldin, Chongli Qin, Jeff Donahue, Marc Deisenroth, Michael Munn, Patrick Cole, Sam Smith and Shakir Mohamed for useful discussions and support throughout this work.

\clearpage
\bibliography{merged_paper}
\bibliographystyle{icml2021_style/icml2021}

\appendix

\onecolumn

\icmltitle{Supplementary material for Discretization Drift in Two-Player Games}

\tableofcontents
\addtocontents{toc}{\protect\setcounter{tocdepth}{1}}

\renewcommand{\theequation}{A.\arabic{equation}}
\renewcommand\thefigure{A.\arabic{figure}}
\renewcommand\thetable{A.\arabic{table}}
\setcounter{figure}{0}
\setcounter{table}{0}
\setcounter{equation}{0}

\section{Proof of the main theorems}

\textbf{Notation}: We use $\vphi \in \mathbb{R}^m$ to denote the parameters of the first player
and $\vtheta \in \mathbb{R}^n$ for the second player. We assume that the players parameters are updated simultaneously
with learning rates $\alpha h$ and $\lambda h$ respectively.
We consider the vector fields $f(\vphi, \vtheta): \mathbb{R}^m \times \mathbb{R}^n \rightarrow \mathbb{R}^m$ and
$g(\vphi, \vtheta): \mathbb{R}^m \times \mathbb{R}^n \rightarrow \mathbb{R}^n$. Unless otherwise specified, we assume that vectors are row vectors.
We denote $\nabla_{\vx} f$ the transpose of the Jacobian of $f$ with respect to $\vx \in \{\vphi, \vtheta\}$, and similarly for $g$.
Thus $\nabla_{\vtheta} f(\vphi, \vtheta) \in \mathbb{R}^{n, m}$ denotes the matrix with entries $\big[\nabla_{\vtheta} f(\vphi, \vtheta)\big]_{i, j} = \frac{d f_j}{d \vtheta_{i}}$ with $i \in \{1, ..., n\}, j \in\{1, ..., m\}$. We use bold notation to denote vectors ---  $\vx$ is a vector while $x$ is a scalar.

We now prove Theorem 3.1 and Theorem 3.2 in the main paper, corresponding to the \textit{simultaneous} Euler updates and to the \textit{alternating} Euler updates, respectively.
In both cases, our goal is to find corrections $f_1$ and $g_1$ to the original system
\begin{align}
 \dot{\vphi} &=  f( \vphi, \vtheta), \label{eq:ode-1}\\
 \dot{\vtheta}  &= g( \vphi, \vtheta), \label{eq:ode-2}
\end{align}
such that the modified continuous system
\begin{align}
 \dot{\vphi} &=  f( \vphi, \vtheta)  + h f_1( \vphi, \vtheta), \\
 \dot{\vtheta}  &= g( \vphi, \vtheta) + h g_1( \vphi, \vtheta),
\end{align}
follows the discrete steps of the method with a local error of order $\mathcal O(h^3)$.
More precisely, if $(\vphi_{t}, \vtheta_{t})$ denotes the discrete step of the method at time $t$, and $(\tilde \vphi(s), \tilde \vtheta(s))$ corresponds to the continuous solution of the modified system above starting at $(\vphi_{t-1}, \vtheta_{t-1})$, we want that the local errors for both players, i.e.,
$$
\| \vphi_{t} - \tilde \vphi(\alpha h) \|
\qquad \textrm{ and } \qquad
\| \vtheta_{t} - \tilde \vtheta(\lambda h) \|
$$
to be of order $\mathcal O(h^3)$. In the expression above, $\alpha h$ and $\lambda h$ are the effective learning rates (or step-sizes) for the first and the second player respectively for both the simultaneous and alternating Euler updates.

Our proofs from backward error analysis follow the same steps:
\begin{enumerate}
  \item Expand the discrete updates to find a relationship between $\vphi_t$ and $\vphi_{t-1}$ and $\vtheta_t$ and $\vtheta_{t-1}$ up to order $\mathcal{O}(h^2)$.
  \item Expand the changes in continuous time of the modified ODE given by backward error analysis.
  \item Find the first order Discretization Drift (DD) by matching the discrete and continuous updates up to second order in learning rates.
\end{enumerate}

\textbf{Notation}: To avoid cluttered notations, we use $f_{(t)}$ to denote the $f(\vphi_t, \vtheta_t)$ and $g_{(t)}$ to denote $g(\vphi_t, \vtheta_t)$ for all $t$. If no index is specified, we use $f$ to denote $f(\vphi, \vtheta)$, where $\vphi$ and $\vtheta$ are variables in a continuous system.

\subsection{Simultaneous updates (Theorem 3.1)}
\label{sec:sim_updates}

Here we prove Theorem 3.1 in the main paper, which we reproduce here:

The simultaneous Euler updates with learning rates $\alpha h$ and $\lambda h$ respectively are given by:
\begin{align}
\vphi_t &= \vphi_{t-1} + \alpha h f(\vphi_{t-1}, \vtheta_{t-1} ) \label{eq:supp_simup1}\\
\vtheta_{t} &= \vtheta_{t-1} + \lambda h g(\vphi_{t-1},\vtheta_{t-1}) \label{eq:supp_simup2}
\end{align}
\begin{customthm}{3.1} \label{thm:supp_2_players_sim} The discrete \emph{simultaneous} Euler updates
in \eqref{eq:supp_simup1} and \eqref{eq:supp_simup2} follow the continuous system
\label{thm:supp_sim}
\begin{align*}
\dot{\vphi} &= f - \frac{\alpha h}{2} \left(f \nabla_{\vphi} f + g \nabla_{\vtheta} f \right) \\
\dot{\vtheta} &= g - \frac{\lambda h}{2} \left(f \nabla_{\vphi} g + g \nabla_{\vtheta} g \right)
\end{align*}
with an error of size $\mathcal O(h^3)$  after one update step.
\end{customthm}

\textbf{Step 1: Expand the updates per player via a Taylor expansion.} \\
We expand the numerical scheme to find a relationship between $\vphi_t$ and $\vphi_{t-1}$ and $\vtheta_t$ and $\vtheta_{t-1}$ up to order $\mathcal{O}(h^2)$.
Here in the case of the simultaneous Euler updates, this does not require any change to Equations~\eqref{eq:supp_simup1} and~\eqref{eq:supp_simup2}.
For the first player, the discrete Euler updates are:
\begin{align}
 \vphi_t &= \vphi_{t-1} + \alpha h f(\vphi_{t-1}, \vtheta_{t-1})
\label{eq:sim_disc_phi}
\end{align}

For the second player, the discrete Euler update has the same form:
\begin{align}
 \vtheta_t &= \vtheta_{t-1} + \lambda h g(\vphi_{t-1}, \vtheta_{t-1})
\label{eq:sim_disc_theta}
\end{align}

\textbf{Step 2: Expand the continuous time changes for the modified ODE given by backward error analysis} \\
We expand the changes in time of the modified ODE of the form:
\begin{align*}
\dot{\vphi} = \tilde{f}(\vphi, \vtheta) \\
\dot{\vtheta} = \tilde{g}(\vphi, \vtheta)\\
\end{align*}

where
\begin{align*}
  \tilde{f}(\vphi, \vtheta) = f(\vphi, \vtheta) + \sum_{i=1} {\tau_{\vphi}}^i f_i (\vphi, \vtheta) \\
  \tilde{g}(\vphi, \vtheta) = g(\vphi, \vtheta) + \sum_{i=1} {\tau_{\theta}}^i g_i (\vphi, \vtheta)
\end{align*}

our aim is to find $f_i$ and $g_i$ which match the discrete updates we have found above.

\begin{lemma}
If:
\begin{align*}
\dot{\vphi} = \tilde{f}(\vphi, \vtheta) \\
\dot{\vtheta} = \tilde{g}(\vphi, \vtheta)\\
\end{align*}
where
\begin{align*}
\tilde{f}(\vphi, \vtheta) = f(\vphi, \vtheta) + \sum_{i=1} {\tau_\phi}^i f_i (\vphi, \vtheta) \\
\tilde{g}(\vphi, \vtheta) = g(\vphi, \vtheta) + \sum_{i=1} {\tau_\theta}^i g_i (\vphi, \vtheta)
\end{align*}
and $\tau_{\theta}$ and $\tau_{\phi}$ are scalars. Then --- for ease of notation, we drop the argument $\tph$ on the evaluations of $\vphi$ and $\theta$ on the RHS:
\begin{align*}
\vphi(\tph + \tau_\phi) =  \vphi(\tph) + \tau_\phi f + \tau_\phi^2 f_1 + \tau_\phi^2\frac{1}{2} f \nabla_{\vphi} f + \tau_\phi^2\frac{1}{2} g \nabla_{\vtheta} f + \mathcal{O}(\tau_\phi^3)
\end{align*}
\label{thm:cont}
\end{lemma}

\begin{proof}
We expand to see what happens after 1 time step of $\tau_\phi$ by doing a Taylor expansion:
\begin{align*}
\vphi(\tph + \tau_\phi) &= \vphi + \tau_\phi\dot{\vphi} + \tau_\phi^2\frac{1}{2}\ddot{\vphi} + \mathcal{O}(\tau_\phi^3)\\
               &= \vphi + \tau_\phi \tilde{f} + \tau_\phi^2\frac{1}{2} \dot{\tilde{f}} +  \mathcal{O}(\tau_\phi^3) \\
               &= \vphi + \tau_\phi \tilde{f} + \tau_\phi^2\frac{1}{2} \left(\tilde{f} \nabla_{\vphi} \tilde{f} + \tilde{g} \nabla_{\vtheta} \tilde{f}\right)  + \mathcal{O}(\tau_\phi^3)\\
              &= \vphi + \tau_\phi (f + \tau_\phi f_1 + \mathcal{O}(\tau_\phi^2))+ \tau_\phi^2\frac{1}{2} \left(\tilde{f} \nabla_{\vphi} \tilde{f} + \tilde{g} \nabla_{\vtheta} \tilde{f}\right)  + \mathcal{O}(\tau_\phi^3)\\
              &= \vphi + \tau_\phi f + \tau_\phi^2 f_1 + \tau_\phi^2\frac{1}{2} \left(\tilde{f} \nabla_{\vphi} \tilde{f} + \tilde{g} \nabla_{\vtheta} \tilde{f}\right)  + \mathcal{O}(\tau_\phi^3)\\
              &= \vphi + \tau_\phi f + \tau_\phi^2 f_1 + \tau_\phi^2\frac{1}{2} \left( \left(f + \tau_\phi f_1 + \mathcal{O}(\tau_\phi^2)\right) \nabla_{\vphi} \tilde{f} + \left(g + \tau_\theta g_1 + \mathcal{O}(\tau_\theta^2)\right)\nabla_{\vtheta} \tilde{f}\right)  + \mathcal{O}(\tau_\phi^3)\\
              &= \vphi + \tau_\phi f + \tau_\phi^2 f_1 + \tau_\phi^2\frac{1}{2} f \nabla_{\vphi} \tilde{f} + \tau_\phi^2\frac{1}{2} g \nabla_{\vtheta} \tilde{f} + \mathcal{O}(\tau_\phi^3)\\
              &= \vphi + \tau_\phi f + \tau_\phi^2 f_1 + \tau_\phi^2\frac{1}{2} f \nabla_{\vphi} f + \tau_\phi^2\frac{1}{2} g \nabla_{\vtheta} f + \mathcal{O}(\tau_\phi^3)\\
\end{align*}
where we assumed that $\tau_\phi$ and $\tau_\theta$ are in the same order of magnitude.
\end{proof}

\textbf{Step 3: Matching the discrete and modified continuous updates.} \\
From Lemma~\ref{thm:cont}, we model how the continuous updates of the two players change in time for the modified  ODEs given by backward error analysis. To do so, we substitute the \textit{current values} as those given by the discrete updates, namely $\vphi_{t-1}$ and $\vtheta_{t-1}$, in order to calculate the displacement according to the continuous updates:

\begin{align}
\vphi(\tph + \tau_\phi) &= \vphi_{t-1} + \tau_\phi f_{(t-1)} + \tau_\phi^2 f_1(\vphi_{t-1}, \vtheta_{t-1}) + \frac{1}{2} \tau_\phi^2  f_{(t-1)} \nabla_{\vphi}  f_{(t-1)} + \frac{1}{2}\tau_\phi^2 g_{(t-1)} \nabla_\theta  f_{(t-1)} + \mathcal{O}(\tau_\phi^3)
\label{eq:sim_ode_phi} \\
\vtheta(\tph + \tau_\theta) &= \vtheta_{t-1} + \tau_\theta g_{(t-1)} + \tau_\theta^2 g_1(\vphi_{t-1}, \vtheta_{t-1}) + \frac{1}{2} \tau_\theta^2 f_{(t-1)} \nabla_{\vphi} g_{(t-1)} + \frac{1}{2}\tau_\theta^2 g_{(t-1)} \nabla_{\vtheta} g_{(t-1)} + \mathcal{O}(\tau_\theta^3)
\label{eq:sim_ode_theta}
\end{align}

In order to find $f_1$  and $g_1$ such that the continuous dynamics of the modified updates $f + \alpha h f_1$
and $g + \lambda h g_1$ match the discrete updates in Equations~\eqref{eq:sim_disc_phi} and ~\eqref{eq:sim_disc_theta}, we look for the corresponding continuous increments of the discrete updates in the modified continuous system, such that $\norm{\vphi(\tph + \tau_\phi)   - \vphi_t }$ and $\norm{\vtheta(\tph + \tau_\theta)   - \vtheta_t }$ are $\mathcal{O}(h^3)$.

 The first order terms in Equations~\eqref{eq:sim_disc_phi} and ~\eqref{eq:sim_disc_theta} and those in Equations~\eqref{eq:sim_ode_phi} and~\eqref{eq:sim_ode_theta} suggest that:
\begin{align*}
\alpha h = \tau_{\phi} \\
\lambda h = \tau_{\theta}
\end{align*}

We can now proceed to find $f_1$ and $g_1$ from the second order terms:
\begin{align*}
  0  &= \alpha^2 h^2 f_1(\vphi_{t-1}, \vtheta_{t-1}) + \frac{1}{2}  \alpha^2 h^2 f_{(t-1)} \nabla_{\vphi} f_{(t-1)} + \frac{1}{2} \alpha^2 h^2 g_{(t-1)} \nabla_{\vtheta} f_{(t-1)} \\
  f_1(\vphi_{t-1}, \vtheta_{t-1}) &= - \frac{1}{2} f_{(t-1)} \nabla_{\vphi} f_{(t-1)} - \frac{1}{2} g_{(t-1)} \nabla_{\vtheta} f_{(t-1)} \\
\end{align*}

Similarly, for $g_1$ we obtain:
\begin{align*}
g_1(\vphi_{t-1}, \vtheta_{t-1}) &= - \frac{1}{2} f_{(t-1)} \nabla_{\vphi}g_{(t-1)} - \frac{1}{2} g_{(t-1)} \nabla_{\vtheta} g_{(t-1)} \\
\end{align*}

Leading to the first order corrections:
\begin{align}
f_1(\vphi_{t-1}, \vtheta_{t-1}) &= - \frac{1}{2} f_{(t-1)} \nabla_{\vphi}f_{(t-1)} - \frac{1}{2} g_{(t-1)} \nabla_{\vtheta} f_{(t-1)} \\
g_1(\vphi_{t-1}, \vtheta_{t-1}) &= - \frac{1}{2} f_{(t-1)} \nabla_{\vphi}g_{(t-1)} - \frac{1}{2} g_{(t-1)} \nabla_{\vtheta} g_{(t-1)}
\label{eq:sim_geneal_eq}
\end{align}

We have found the functions $f_1$ and $g_1$ such that after one discrete optimization step the ODEs $\dot{\vphi} = f + \alpha h f_1$ and $\dot{\vtheta} = g + \lambda h g_1$ follow the discrete updates up to order $\mathcal{O}(h^3)$, finishing the proof.

\subsection{Alternating updates (Theorem 3.2)}
\label{sec:multiple_updates}

Here we prove Theorem 3.2 in the main paper, which we reproduce here:

For the \textit{alternating Euler updates}, the players take turns to update their parameters, and can perform multiple updates each. We denote the number of alternating updates of the first  player (resp. second player) by $m$ (resp. $k$).
We scale the learning rates by the number of updates, leading to the following updates $ \vphi_{t} := \vphi_{m,t} $ and $\vtheta_{t} := \vtheta_{k, t}$ where
\begin{align}
 \vphi_{i, t} &=  \vphi_{i-1, t} + \frac{\alpha h}{m}   f(\vphi_{i-1,t}, \vtheta_{t-1}) , \hspace{1em} i = 1 \dots m, \label{eq:supp_altup1} \\
 \vtheta_{j, t} &=  \vtheta_{j-1, t} + \frac{\lambda h}{k} g(\vphi_{m, t}, \vtheta_{j-1, t}), \hspace{1em} j = 1 \dots k. \label{eq:supp_altup2}
\end{align}
\begin{customthm}{3.2} \label{thm:supp_2_players_alt}
The discrete \emph{alternating} Euler updates in \eqref{eq:supp_altup1} and \eqref{eq:supp_altup2} follow the continuous system
\begin{align*}
\dot{\vphi} &= f - \frac{\alpha h}{2} \left(\frac{1}{m}f \nabla_{\vphi} f + g \nabla_{\vtheta} f \right) \\
\dot{\vtheta} &= g - \frac{\lambda h}{2} \left((1- \frac {2\alpha} {\lambda})f \nabla_{\vphi} g + \frac{1}{k} g \nabla_{\vtheta} g \right)
\end{align*}
with an error of size $\mathcal O(h^3)$ after one update step.
\end{customthm}

\textbf{Step 1: Discrete updates} \\
In the case of alternating Euler discrete updates, we have:
\begin{align*}
 \vphi_{1,t} &=  \vphi_{t-1} + \frac{\alpha}{m} h f(\vphi_{t-1}, \vtheta_{t-1}) \\
 \vphi_{2,t} &=  \vphi_{1,t} + \frac{\alpha}{m} h f(\vphi_{1,t}, \vtheta_{t-1}) \\
 \dots\\
 \vphi_{m,t} &=  \vphi_{m-1,t} + \frac{\alpha}{m} h f(\vphi_{m-1, t}, \vtheta_{t-1}) \\
 \vtheta_{1,t} &=  \vtheta_{t-1} + \frac{\lambda}{k} h g(\textcolor{black}{\vphi_{m,t}}, \vtheta_{t-1}) \\
 \vtheta_{2,t} &=  \vtheta_{1, t-1} + \frac{\lambda}{k} h g(\vphi_{m,t}, \vtheta_{1, t}) \\
 \dots \\
 \vtheta_{k,t} &=  \vtheta_{k-1, t-1} + \frac{\lambda}{k} h g(\vphi_{m,t}, \vtheta_{k-1, t}) \\
 \vphi_t &= \vphi_{m,t} \\
 \vtheta_t &= \vtheta_{k,t}
\end{align*}

\begin{lemma}
For update with $\vphi_{m,t} = \vphi_{m-1, t} + h f(\vphi_{m-1, t}, \vtheta_{t-1})$ with step size $h$, the $m$-step update has the form:
\begin{align*}
 \vphi_{m,t} &= \vphi_{t-1} + mh f_{(t-1)} + \frac{m (m-1)}{2}h^2 f_{(t-1)}\nabla_{\vphi}f_{(t-1)}  + \mathcal{O}(h^3)\\
\end{align*}

\begin{proof}
Proof by induction.
\label{thm:discrete}
Base step.

\begin{align*}
 \vphi_{2,t} &=  \vphi_{1,t} + h f(\vphi_{1,t}, \vtheta_{t-1}) \\
            &=  \vphi_{t-1} +  h f(\vphi_{t-1}, \vtheta_{t-1}) + h f(\vphi_{1,t}, \vtheta_{t-1}) \\
             &= \vphi_{t-1} + h f_{(t-1)} + h f\big(\vphi_{t-1} + h f_{(t-1)}, \vtheta_{t-1} \big) \\
             &= \vphi_{t-1} + h f_{(t-1)} + h \left(f_{(t-1)} + h f_{(t-1)} \nabla_{\vphi}f_{(t-1)} + \mathcal{O}(h^2) \right) \\
             &= \vphi_{t-1} + 2h f_{(t-1)} + h^2 f_{(t-1)} \nabla_{\vphi}f_{(t-1)} + \mathcal{O}(h^3) \\
\end{align*}

Inductive step:
\begin{align*}
 \vphi_{m+1, t} &= \vphi_{m, t} + h f(\vphi_{m, t}, \vtheta_{t-1}) \\
                 &= \vphi_{t-1} + mh f_{(t-1)} + \frac{m (m-1)}{2}h^2 f_{(t-1)} \nabla_{\vphi}f_{(t-1)} \\
                 &\quad + h f\Big(\vphi_{t-1} + mh f_{(t-1)} + \frac{m (m-1)}{2}h^2 f_{(t-1)} \nabla_{\vphi}f_{(t-1)} + \mathcal{O}(h^3), \vtheta_{t-1}\Big) + \mathcal{O}(h^3) \\
                &= \vphi_{t-1} + mh f_{(t-1)} + \frac{m (m-1)}{2}h^2 f_{(t-1)} \nabla_{\vphi}f_{(t-1)}\\
                &\quad + h \left(f_{(t-1)} + (mh f_{(t-1)} + \frac{m (m-1)}{2}h^2 f_{(t-1)} \nabla_{\vphi}f_{(t-1)}) \nabla_{\vphi}f_{(t-1)} \right) + \mathcal{O}(h^3) \\
                &= \vphi_{t-1} + (m +1)h f_{(t-1)} + \frac{m (m-1)}{2}h^2 f_{(t-1)} \nabla_{\vphi}f_{(t-1)} \\
                &\quad + h \left(mh f_{(t-1)} + \frac{m (m-1)}{2}h^2 f_{(t-1)} \nabla_{\vphi}f_{(t-1)})\right)\nabla_{\vphi}f_{(t-1)} + \mathcal{O}(h^3) \\
                 &= \vphi_{t-1} + (m +1)h f_{(t-1)} + \frac{m (m+1)}{2}h^2 f_{(t-1)} \nabla_{\vphi}f_{(t-1)} + \mathcal{O}(h^3)
\end{align*}
\end{proof}
\end{lemma}

From Lemma~\ref{thm:discrete} with  $h = \alpha h/m$ we have that:
\begin{align*}
 \vphi_{m,t} &= \phi_{t-1} + \alpha h f_{(t-1)} + \frac{m-1}{2m} \alpha^2 h^2 f_{(t-1)} \nabla_{\vphi}f_{(t-1)}  + \mathcal{O}(h^3)\\
\end{align*}

We now turn our attention to the second player.
We define $g'_t = g(\vphi_{m, t}, \vtheta_{t-1} )$.
This is where we get the difference between simultaneous and alternating updates comes in.
From Lemma~\ref{thm:discrete} with $h = \lambda h/k$ we have that:
\begin{align*}
 \vtheta_{k, t} &= \vtheta_{t-1} + \lambda h g'_t + \frac{(k - 1)}{2k}\lambda^2 h^2 g'_t \nabla_{\vtheta}g'_t + \mathcal{O}(h^3)
\end{align*}

We now expand $g'_t$ by Taylor expansion:
\begin{align*}
g'_t &= g(\vphi_{m, t}, \vtheta_{t-1}) \\
     &= g(\vphi_{t-1} +  \alpha h f_{(t-1)} +  \alpha^2 \frac{m-1}{2m}h^2 f_{(t-1)} \nabla_{\vphi}f_{(t-1)} + \mathcal{O}(h^3), \vtheta_{t-1}) \\
     &= g_{(t-1)} + \left( \alpha h f_{(t-1)} + \frac{(m -1)}{2m} \alpha^2 h^2 f_{(t-1)} \nabla_{\vphi}f_{(t-1)} + \mathcal{O}(h^3)\right) \nabla_{\vphi} g_{(t-1)}
\end{align*}

Thus, if we expand the RHS:
\begin{align*}
 \vtheta_{k, t} &= \vtheta_{t-1} + \lambda h g'_{t-1} + \frac{k-1}{2k}\lambda^2 h^2 g'_{t-1} \nabla_{\vtheta}g'_{t-1} + \mathcal{O}(h^3) \\
                 &= \vtheta_{t-1} + \lambda h \left(g_{(t-1)} + \left(\alpha h f_t + \frac{m -1}{2m}\alpha^2 h^2 f_{(t-1)} \nabla_{\vphi}f_{(t-1)} + \mathcal{O}(h^3)\right) \nabla_{\vphi} g_{(t-1)} \right) + \frac{k-1}{2k}\lambda^2 h^2 g'_{t-1} \nabla_{\vtheta}g'_{t-1} + \mathcal{O}(h^3) \\
                 &= \vtheta_{t-1} + \lambda h g_{(t-1)} + \lambda  h\left(\alpha h f_{(t-1)} + \frac{m -1}{2m}\alpha^2 h^2 f_{(t-1)} \nabla_{\vphi}f_{(t-1)} \right) \nabla_{\vphi} g_{(t-1)} + \frac{k-1}{2k}\lambda^2 h^2 g'_{t-1} \nabla_{\vtheta}g'_{t-1} + \mathcal{O}(h^3) \\
                &= \vtheta_{t-1} + \lambda h g_{(t-1)} + \lambda \alpha h^2 f_{(t-1)} \nabla_{\vphi} g_{(t-1)} + \frac{k-1}{2k}\lambda^2 h^2 g'_{t-1} \nabla_{\vtheta}g'_{t-1} + \mathcal{O}(h^3) \\
                &= \vtheta_{t-1} +  \lambda h g_{(t-1)} + \lambda \alpha h^2 f_{(t-1)} \nabla_{\vphi} g_{(t-1)} \\
                &\quad + \frac{k-1}{2k} h^2 \left(g_t + \left(\alpha h f_{(t-1)} + \frac{(m -1)}{2m} \alpha^2 h^2 f_{(t-1)} \nabla_{\vphi}f_{(t-1)} + \mathcal{O}(h^3)\right) \nabla_{\vphi} g_{(t-1)}\right) \nabla_{\vtheta}g'_{t-1} + \mathcal{O}(h^3) \\
                 &= \vtheta_{t-1} + \lambda h g_{(t-1)} + \lambda \alpha h^2 f_{(t-1)} \nabla_{\vphi} g_{(t-1)} + \frac{k-1}{2k}\lambda^2 h^2 g_{(t-1)} \nabla_{\vtheta}g'_{t-1} + \mathcal{O}(h^3) \\
                &= \vtheta_{t-1} + \lambda h g_{(t-1)} + \lambda \alpha h^2 f_{(t-1)} \nabla_{\vphi} g_{(t-1)} \\
                &\quad + \frac{k-1}{2k}\lambda^2 h^2 g_{(t-1)} \nabla_{\vtheta} \left(g_{(t-1)} + \left(\alpha h f_{(t-1)} + \frac{(m -1)}{2m}\alpha^2 h^2 f_{(t-1)} \nabla_{\vphi}f_{(t-1)} + \mathcal{O}(h^3)\right) \nabla_{\vphi} g_{(t-1)}\right) + \mathcal{O}(h^3) \\
                &= \vtheta_{t-1} +\lambda h g_{(t-1)} +\lambda \alpha h^2 f_{(t-1)} \nabla_{\vphi} g_{(t-1)} + \frac{k-1}{2k}\lambda^2 h^2 g_{(t-1)} \nabla_{\vtheta} g_{(t-1)}  + \mathcal{O}(h^3)
\end{align*}

We then have:
\begin{align}
 \vphi_{m,t} &= \vphi_{t-1} + \alpha h f_{(t-1)} + \frac{m-1}{2m}\alpha^2 h^2 f_{(t-1)} \nabla_{\vphi}f_{(t-1)}  + \mathcal{O}(h^3)
 \label{eq:alt_general_learning_rates_discrete_phi} \\
 \vtheta_{k,t} &= \vtheta_{t-1} + \lambda h g_{(t-1)} + \lambda \alpha h^2 f_{(t-1)} \nabla_{\vphi} g_{(t-1)} + \frac{k-1}{2k} \lambda^2 h^2 g_{(t-1)} \nabla_{\vtheta} g_{(t-1)}  + \mathcal{O}(h^3)
 \label{eq:alt_general_learning_rates_discrete_theta}
\end{align}

\textbf{Step 2: Expand the continuous time changes for the modified ODE given by backward error analysis} \\
(Identical to the simultaneous update case.)

\textbf{Step 3: Matching the discrete and modified continuous updates.} \\
The linear terms are identical to those in the simultaneous updates, which we reproduce here:
\begin{align*}
\tau_\phi = \alpha h  \\
\tau_\theta = \lambda h \\
\end{align*}

We can then obtain $f_1$ from matching the quadratic terms in Equations~\eqref{eq:sim_ode_phi} and Equations~\eqref{eq:alt_general_learning_rates_discrete_phi} --- below we denote $f_1(\vphi_{t-1}, \vtheta_{t-1})$ by $f_1$ and $g_1(\vphi_{t-1}, \vtheta_{t-1})$ by $g_1$, for brevity:
\begin{align*}
\alpha^2 h^2 f_1 + \frac{1}{2} \alpha^2 h^2 f_{(t-1)} \nabla_{\vphi} f_{(t-1)} + \frac{1}{2} \alpha^2 h^2  g_{(t-1)} \nabla_{\vtheta} f_{(t-1)} &=  \frac{m-1}{2m} \alpha^2 h^2 f_{(t-1)} \nabla_{\vphi}f_{(t-1)} \\
 f_1 + \frac{1}{2}  f_{(t-1)} \nabla_{\vphi} f_{(t-1)} + \frac{1}{2}   g_{(t-1)} \nabla_{\vtheta} f_{(t-1)} &=  \frac{m-1}{2m}  f_{(t-1)} \nabla_{\vphi}f_{(t-1)} \\
  f_1 &=  \left(\frac{m-1}{2m} -\frac{1}{2}\right) f_{(t-1)} \nabla_{\vphi}f_{(t-1)}  - \frac{1}{2}g_{(t-1)} \nabla_{\vtheta} f_{(t-1)} \\
f_1 &=  - \frac{1}{2m}f_{(t-1)} \nabla_{\vphi}f_{(t-1)} - \frac{1}{2}g_{(t-1)} \nabla_{\vtheta} f_{(t-1)}
\end{align*}

For $g_1$, from Equations~\eqref{eq:sim_ode_theta} and \eqref{eq:alt_general_learning_rates_discrete_theta}:
\begin{align*}
\lambda^2 h^2 g_1 + \lambda^2 h^2 \frac{1}{2} f_{(t-1)} \nabla_{\vphi} g_{(t-1)} + \lambda^2 h^2 \frac{1}{2} g_{(t-1)} \nabla_{\vtheta} g_{(t-1)} &= \lambda \alpha h^2 f_{(t-1)} \nabla_{\vphi} g_{(t-1)} + \frac{(k-1)}{2k} \lambda^2 h^2 g_{(t-1)} \nabla_{\vtheta} g_{(t-1)} \\
 g_1 +  \frac{1}{2} f_{(t-1)} \nabla_{\vphi} g_{(t-1)} +  \frac{1}{2} g_{(t-1)} \nabla_{\vtheta} g_{(t-1)} &= \frac{\alpha}{\lambda} f_{(t-1)} \nabla_{\vphi} g_{(t-1)} + \frac{(k-1)}{2k}  g_{(t-1)} \nabla_{\vtheta} g_{(t-1)} \\
  g_1 +  \frac{1}{2} f_{(t-1)} \nabla_{\vphi} g_{(t-1)} +  \frac{1}{2} g_{(t-1)} \nabla_{\vtheta} g_{(t-1)} &= \frac{\alpha}{\lambda} f_{(t-1)} \nabla_{\vphi} g_{(t-1)} + \frac{(k-1)}{2k}  g_{(t-1)} \nabla_{\vtheta} g_{(t-1)} \\
   g_1 &= \left(\frac{\alpha}{\lambda} -\frac{1}{2}\right) f_{(t-1)} \nabla_{\vphi} g_{(t-1)} - \frac{1}{2k}  g_{(t-1)} \nabla_{\vtheta} g_{(t-1)}
\end{align*}

We thus have that:
\begin{align}
f_1(\vphi_{t-1}, \vtheta_{t-1}) &=  - \frac{1}{2m}f_{(t-1)} \nabla_{\vphi}f_{(t-1)} - \frac{1}{2}g_{(t-1)} \nabla_{\vtheta} f_{(t-1)}
\label{eq:alternating_f1}\\
g_1(\vphi_{t-1}, \vtheta_{t-1}) &= \left(\frac{\alpha}{\lambda} -\frac{1}{2}\right) f_{(t-1)} \nabla_{\vphi} g_{(t-1)} - \frac{1}{2k}  g_{(t-1)} \nabla_{\vtheta} g_{(t-1)}
\label{eq:alternating_g1}
\end{align}

We have found the functions $f_1$ and $g_1$ such that after one discrete optimization step the ODEs $\dot{\vphi} = f + \alpha h f_1$ and $\dot{\vtheta} = g + \lambda h g_1$ follow the discrete updates up to order $\mathcal{O}(h^3)$, finishing the proof.

\section{Proof of the main corollaries}

In this section, we will write the modified equations in the case of using gradient descent common-payoff games and zero-sum games. This amounts to specialize the following first order corrections we have derived in the previous sections.

To do so, we will replace $f$ and $g$ for the values given by gradient descent, eg. in the common pay-off case $f = - \nabla_{\vphi} E$ and $g = -\nabla_{\vtheta} E$ and  $f = \nabla_{\vphi} E$ and $g = -\nabla_{\vtheta} E$ where $E(\vphi, \vtheta)$ is a function of the player parameters. We will use the following identities:
\begin{eqnarray*}
\nabla_{\vphi} E\nabla_{\vphi} \nabla_{\vphi} E        &=  \nabla_{\vphi} \Big(\frac{\|\nabla_{\vphi} E\|^2}2\Big), \hspace{2em}
\nabla_{\vtheta} E\nabla_{\vtheta} \nabla_{\vphi} E    &=  \nabla_{\vphi} \Big(\frac{\|\nabla_{\vtheta} E\|^2}2\Big) \\
\nabla_{\vphi} E\nabla_{\vphi} \nabla_{\vtheta} E     &=  \nabla_{\vtheta} \Big(\frac{\|\nabla_{\vphi} E\|^2}2\Big), \hspace{2em}
\nabla_{\vtheta} E\nabla_{\vtheta} \nabla_{\vtheta} E  &=  \nabla_{\vtheta} \Big(\frac{\|\nabla_{\vtheta} E\|^2}2\Big)
\end{eqnarray*}

\subsection{Common-payoff alternating two player-games (Corollary 5.1)}

\begin{customcor}{5.1} In a two-player common-payoff game with common loss $E$, \textit{alternating} gradient descent -- as described in Equations~\eqref{eq:supp_altup1} and ~\eqref{eq:supp_altup2} - with one update per player follows a gradient flow given by the modified losses
\begin{align}
\tilde L_1&= E + \frac{\alpha h}{4} \left(\norm{\nabla_{\vphi} E}^2  + \norm{\nabla_{\vtheta} E}^2\right) \\
\tilde L_2&=  E + \frac{\lambda h}{4} \left((1 - \frac {2 \alpha} {\lambda}) \norm{\nabla_{\vphi} E}^2  + \norm{\nabla_{\vtheta} E}^2\right)
\end{align}
with an error of size $\mathcal O(h^3)$  after one update step.
\end{customcor}

In the common-payoff case, both players minimize the same loss function $E$.
Substituting $f=-\nabla_{\vphi} E$ and $g=-\nabla_{\vtheta} E$ into the corrections $f_1$ and $g_1$ for the alternating Euler updates in Theorem 3.2
and using the gradient identities above yields
\begin{align*}
f_1 & = - \frac{1}{2} \left(f \nabla_{\vphi} f + g \nabla_{\vtheta} f \right) \\
    &=  -\frac{1}2\left(
                                \frac 1m \nabla_{\vphi} E\nabla_{\vphi} \nabla_{\vphi} E
                                + \nabla_{\vtheta} E\nabla_{\vtheta} \nabla_{\vphi} E
                      \right),\\
        &= -\frac{1}2\left(
                                \frac 1m \nabla_{\vphi} \frac{\|\nabla_{\vphi} E\|^2}2\
                                +\nabla_{\vphi} \frac{\|\nabla_{\vtheta} E\|^2}2
                      \right),\\
    &= -\nabla_{\vphi} \left(\frac{1}{4m} \|\nabla_{\vphi} E\|^2  + \frac{1}{4} \|\nabla_{\vtheta} E\|^2 \right) \\ \\
g_1 &= - \frac{1}{2} \left((1- \frac {2\alpha} {\lambda})f \nabla_{\vphi} g + \frac{1}{k} g \nabla_{\vtheta} g \right) \\
 &= -\frac{1}2\left(
                                (1 - \frac{2\alpha}\lambda) \nabla_{\vphi} E\nabla_{\vphi} \nabla_{\vtheta} E
                                + \frac 1k \nabla_{\vtheta} E\nabla_{\vtheta} \nabla_{\vtheta} E
                       \right),\\
    &= -\frac{1}2\left(
                                (1 - \frac{2\alpha}\lambda) \nabla_{\vtheta} \frac{\|\nabla_{\vphi} E\|^2}2
                                + \frac 1k\nabla_{\vtheta}\frac{\|\nabla_{\vtheta} E\|^2}2
                       \right),\\
    &= -\nabla_{\vtheta} \left(
                                \frac{1}4 (1 - \frac{2\alpha}\lambda)  \|\nabla_{\vphi} E\|^2
                                + \frac {1}k\|\nabla_{\vtheta} E\|^2
                         \right)
\end{align*}

Now, replacing the gradient expressions for $f_1$ and $g_1$ calculated above into the modified equations $\dot \vphi = -\nabla_{\vphi} E + \alpha hf_1$ and $\dot \vtheta = -\nabla_{\vtheta} E + \lambda hg_1$ and factoring out the gradients, we obtain the modified equations in the form of the ODEs:
\begin{align}
\dot \vphi   &= -\nabla_{\vphi} \widetilde L_1, \\
\dot \vtheta &= -\nabla_{\vtheta} \widetilde L_2,
\end{align}
with the following modified losses for each players:
\begin{align}
\widetilde L_1 & = E + \frac{\alpha h}4  \left(\frac 1m\|\nabla_{\vphi} E\|^2 + \|\nabla_{\vtheta} E\|^2\right), \\
\widetilde L_2 & = E + \frac{\lambda h}4 \left( (1 - \frac{2\alpha}{\lambda} )\|\nabla_{\vphi} E\|^2 + \frac 1k \|\nabla_{\vtheta} E\|^2 \right).
\end{align}
We obtain Corollary 5.1 by setting the number of player updates to one: $m=k=1$.

\subsection{Zero-sum simultaneous two player-games (Corollary 6.1)}
\label{sec:zero-sum-sim}

\begin{customcor}{6.1} In a zero-sum two-player differentiable game, \textit{simultaneous} gradient descent updates - as described in Equations~\eqref{eq:supp_simup1} and ~\eqref{eq:supp_simup2} - follows a gradient flow given by the modified losses
\begin{align}
\tilde L_1&= - E + \frac{\alpha h}{4} \norm{\nabla_{\vphi} E}^2  -  \frac{\alpha h}{4} \norm{\nabla_{\vtheta} E}^2
, \\
\tilde L_2&= E -  \frac{\lambda h}{4} \norm{\nabla_{\vphi} E}^2  +  \frac{\lambda h}{4} \norm{\nabla_{\vtheta} E}^2,
\end{align}
with an error of size $\mathcal O(h^3)$ after one update step.
\label{cor:zs-sim}
\end{customcor}

In this case, substituting $f = \nabla_{\vphi} E$ and $g=-\nabla_{\vtheta} E$ into the corrections $f_1$ and $g_1$ for the simultaneous Euler updates and using the gradient identities above yields
\begin{align*}
f_1 & = - \frac{1}{2} \left(f \nabla_{\vphi} f + g \nabla_{\vtheta} f \right) \\
    &=  -\frac{1}2\left(
                                \nabla_{\vphi} E\nabla_{\vphi} \nabla_{\vphi} E
                                - \nabla_{\vtheta} E\nabla_{\vtheta} \nabla_{\vphi} E
                      \right),\\
    &= -\frac{1}2\left(
                                  \nabla_{\vphi} \frac{\|\nabla_{\vphi} E\|^2}2
                                - \nabla_{\vphi} \frac{\|\nabla_{\vtheta} E\|^2}2
                      \right),\\
    &= -\nabla_{\vphi} \left(\frac{1}{4} \|\nabla_{\vphi} E\|^2  - \frac{1}{4} \|\nabla_{\vtheta} E\|^2 \right) \\ \\
g_1 &=  - \frac{1}{2} \left(f \nabla_{\vphi} g + g \nabla_{\vtheta} g \right) \\
    &=  - \frac{1}{2} \left(- \nabla_{\vphi} E\nabla_{\vphi} \nabla_{\vtheta} E + \nabla_{\vtheta} E\nabla_{\vtheta} \nabla_{\vtheta} E\right) \\
    &= -\frac{1}2\left(
                                - \nabla_{\vtheta} \frac{\|\nabla_{\vphi} E\|^2}2
                                + \nabla_{\vtheta}\frac{\|\nabla_{\vtheta} E\|^2}2
                       \right),\\
    &= -\nabla_{\vtheta} \left(
                                - \frac{1}4 \|\nabla_{\vphi} E\|^2
                                + \frac{1}4 \|\nabla_{\vtheta} E\|^2
                         \right)
\end{align*}

Now, replacing the gradient expressions for $f_1$ and $g_1$ calculated above into the modified equations $\dot \vphi = -\nabla_{\vphi}(-E) + \alpha hf_1$ and $\dot \vtheta = -\nabla_{\vtheta} E + \lambda hg_1$ and factoring out the gradients, we obtain the modified equations in the form of the ODEs:
\begin{align}
\dot \vphi   &= -\nabla_{\vphi} \widetilde L_1, \\
\dot \vtheta &= -\nabla_{\vtheta} \widetilde L_2,
\end{align}

with the following modified losses for each players:
 \begin{align}
\widetilde L_1 & = -E + \frac{\alpha h}{4} \|\nabla_{\vphi} E\|^2  - \frac{\alpha h}{4} \|\nabla_{\vtheta} E\|^2 , \\
\widetilde L_2 & = E  - \frac{\lambda h}4 \|\nabla_{\vphi} E\|^2 + \frac{\lambda h}4 \|\nabla_{\vtheta} E\|^2 .
 \end{align}

\subsection{Zero-sum alternating two-player games (Corollary 6.2)}

\begin{customcor}{6.2} In a zero-sum two-player differentiable game, \textit{alternating} gradient descent - as described in Equations~\eqref{eq:supp_altup1} and ~\eqref{eq:supp_altup2} - follows a gradient flow given by the modified losses
\begin{align}
\tilde L_1&= -E + \frac{\alpha h}{4m} \norm{\nabla_{\vphi} E}^2  - \frac{\alpha h}{4} \norm{\nabla_{\vtheta} E}^2
\\
\tilde L_2&= E - \frac{\lambda h}{4} (1 - \frac{2 \alpha}{\lambda}) \norm{\nabla_{\vphi} E}^2  + \frac{\lambda h}{4k}\norm{\nabla_{\vtheta} E}^2
\end{align}
with an error of size $\mathcal O(h^3)$ after one update step.
\end{customcor}

In this last case, substituting $f=\nabla_{\vphi} E$ and $g=-\nabla_{\vtheta} E$ into the corrections $f_1$ and $g_1$ for the alternating Euler updates and using the gradient identities above yields the modified system as well as the modified losses exactly in the same way as for the two previous cases above. (This amounts to a single sign change in the proof of Corollary 5.1)

\subsection{Self and interaction terms in zero-sum games}

\begin{remark}
Throughout the Supplementary Material, we will refer to self terms and interaction terms, as originally defined in our paper (Definition 3.1), and we will also use this terminology to refer to terms in our derivations that originate from the self terms and interaction terms. While a slight abuse of language, we find it useful to emphasize the provenance of these terms in our discussion.\end{remark}

For the case of zero-sum games trained with simultaneous gradient descent, the self terms encourage the minimization of the player's own gradient norm, while the interaction terms encourage the maximization of the other player's gradient norm:
\begin{align*}
\widetilde L_1 & = -E + \underbrace{\frac{\alpha h}{4} \|\nabla_{\vphi} E\|^2}_{self} \underbrace{- \frac{\alpha h}{4} \|\nabla_{\vtheta} E\|^2}_{interaction}, \\
\widetilde L_2 & = E  \underbrace{- \frac{\lambda h}4 \|\nabla_{\vphi} E\|^2}_{interaction} \underbrace{+ \frac{\lambda h}4 \|\nabla_{\vtheta} E\|^2 }_{self}.
\end{align*}

Similar terms are obtained for alternating gradient descent, with the only difference that the sign of the \textit{interaction} term for the second player can change and become positive.

\section{General differentiable two-player games}

Consider now the case where we have two loss functions for the two players respectively $L_1(\vphi, \vtheta): \mathbb{R}^m\times \mathbb{R}^n \rightarrow \mathbb{R}$ and $L_2(\vphi, \vtheta): \mathbb{R}^m\times \mathbb{R}^n \rightarrow \mathbb{R}$.
This leads to the update functions $f = -\nabla_{\vphi} L_1$ and $g = -\nabla_{\vtheta} L_2$.

We show below that in the most general case \textit{we cannot write the modified updates as gradient}. That is, in the general case we cannot write $f + h f_1$ as $\nabla_{\vphi} \tilde{L}_1$, since $f_1$ will not be the gradient of a function, and similarly for $g_1$. Consequently, if we want to study the effect of DD on general games, we have to work at the level of the \textit{modified vector fields} $f + hf_1$ and $g + hg_1$ defining the modified ODEs directly --- as we have done for the stability analysis results --- rather than working with losses.

By using $f = -\nabla_{\vphi} L_1$ and $g = -\nabla_{\vtheta} L_2$, we can rewrite the drift of the simultaneous Euler updates (corresponding to simultaneous gradient descent in this setting) as:

\begin{align}
f_1 &= - \frac{1}{2} f \nabla_{\vphi}f - \frac{1}{2} g \nabla_{\vtheta} f \\
     &= - \frac{1}{2} \nabla_{\vphi} L_1 \nabla_{\vphi} (\nabla_{\vphi} L_1) - \frac{1}{2} (\nabla_{\vtheta} L_2) \nabla_{\vtheta} (\nabla_{\vphi} L_1) \\
 &= - \frac{1}{4}  \nabla_{\vphi} \norm{ \nabla_{\vphi} L_1}^2  - \frac{1}{2} (\nabla_{\vtheta} L_2) \nabla_{\vtheta} (\nabla_{\vphi} L_1)
\end{align}
and similarly
\begin{align}
g_1 &= - \frac{1}{4}  \nabla_{\vtheta} \norm{ \nabla_{\vtheta} L_2}^2  - \frac{1}{2} (\nabla_{\vphi} L_1) \nabla_{\vphi} (\nabla_{\vtheta} L_2)
\end{align}

As we can see here, it is not possible in general to write $f_1$ and $g_1$ as gradient functions, as was possible for a zero-sum game or a common-payoff game.

\section{Discretization drift in Runge-Kutta 4 (RK4)}

\textbf{Runge-Kutta 4 for one player}\\
RK4 is a Runge-Kutta a method of order 4. This means that the discrete steps of RK4 coincide with the exact flow of the original ODE up to $\mathcal O(h^5)$ (i.e., the local error after one step is of order $\mathcal O(h^5)$). The modified equation for a method of order $n$ starts with corrections at order $h^{n+1}$ (i.e., all the lower corrections vanish; see \cite{backward_lifespan} and \cite{GNI} for further details). This means that RK4 has no DD up to order $\mathcal O(h^5)$, and why for small learning rates RK4 can be used as a proxy for the exact flow.

\textbf{Runge-Kutta 4 for two players}\\
When we use equal learning rates and simultaneous updates, the two-player game is always equivalent to the one player case, so Runge-Kutta 4 will have a local error of $\mathcal O(h^5)$. However, in the case of two-players games, we have the additional freedom of having different learning rates for each of the players. We now show that when the learning rates of the two players are different, \textit{RK4 also also has a drift effect of order $2$ and the DD term comes exclusively from the interaction terms}. To do so, we apply the same steps as we have done for the Euler updates.

\textbf{Step 1: Expand the updates per player via a Taylor expansion.} \\
The simultaneous Runge-Kutta 4 updates for two players are:
\begin{align*}
 k_{1, \vphi} &= f(\vphi_{t-1}, \vtheta_{t-1}) \\
 k_{1, \vtheta} &= g(\vphi_{t-1}, \vtheta_{t-1}) \\
 k_{2,\vphi} &= f(\vphi_{t-1} + \frac{\alpha h}{2} k_{1, \vphi}, \vtheta_{t-1} + \frac{\lambda h}{2} k_{1, \vtheta}) \\
 k_{2,\vtheta} &= g(\vphi_{t-1} + \frac{\alpha h}{2} k_{1, \vphi}, \vtheta_{t-1} + \frac{\lambda h}{2} k_{1, \vtheta}) \\
 k_{3,\vphi} &= f(\vphi_{t-1} + \frac{\alpha h}{2} k_{2, \vphi}, \vtheta_{t-1} + \frac{\lambda h}{2} k_{2, \vtheta}) \\
 k_{3,\vtheta} &=  g(\vphi_{t-1} + \frac{\alpha h}{2} k_{2, \vphi}, \vtheta_{t-1} + \frac{\lambda h}{2} k_{2, \vtheta}) \\
 k_{4,\vphi} &= f(\vphi_{t-1} + \frac{\alpha h}{2} k_{3, \vphi}, \vtheta_{t-1} + \frac{\lambda h}{2} k_{3, \vtheta}) \\
 k_{4,\vtheta} &=  g(\vphi_{t-1} + \frac{\alpha h}{2} k_{3, \vphi}, \vtheta_{t-1} + \frac{\lambda h}{2} k_{3, \vtheta}) \\
 \end{align*}
 \begin{align*}
 k_{\vphi} &= \frac{1}{6} \left(k_{1,\vphi} + 2 k_{2,\vphi} + 2 k_{3,\vphi} + k_{4,\vphi}\right) \\
 k_{\vtheta} &= \frac{1}{6} \left(k_{1,\vtheta} + 2 k_{2,\vtheta} + 2 k_{3,\vtheta} + k_{4,\vtheta}\right) \\
 \end{align*}
 \begin{align*}
 \vphi_{t}  &= \vphi_{t-1} + \alpha h  k_{\vphi} \\
 \vtheta_{t}  &= \vtheta_{t-1} + \lambda h k_{\vtheta}
\end{align*}

We expand each intermediate step:
\begin{align*}
 k_{1, \vphi} &= f_{(t-1)} \\
 k_{1, \vtheta} &= g_{(t-1)} \\
 k_{2,\vphi} &= f(\vphi_{t-1} + \frac{\alpha h}{2} k_{1, \vphi}, \vtheta_{t-1} + \frac{\lambda h}{2} k_{1, \vtheta}) = f(\vphi_{t-1},  \vtheta_{t-1} + \frac{\lambda h}{2} k_{1, \vtheta}) + \frac{\alpha h}{2} k_{1,\vphi} \nabla_{\vphi}f(\vphi_{t-1}, \vtheta_{t-1} + \frac{\lambda h}{2} k_{1, \vtheta}) + \mathcal{O}(h^2) \\
    &= f_{(t-1)} + \frac{\lambda h}{2} k_{1, \vtheta} \nabla_{\vtheta} f_{(t-1)} + \frac{\alpha h}{2} k_{1,\vphi} \nabla_{\vphi}f_{(t-1)}  + \mathcal{O}(h^2) \\
    &= f_{(t-1)} + \frac{\lambda h}{2} g_{(t-1)} \nabla_{\vtheta} f_{(t-1)} + \frac{\alpha h}{2} f_{(t-1)}\nabla_{\vphi}f_{(t-1)}  + \mathcal{O}(h^2) \\
 k_{2,\vtheta} &= g(\vphi_{t-1} + \frac{\alpha h}{2} k_{1, \vphi}, \vtheta + \frac{\lambda h}{2} k_{1, \vtheta}) \\
    &=g_{(t-1)} + \frac{\lambda h}{2} k_{1, \vtheta} \nabla_{\vtheta}g_{(t-1)} + \frac{\alpha h}{2} k_{1,\vphi} \nabla_{\vphi}g_{(t-1)}  + \mathcal{O}(h^2) \\
    &=g_{(t-1)} + \frac{\lambda h}{2} g_{(t-1)} \nabla_{\vtheta}g_{(t-1)} + \frac{\alpha h}{2} f_{(t-1)}\nabla_{\vphi}g_{(t-1)}  + \mathcal{O}(h^2) \\
 \end{align*}
 \begin{align*}
 k_{3,\vphi} &= f_{(t-1)} + \frac{\lambda h}{2} k_{2, \vtheta} \nabla_{\vtheta} f_{(t-1)} + \frac{\alpha h}{2} k_{2,\vphi} \nabla_{\vphi}f_{(t-1)}  + \mathcal{O}(h^2) \\
    &= f_{(t-1)} + \frac{\lambda h}{2} g_{(t-1)} \nabla_{\vtheta} f_{(t-1)} + \frac{\alpha h}{2} f_{(t-1)}\nabla_{\vphi}f_{(t-1)}  + \mathcal{O}(h^2) \\
k_{3,\vtheta} &=g_{(t-1)} + \frac{\lambda h}{2} g_{(t-1)} \nabla_{\vtheta}g_{(t-1)} + \frac{\alpha h}{2} f(\vphi, \vtheta)\nabla_{\vphi}g_{(t-1)}  + \mathcal{O}(h^2) \\
k_{4,\vphi} &= f_{(t-1)} + \frac{\lambda h}{2} g_{(t-1)} \nabla_{\vtheta} f_{(t-1)} + \frac{\alpha h}{2} f_{(t-1)}\nabla_{\vphi}f_{(t-1)}  + \mathcal{O}(h^2) \\
k_{4,\vtheta} &=g_{(t-1)} + \frac{\lambda h}{2} g_{(t-1)} \nabla_{\vtheta}g_{(t-1)} + \frac{\alpha h}{2} f_{(t-1)}\nabla_{\vphi} g_{(t-1)}  + \mathcal{O}(h^2) \\
\end{align*}

with the update direction:
\begin{align*}
k_{\vphi} &= \frac{1}{6} \left(k_{1,\vphi} + 2 k_{2,\vphi} + 2 k_{3,\vphi} + k_{4,\vphi}\right) = f_{(t-1)} + \frac{ \lambda h}{2} g_{(t-1)} \nabla_{\vtheta} f_{(t-1)} + \frac{\alpha h}{2}  f_{(t-1)} \nabla_{\vphi} f_{(t-1)} + \mathcal{O}(h^2) \\
k_{\vtheta} &= \frac{1}{6} \left(k_{1,\vtheta} + 2 k_{2,\vtheta} + 2 k_{3,\vtheta} + k_{4,\vtheta}\right) = g_{(t-1)} + \frac{\lambda h}{2}  g_{(t-1)} \nabla_{\vtheta}g_{(t-1)} + \frac{\alpha h}{2} f_{(t-1)} \nabla_{\vphi}g_{(t-1)} + \mathcal{O}(h^2) \\
\end{align*}

and thus the discrete update of the Runge-Kutta 4 for two players are:
\begin{align}
\vphi_{t}  &= \vphi_{t-1} + \alpha h f_{(t-1)} + \frac{1}{2} \alpha \lambda h^2 g_{(t-1)} \nabla_{\vtheta}f_{(t-1)} + \frac{1}{2} \alpha^2 h^2 f_{(t-1)} \nabla_{\vphi}f_{(t-1)} + \mathcal{O}(h^3) \label{eq:rk_1}\\
\vtheta_{t}  &= \vtheta_{t-1} + \lambda h g_{(t-1)} + \frac{1}{2} \lambda^2 h^2 g_{(t-1)} \nabla_{\vtheta} g_{(t-1)} + \frac{1}{2} \alpha\lambda h^2 f_{(t-1)} \nabla_{\vphi} g_{(t-1)} + \mathcal{O}(h^3) \label{eq:rk_2}
\end{align}

\textbf{Step 2: Expand the continuous time changes for the modified ODE }\\
(Identical to the simultaneous Euler updates.)

\textbf{Step 3: Matching the discrete and modified continuous updates.} \\
As in the always in Step 3,  we substitute  $\vphi_{t-1}$ and $\vtheta_{t-1}$ in Lemma~\ref{thm:cont}:
\begin{align}
\vphi(\tph + \tau_\phi) &= \vphi_{t-1} + \tau_\phi f_{(t-1)} + \tau_\phi^2 f_1(\vphi_{t-1}, \vtheta_{t-1}) + \frac{1}{2} \tau_\phi^2  f_{(t-1)} \nabla_{\vphi}  f_{(t-1)} + \frac{1}{2}\tau_\phi^2 g_{(t-1)} \nabla_{\vtheta}  f_{(t-1)} + \mathcal{O}(\tau_\phi^3) \\
\vtheta(\tph + \tau_\theta) &= \vtheta_{t-1} + \tau_\theta g_{(t-1)} + \tau_\theta^2 g_1(\vphi_{t-1}, \vtheta_{t-1}) + \frac{1}{2} \tau_\theta^2 f_{(t-1)} \nabla_{\vphi} g_{(t-1)} + \frac{1}{2}\tau_\theta^2 g_{(t-1)} \nabla_{\vtheta} g_{(t-1)} + \mathcal{O}(\tau_\theta^3)
\end{align}

For the first order terms we obtain $\tau_{\phi} = \alpha h$ and $\tau_{\theta} = \lambda h$. We match the $\mathcal{O}(h^2)$ terms in the equations above with the discrete Runge-Kutta 4 updates shown in Equation~\eqref{eq:rk_1} and~\eqref{eq:rk_2}
and notice that:
\begin{align}
f_1(\vphi_{t-1}, \vtheta_{t-1}) = \frac 1 2 (\frac \lambda \alpha- 1) g_{(t-1)} \nabla_{\vtheta} f_{t-1} \\
g_1(\vphi_{t-1}, \vtheta_{t-1}) = \frac 1 2(\frac \alpha \lambda- 1) f_{(t-1)} \nabla_{\vphi} g_{t-1}
\end{align}

Thus, if $\alpha \ne \lambda$ RK4 has second order drift. This is why, in all our experiments comparing with RK4, we use the same learning rates for the two players $\alpha = \lambda$, to ensure that we use a method which has no DD up to order $\mathcal{O}(h^5)$.

\section{Stability analysis}

In this section, we give all the details of the stability analysis results, to showcase how the modified ODEs we have derived can be used as a tool for stability analysis. We  provide the full computation for the Jacobian of the modified vector fields for simultaneous and alternating Euler updates, as well as the calculation of their trace, and show how this can be used to determine the stability of the modified vector fields. While analyzing the modified vector fields is not equivalent to analyzing the discrete dynamics due to the higher order errors of the drift which we ignore, it provides a better approximation than what is often used in practice, namely the original ODEs, which completely ignore the drift.

\subsection{Simultaneous Euler updates}

Consider a two-player game with dynamics given by $\vphi_t = \vphi_{t-1}  +  \alpha h f_{(t-1)}$ and  $\vtheta_t = \vtheta_{t-1}  +  \lambda h g_{(t-1)} $ (Equations \eqref{eq:supp_simup1} and \eqref{eq:supp_simup2}). The modified dynamics for this game are given by $\dot{\vphi} = \widetilde{f} $, $\dot{\vtheta} = \widetilde{g}$, where $\widetilde{f} = f - \frac{ \alpha h}{2} (f \nabla_{\vphi} f + g \nabla_{\vtheta} f)$ and $\widetilde{g} = g - \frac{ \lambda h}{2} (f \nabla_{\vphi} g + g \nabla_{\vtheta} g )$ (Theorem \ref{thm:supp_2_players_sim}).

The stability of this system can be characterized by the modified Jacobian matrix evaluated at the equilibria of the two-player game. The equilibria that we are interested in for our stability analysis are the steady-state solutions of Equations \eqref{eq:supp_simup1} and \eqref{eq:supp_simup2}, given by $f = \myvec{0}, g = \myvec{0}$. These are also equilibrium solutions for the steady-state modified equations\footnote{There are additional steady-state solutions for the modified equations. However, we can ignore these since they are spurious solutions that do not correspond to steady states of the two-player game, arising instead as an artifact of the $\mathcal{O}(h^3)$ error in backward error analysis.} given by $\tilde{f} = \mathbf{0}, \tilde{g} = \mathbf{0}$.

The modified Jacobian can be written, using block matrix notation as:
\begin{gather}
 \widetilde{J}
 =
  \begin{bmatrix}
    \nabla_{\vphi}\widetilde{f} & \nabla_{\vtheta}\widetilde{f}\\
    \nabla_{\vphi}\widetilde{g} &
   \nabla_{\vtheta}\widetilde{g}
   \end{bmatrix}
\end{gather}
Next, we calculate each term in this block matrix. (In the following analysis, each term is evaluated at an equilibrium solution given by $f = \mathbf{0}, g = \mathbf{0}$). We find:
\begin{align*}
\nabla_{\vphi}\tilde{f} &= \nabla_{\vphi}f - \frac{ \alpha h}{2} \left(( \nabla_{\vphi} f)^2 + f \nabla_{\vphi, \vphi} f + \nabla_{\vphi}g \nabla_{\vtheta}f  + g \nabla_{\vphi,\vtheta} f  \right)\\
&=\nabla_{\vphi}f - \frac{ \alpha h}{2} \left(( \nabla_{\vphi} f)^2 + \nabla_{\vphi}g \nabla_{\vtheta}f  \right) \\
\nabla_{\vtheta}\tilde{f} &= \nabla_{\vtheta}f - \frac{ \alpha h}{2} \left(\nabla_{\vtheta} f \nabla_{\vphi} f + f \nabla_{\vtheta, \vphi} f + \nabla_{\vtheta}g \nabla_{\vtheta}f  + g \nabla_{\vtheta,\vtheta} f  \right)\\
&= \nabla_{\vtheta}f - \frac{ \alpha h}{2} \left(\nabla_{\vtheta} f \nabla_{\vphi} f  + \nabla_{\vtheta}g \nabla_{\vtheta}f \right) \\
\nabla_{\vphi}\tilde{g} &= \nabla_{\vphi}g - \frac{ \lambda h}{2} (\nabla_{\vphi}g \nabla_{\vtheta} g + g \nabla_{\vphi,\vtheta} g + \nabla_{\vphi}f \nabla_{\vphi} g +f \nabla_{\vphi,\vphi} g ) \\
&= \nabla_{\vphi}g - \frac{ \lambda h}{2} ( \nabla_{\vphi}g \nabla_{\vtheta} g + \nabla_{\vphi}f \nabla_{\vphi} g )\\
\nabla_{\vtheta} \tilde{g} &= \nabla_{\vtheta} g - \frac{ \lambda h}{2} \left( (\nabla_{\vtheta} g)^2+ g \nabla_{\vtheta,\vtheta} g +  \nabla_{\vtheta} f \nabla_{\vphi} g +  f \nabla_{\vtheta,\vphi} g\right)\\
&= \nabla_{\vtheta} g - \frac{ \lambda h}{2} \left( (\nabla_{\vtheta} g)^2+  \nabla_{\vtheta} f \nabla_{\vphi} g \right)\\
\end{align*}
Given these calculations, we can now write:
\begin{gather}
 \widetilde{J}
 =
  \begin{bmatrix}
    \nabla_{\vphi}\widetilde{f} & \nabla_{\vtheta}\widetilde{f}\\
    \nabla_{\vphi}\widetilde{g} &
   \nabla_{\vtheta}\widetilde{g}
   \end{bmatrix}
   = J - \frac{h}{2}K_{\text{sim}}
\end{gather}
where $J$ is the Jacobian of the unmodified ODE:
\begin{gather}
 J
 =
  \begin{bmatrix}
    \nabla_{\vphi}f & \nabla_{\vtheta}f\\
    \nabla_{\vphi}g &
   \nabla_{\vtheta}g
   \end{bmatrix}
\end{gather}
and
\begin{gather}
K_{\text{sim}}
 =
  \begin{bmatrix}
    \alpha ( \nabla_{\vphi} f)^2 +  \alpha \nabla_{\vphi}g\nabla_{\vtheta}f   & \alpha \nabla_{\vtheta}f\nabla_{\vphi}f +  \alpha \nabla_{\vtheta}g\nabla_{\vtheta}f\\
   \lambda \nabla_{\vphi}g\nabla_{\vtheta}g +  \lambda \nabla_{\vphi}f\nabla_{\vphi}g&
   \lambda( \nabla_{\vtheta} g)^2 +  \lambda \nabla_{\vtheta}f\nabla_{\vphi}g
   \end{bmatrix}
\end{gather}

The modified system of equations are asymptotically stable if all the real parts of all the eigenvalues of the modified Jacobian are negative. If some of the eigenvalues are zero, the equilibrium may or may not be stable, depending on the non-linearities of the system. If the real parts of any eigenvalues are positive, then the dynamical system is unstable.

A necessary condition for stability is that the trace of the modified Jacobian is less than or equal to zero (i.e. $\Tr(\widetilde{J})\leq0$), since the trace is the sum of eigenvalues. Using the property of trace additivity and the trace cyclic property we see that:\begin{equation}\label{eqn:modified_trace}
\Tr(\widetilde{J}) = \Tr(J) - \frac{ h}{2}\left(\alpha \Tr((\nabla_{\vphi} f)^2)+ \lambda  \Tr((\nabla_{\vtheta} g)^2)    \right) - \frac{ h}{2}(\alpha + \lambda )\Tr( \nabla_{\vphi}g\nabla_{\vtheta}f )
\end{equation}

\subsubsection{Instability caused by discretization drift}

We now use the above analysis to show that the equilibria of a two-player game following the modified ODE obtained simultaneous Euler updates as defined by Equations \eqref{eq:supp_simup1} and \eqref{eq:supp_simup2} can become asymptotically unstable for some games.

There are choices of $f$ and $g$ that have stable equilibrium without DD, but are unstable under DD. For example, consider the zero-sum two-player game with $f=\nabla_{\vphi} E $ and $g=- \nabla_{\vtheta} E $. Now, we see that \begin{equation}\label{eqn:modified_trace_zero_sum}
\Tr(\widetilde{J}) = \Tr(J) - \frac{ h}{2}\left(\alpha {\| \nabla_{\vphi, \vphi} E\|}^2_F  + \lambda {\| \nabla_{\vtheta, \vtheta} E\|}^2_F    \right) + \frac{ h}{2}(\alpha + \lambda ) {\| \nabla_{\vphi, \vtheta} E\|}^2_F
\end{equation}
where $\|.\|_F$ denotes the Frobenius norm. The Dirac-GAN is an example of a zero-sum two-player game that is stable without DD, but becomes unstable under DD with  $\Tr(\widetilde{J})  = h(\alpha + \lambda ) {\| \nabla_{\vphi, \vtheta} E\|}^2_F/2 > 0$ (see Section \ref{sec:dirac-gan-stability}).

\subsection{Alternating Euler updates}

Consider a two-player game with dynamics given by Equations \eqref{eq:supp_altup1} and \eqref{eq:supp_altup2}. The modified dynamics for this game are given by $\dot{\vphi} = f - \frac{\alpha h}{2} \left(\frac{1}{m}f \nabla_{\vphi} f + g \nabla_{\vtheta} f \right)$,
$ \dot{\vtheta} = g - \frac{\lambda h}{2} \left((1- \frac {2\alpha} {\lambda})f \nabla_{\vphi} g + \frac{1}{k} g \nabla_{\vtheta} g \right)$ (Theorem \ref{thm:supp_2_players_alt}).

The stability of this system can be characterized by the modified Jacobian matrix evaluated at the equilibria of the two-player game. The equilibria that we are interested in for our stability analysis are the steady-state solutions of Equations \eqref{eq:supp_simup1} and \eqref{eq:supp_simup2}, given by $f = \myvec{0}, g = \myvec{0}$. These are also equilibrium solutions for the steady-state modified equations\footnote{There are additional steady-state solutions for the modified equations. However, we can ignore these since they are spurious solutions that do not correspond to steady states of the two-player game, arising instead as an artifact of the $\mathcal{O}(h^3)$ error in backward error analysis.} given by $\tilde{f} = \mathbf{0}, \tilde{g} = \mathbf{0}$.

The modified Jacobian can be written, using block matrix notation as:
\begin{gather}
 \widetilde{J}
 =
  \begin{bmatrix}
    \nabla_{\vphi}\widetilde{f} & \nabla_{\vtheta}\widetilde{f}\\
    \nabla_{\vphi}\widetilde{g} &
   \nabla_{\vtheta}\widetilde{g}
   \end{bmatrix}
\end{gather}
Next, we calculate each term in this block matrix. (In the following analysis, each term is evaluated at an equilibrium solution given by $f = \mathbf{0}, g = \mathbf{0}$). We find:
\begin{align*}
\nabla_{\vphi}\tilde{f} &= \nabla_{\vphi}f - \frac{ \alpha h}{2} \left(\frac 1 m ( \nabla_{\vphi} f)^2 +  \frac 1 m f \nabla_{\vphi, \vphi} f + \nabla_{\vphi}g \nabla_{\vtheta}f  + g \nabla_{\vphi,\vtheta} f  \right)\\
&=\nabla_{\vphi}f - \frac{ \alpha h}{2} \left(\frac 1 m ( \nabla_{\vphi} f)^2 + \nabla_{\vphi}g \nabla_{\vtheta}f  \right) \\
\nabla_{\vtheta}\tilde{f} &= \nabla_{\vtheta}f - \frac{ \alpha h}{2} \left(\frac 1 m \nabla_{\vtheta} f \nabla_{\vphi} f + \frac 1 m f \nabla_{\vtheta, \vphi} f + \nabla_{\vtheta}g \nabla_{\vtheta}f  + g \nabla_{\vtheta,\vtheta} f  \right)\\
&= \nabla_{\vtheta}f - \frac{ \alpha h}{2} \left(\frac 1 m \nabla_{\vtheta} f \nabla_{\vphi} f  + \nabla_{\vtheta}g \nabla_{\vtheta}f \right) \\
\nabla_{\vphi}\tilde{g} &= \nabla_{\vphi}g - \frac{ \lambda h}{2} (\frac 1 k \nabla_{\vphi}g \nabla_{\vtheta} g + \frac 1 k g \nabla_{\vphi,\vtheta} g + (1 - \frac{2 \alpha}{\lambda})\nabla_{\vphi}f \nabla_{\vphi} g + (1 - \frac{2 \alpha}{\lambda})f \nabla_{\vphi,\vphi} g ) \\
&= \nabla_{\vphi}g - \frac{ \lambda h}{2} (\frac 1 k \nabla_{\vphi}g \nabla_{\vtheta} g + (1 - \frac{2 \alpha}{\lambda})\nabla_{\vphi}f \nabla_{\vphi} g )\\
\nabla_{\vtheta} \tilde{g} &= \nabla_{\vtheta} g - \frac{ \lambda h}{2} \left(\frac 1 k (\nabla_{\vtheta} g)^2 + \frac 1 k g \nabla_{\vtheta,\vtheta} g + (1 - \frac{2 \alpha}{\lambda}) \nabla_{\vtheta} f \nabla_{\vphi} g + (1 - \frac{2 \alpha}{\lambda})  f \nabla_{\vtheta,\vphi} g\right)\\
&= \nabla_{\vtheta} g - \frac{ \lambda h}{2} \left(\frac 1 k (\nabla_{\vtheta} g)^2+  (1 - \frac{2 \alpha}{\lambda}) \nabla_{\vtheta} f \nabla_{\vphi} g \right)\\
\end{align*}
Given these calculations, we can now write:
\begin{gather}
 \widetilde{J}
 =
  \begin{bmatrix}
    \nabla_{\vphi}\widetilde{f} & \nabla_{\vtheta}\widetilde{f}\\
    \nabla_{\vphi}\widetilde{g} &
   \nabla_{\vtheta}\widetilde{g}
   \end{bmatrix}
   = J - \frac{h}{2}K_{\text{alt}}
\end{gather}
where $J$ is the Jacobian of the unmodified ODE:
\begin{gather}
 J
 =
  \begin{bmatrix}
    \nabla_{\vphi}f & \nabla_{\vtheta}f\\
    \nabla_{\vphi}g &
   \nabla_{\vtheta}g
   \end{bmatrix}
   \label{eq:jacobian}
\end{gather}
and
\begin{gather}
K_{\text{alt}}
 =
  \begin{bmatrix}
    \frac \alpha m ( \nabla_{\vphi} f)^2 + \alpha \nabla_{\vphi}g\nabla_{\vtheta}f   & \frac \alpha m \nabla_{\vtheta}f\nabla_{\vphi}f +  \alpha \nabla_{\vtheta}g\nabla_{\vtheta}f\\
   \frac \lambda k \nabla_{\vphi}g\nabla_{\vtheta}g +  \lambda (1 - \frac{2 \alpha}{\lambda}) \nabla_{\vphi}f\nabla_{\vphi}g&
   \frac \lambda k( \nabla_{\vtheta} g)^2 +  \lambda (1 - \frac{2 \alpha}{\lambda}) \nabla_{\vtheta}f\nabla_{\vphi}g
   \end{bmatrix}
\end{gather}

The modified system of equations are asymptotically stable if all the real parts of all the eigenvalues of the modified Jacobian are negative. If some of the eigenvalues are zero, the equilibrium may or may not be stable, depending on the non-linearities of the system. If the real parts of any eigenvalues are positive, then the dynamical system is unstable.

A necessary condition for stability is that the trace of the modified Jacobian is less than or equal to zero (i.e. $\Tr(\widetilde{J})\leq0$), since the trace is the sum of eigenvalues. Using the property of trace additivity and the trace cyclic property we see that:\begin{equation}\label{eqn:modified_trace_alternating}
\Tr(\widetilde{J}) = \Tr(J) - \frac{ h}{2}\left(\frac \alpha m \Tr((\nabla_{\vphi} f)^2)+ \frac \lambda k  \Tr((\nabla_{\vtheta} g)^2)    \right) - \frac{ h}{2}(\lambda - \alpha)\Tr( \nabla_{\vphi}g\nabla_{\vtheta}f )
\end{equation}

We note that unlike for simultaneous updates, even if $\Tr( \nabla_{\vphi}g\nabla_{\vtheta}f )$ is negative, if $\lambda < \alpha$, the trace of the modified system will stay negative, so the necessary condition for the system to remain stable is still satisfied. However, since this is not a sufficient condition, the modified system could still be unstable.

\subsection{A new tool for stability analysis: different ODEs for simultaneous and alternating Euler updates}

We will now highlight how having access to different systems of ODEs to closely describe the dynamics of simultaneous and alternating Euler updates can be a useful tool for stability analysis. Stability analysis has been used to understand the local convergence properties of games such as GANs~\citep{nagarajan2017gradient}. Thus far, this type of analysis has relied on the original ODEs describing the game and has ignored discretization drift. Moreover, since no ODEs were available to capture the difference in dynamics between simultaneous and alternating Euler updates, alternating updates have not been studied using stability analysis, despite being predominantly used in practice by practitioners. We will use an illustrative example to show how the modified ODEs we have derived using backward error analysis can be used to analyze the local behaviour of Euler updates and uncover the different behavior of simultaneous and alternating updates. As before, we use the ODEs
\begin{align*}
\dot{\phi} = f(\phi, \theta) =  - \epsilon_1 \phi  + \theta; \hspace{2em} \dot{\theta} = g(\phi, \theta) = \epsilon_2 \theta  - \phi
\end{align*}

We set $\epsilon_1 = 0.09, \epsilon_2 = 0.09$ and a learning rate $0.2$ and show the behavior of the original flow as well as simultaneous and alternating Euler updates and their corresponding modified flows in Figure~\ref{fig:supp_igr_divergence}. By replacing the values of $f$ and $g$ into the results for the Jacobian of the modified ODEs obtain above, we obtain the corresponding Jacobians. For simultaneous updates:
\begin{gather}
 \tilde{J}_{\text{sim}}
 =
  \begin{bmatrix}
    -\epsilon_1 - h / 2 \epsilon_1^2 + h /2 & 1 + h / 2  \epsilon_1 - h / 2 \epsilon_2 \\
  -1 - h /2  \epsilon_1 + h /2  \epsilon_2 & \epsilon_2 + h / 2 - h/2  \epsilon_2^2
   \end{bmatrix}
\end{gather}

For alternating updates:
\begin{gather}
 \tilde{J}_{\text{alt}}
 =
  \begin{bmatrix}
  - \epsilon_1 - h / 2   \epsilon_1^2 + h /2 & 1 + h / 2   \epsilon_1 - h / 2  \epsilon_2 \\
  -1 + h /2   \epsilon_1 + h /2  \epsilon_2 & \epsilon_2 - h / 2 - h/2  \epsilon_2 ^2
   \end{bmatrix}
\end{gather}

When we replace the values of $\epsilon_1 = \epsilon_2 = 0.09$ and $h = 0.2$, we obtain that $\Tr({\tilde{J}_{\text{sim}})} = 0.19 > 0$, thus the system of the modified ODEs corresponding to simultaneous Euler updates diverge. For alternating updates, we obtain $\Tr({\tilde{J}_{\text{alt}})} = -0.0016 < 0$ and  $|\tilde{J}_{\text{alt}}| = 0.981 > 0$, leading to a stable system. Our analysis is thus consistent with what we observe empirically in Figure~\ref{fig:supp_igr_divergence}. We observe the same results for other choices of $\epsilon_1, \epsilon_2$ and $h$.

\textit{Benefits and caveats of using the modified ODEs for stability analysis}: Using the modified ODEs we have derived using backward error analysis allows us to study systems which are closer to the discrete Euler updates, as they do not ignore discretization drift. The modified ODEs also provide a tool to discriminate between alternating and simultaneous Euler updates in two-player games when performing stability analysis to understand the discrete behavior of the system. Accounting for the drift in the analysis of two-player games is crucial, since as we have seen, the drift can change a stable equilibrium into an unstable one -- which is not the case for supervised learning~\citep{igr}. These strong benefits of using the modified ODEs we propose for stability analysis also come with caveats: the modified ODEs nonetheless ignore higher order drift terms, and our approximations might not hold for very large learning rates.

\begin{figure}[t]
  \centering
  \includegraphics[width=0.8\columnwidth]{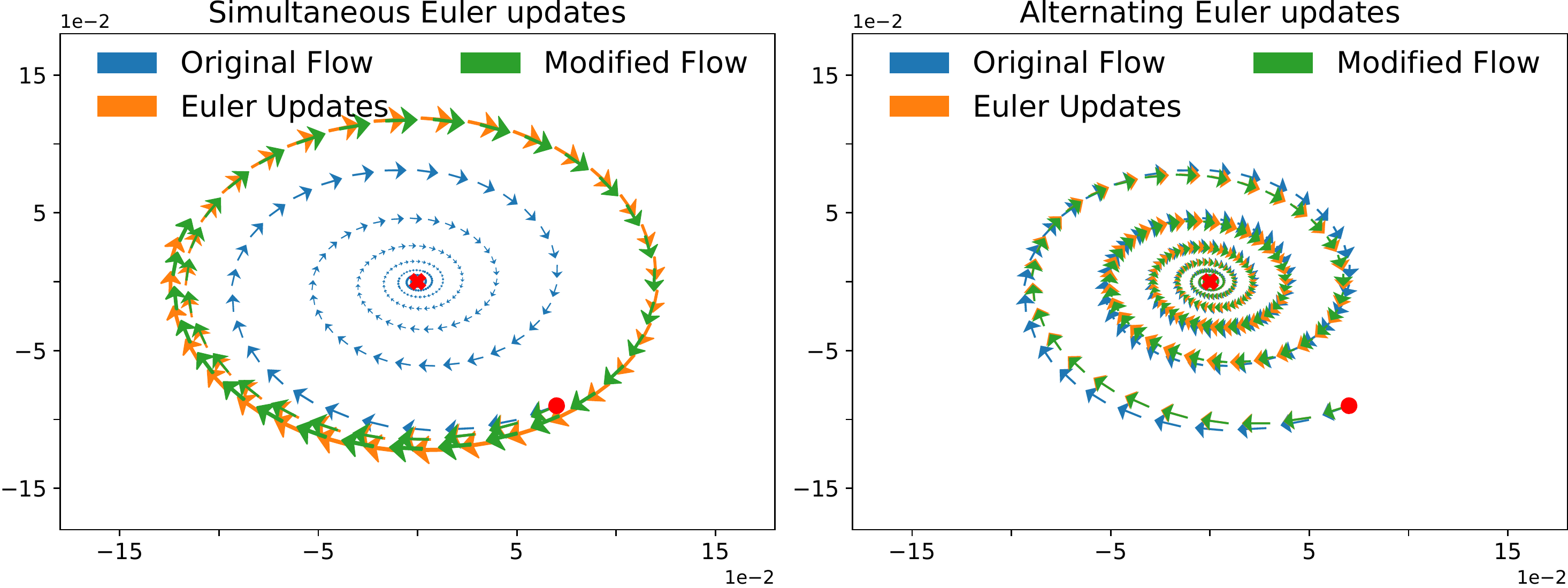}
  \caption{Discretization drift can change the stability of a game. $\epsilon_1 =\epsilon_2 = 0.09$ and a learning rate $0.2$.}
  \label{fig:supp_igr_divergence}
\end{figure}

\section{SGA in two-player games}

For clarity, this section reproduces Symplectic Gradient Adjustment (SGA) from ~\citet{balduzzi2018mechanics} for two-player games, using notations consistent with this paper. If we have two players $\vphi$ and $\vtheta$ minimizing loss functions $L_1$ and $L_2$, SGA defines the vector:
\begin{gather}
 \vxi
 =
  \begin{bmatrix}
    \nabla_{\vphi} L_1 \\
    \nabla_{\vtheta} L_2
   \end{bmatrix}
   \label{eq:sga_neg_vector_field}
\end{gather}

As defined in SGA, $\vxi$ is the \textit{negative of the vector field} that the system dynamics follow. This entails that
according to our notation: \begin{gather}
 \vxi
 =
  \begin{bmatrix}
    - f \\
    - g
   \end{bmatrix}
   \label{eq:correspondance}
\end{gather}

 defining the Jacobian:
\begin{gather}
 J_{ \vxi}
 =
  \begin{bmatrix}
    - \nabla_{\vphi}f & - \nabla_{\vtheta}f\\
   - \nabla_{\vphi}g &
   - \nabla_{\vtheta}g
   \end{bmatrix}
   \label{eq:j}
\end{gather}

Since in SGA $f = \nabla_{\vphi} L_1$ and $g = - \nabla_{\vtheta} L_2$ the Jacobian has as an \emph{anti-symmetric} component
\begin{gather}
 A
 = \frac{1}{2}(J_{\vxi} - J_{\vxi}^T) =
  \frac{1}{2} \begin{bmatrix}
    0 & - \nabla_{\vtheta}f + \nabla_{\vphi}g\\
    - \nabla_{\vphi}g + \nabla_{\vtheta}f &
   0
   \end{bmatrix}
   \label{eq:j_anti_symmetric}
\end{gather}

Ignoring the sign change from alignment \citep{balduzzi2018mechanics} for simplicity, the vector field $\vxi$ is modified according to SGA as
\begin{gather}
 \hat{\vxi}
 = \vxi + A^T \, \vxi = \begin{bmatrix} - f \\ - g \end{bmatrix} +
  \frac{1}{2} \begin{bmatrix}
    \left( \nabla_{\vphi}g - \nabla_{\vtheta}f \right)^T \, g\\
    \left(\nabla_{\vtheta}f - \nabla_{\vphi}g \right)^T \, f
   \end{bmatrix}
   \label{eq:j_sga}
\end{gather}

For zero-sum games, we have that $f  = - \nabla_{\vphi} E $ and $g = \nabla_{\vtheta} E$ and thus:
\begin{align*}
& \nabla_{\vphi}g^T g = -\nabla_{\vtheta}f^T g = \frac{1}{2}\nabla_{\vphi} \norm{\nabla_{\vtheta}E}^2 \\
& \nabla_{\vtheta}f^T f = - \nabla_{\vphi}g^T f = \frac{1}{2}\nabla_{\vtheta} \norm{\nabla_{\vphi}E}^2
\end{align*}

Thus the modified gradient field can be simplified to
\begin{gather}
 \hat{\vxi}
 = \begin{bmatrix} -f \\ - g \end{bmatrix} +
  \begin{bmatrix}
    \nabla_{\vphi}g^T g\\
    \nabla_{\vtheta}f^T f
   \end{bmatrix}
 = \begin{bmatrix} \nabla_{\vphi} E \\ - \nabla_{\vtheta} E \end{bmatrix} +
  \frac{1}{2}\begin{bmatrix}
    \nabla_{\vphi} \norm{\nabla_{\vtheta}E}^2\\
    \nabla_{\vtheta} \norm{\nabla_{\vphi}E}^2
   \end{bmatrix}
   = \begin{bmatrix} \nabla_{\vphi} (E + \frac{1}{2} \norm{\nabla_{\vtheta}E}^2)\\
     \nabla_{\vtheta} (- E + \frac{1}{2} \norm{\nabla_{\vphi}E}^2)
   \end{bmatrix}
   \label{eq:j_zero_su}
\end{gather}
Therefore, since $\hat{\vxi}$ defines the negative of the vector field followed by the system, the modified losses for the two players can be written respectively as:
\begin{align}
\widetilde{L}_1 &= E + \frac{1}{2}\norm{\nabla_{\vtheta}E}^2 \\
\widetilde{L}_2 &= -E + \frac{1}{2}\norm{\nabla_{\vphi}E}^2
\end{align}

The functional form of the modified losses given by SGA is the same used to cancel the interaction terms of DD in the case of simultaneous gradient descent updates in zero sum games. We do however highlight a few differences in our approach compared to SGA: our approach extends to alternating updates and provides the optimal regularization coefficients; canceling the interaction terms of the drift is different compared to SGA for general games (see Equation~\ref{eq:j_sga} for SGA and Theorem~\ref{thm:supp_sim} for the interaction terms of DD).

\section{DiracGAN - an illustrative example}
\label{app:dirac_gan}
In their work assessing the convergence of GAN training~\citet{mescheder2018training}
introduce the example of the DiracGAN, where the GAN is trying to learn a delta distribution with mass at zero.
More specifically, the generator $G_\theta(z) = \theta$ with parameter $\theta$ parametrizes constant functions whose images $\{\theta\}$ correspond to the support of the delta distribution $\delta_\theta$. The discriminator is a linear model $D_\phi(x) = \phi \cdot x$ with
parameter $\phi$.

The loss function is given by:
\begin{align}
E(\theta, \phi) = l(\theta \phi) + l(0)
\label{eq:supp_dirac_gan}
\end{align}
where $f$ depends on the GAN used - for the standard GAN it is $l = - \log (1 + e^{-t})$. As in~\citet{mescheder2018training}, we assume $l$ is continuously differentiable with $l'(x) \ne 0$ for all $x \in \mathbb{R}$. The partial derivatives of the loss function
\begin{align*}
\frac{\partial E}{\partial\phi} = l'(\theta \phi) \theta, \hspace{3em}  \frac{\partial E}{\partial\theta} = l'(\theta \phi) \phi,
\end{align*}
lead to the underlying continuous dynamics:
\begin{align}
\dot{\phi} =  f(\theta, \phi) =  l'(\theta \phi) \theta, \hspace{3em} \dot{\theta} =  g(\theta, \phi) =  -l'(\theta \phi) \phi.
\label{eq:supp_dirac_gan_vector_field}
\end{align}

Thus the only equilibrium of the game is $\theta = 0$ and $\phi = 0$.

\subsection{Reconciling discrete and continuous updates in Dirac-GAN}

\citet{mescheder2018training} observed a discrepancy between the continuous and discrete dynamics.
They show that, for the problem in Equation~\eqref{eq:supp_dirac_gan}, the continuous dynamics preserve $\theta^2 + \phi^2$, and thus cannot converge (Lemma 2.3 in~\citet{mescheder2018training}), since:
\begin{align*}
\frac{d \left(\theta^2 + \phi^2\right)}{dt} = 2 \theta \frac{d \theta}{dt} + 2 \phi \frac{d \phi}{dt} = - 2 \theta l'(\theta \phi) \phi + 2 \phi l'(\theta \phi) \theta = 0.
\end{align*}

They also observe that with the discrete dynamics of simultaneous gradient descent that $\theta^2 + \phi^2$ increases in time (Lemma 2.4 in~\citet{mescheder2018training}). We resolve this discrepancy here, by showing that the modified continuous dynamics given by discretization drift result in behavior consistent with that of the discrete updates.

\begin{proposition} The continuous vector field given by discretization drift for simultaneous Euler updates in DiracGAN increases $\theta^2 + \phi^2$ in time.
\end{proposition}
\begin{proof}

We assume both the generator and the discriminator use learning rates $h$, as in~\citep{mescheder2018training}. We first compute terms used by the zero-sum Colloraries in the main paper.
\begin{align*}
\norm{\nabla_{\theta} E}^2 = l'(\theta \phi)^2 \phi^2 \\
\norm{\nabla_{\phi} E}^2 = l'(\theta \phi)^2 \theta^2
\end{align*}
and
\begin{align*}
\nabla_{\theta} \norm{\nabla_{\theta} E}^2 &= 2 \phi^3 l'(\theta \phi) l''(\theta \phi)\\
\nabla_{\theta} \norm{\nabla_{\phi} E}^2 &=  2 \theta l'(\theta \phi)^2 + 2 \theta^2 \phi l'(\theta \phi) l''(\theta \phi) \\
\nabla_{\phi} \norm{\nabla_{\theta} E}^2 &= 2 \phi l'(\theta \phi)^2 + 2 \phi^2 \theta l'(\theta \phi) l''(\theta \phi) \\
\nabla_{\phi} \norm{\nabla_{\phi} E}^2 &= 2 \theta^3  l'(\theta \phi) l''(\theta \phi)\\
\end{align*}
Thus, the modified ODEs are given by:
\begin{align*}
\dot{\phi} &=   l'(\theta \phi) \theta + \frac{h}{2} \left[- \theta^3  l'(\theta \phi) l''(\theta \phi) + \phi l'(\theta \phi)^2 +  \phi^2 \theta l'(\theta \phi) l''(\theta \phi)\right] \\
\dot{\theta} &= - l'(\theta \phi) \phi - \frac{h}{2} \left[ - \theta l'(\theta \phi)^2 - \theta^2 \phi l'(\theta \phi) l''(\theta \phi)  +  \phi^3 l'(\theta \phi) l''(\theta \phi) \right]
\end{align*}

By denoting $l'(\theta \phi)$ by $l'$ and $l''(\theta \phi)$ by $l''$, then we have:
\begin{align*}
\frac{d \left(\theta^2 + \phi^2\right)}{dt}
&= 2 \theta \frac{d \theta}{dt} + 2 \phi \frac{d \phi}{dt} \\
&= 2 \theta \left(- l' \phi - \frac{h}{2} \left[ - \theta l'^2 - \theta^2 \phi l'l''  +  \phi^3 l'l'' \right]\right)
  +  2\phi \left(l' \theta + \frac{h}{2} \left[- \theta^3  l'l'' +  \phi l'^2 +  \phi^2 \theta l'l''\right]\right) \\
  &= 2 \theta \left( - \frac{h}{2} \left[ - \theta l'^2 - \theta^2 \phi l'l''  +  \phi^3 l'l'' \right]\right)
    +  2 \phi \left(\frac{h}{2} \left[- \theta^3  l'l'' +  \phi l'^2 + \phi^2 \theta l'l''\right]\right) \\
  & = h \theta^2 l'^2 + h \theta^3 \phi l'l'' - h  \phi^3 \theta l'l'' - h \phi \theta^3 l'l'' + h \phi^2 l'^2 + h \phi^3 \theta l'l'' \\
  & =  h \theta^2 l'^2 + h \phi^2 l'^2 > 0
\end{align*}

for all $\phi,\theta\neq 0$, which shows that $\theta^2 + \phi^2$ is not preserved and it will strictly increase for all values away from the equilibrium (we have used the assumption that $l'(x) \ne 0, \forall x \in \mathbb{R}$).
\end{proof}

We have thus identified a continuous system which exhibits the same behavior as described by Lemma 2.4 in~\citet{mescheder2018training}, where a discrete system is analysed. Figure~\ref{fig:dirac_gan_sup} illustrates that the divergent behavior of simultaneous gradient descent in this case can be predicted from the dynamics of the modified continuous system given by backward error analysis.

\begin{figure}[t]
  \centering
  \includegraphics[width=0.45\columnwidth]{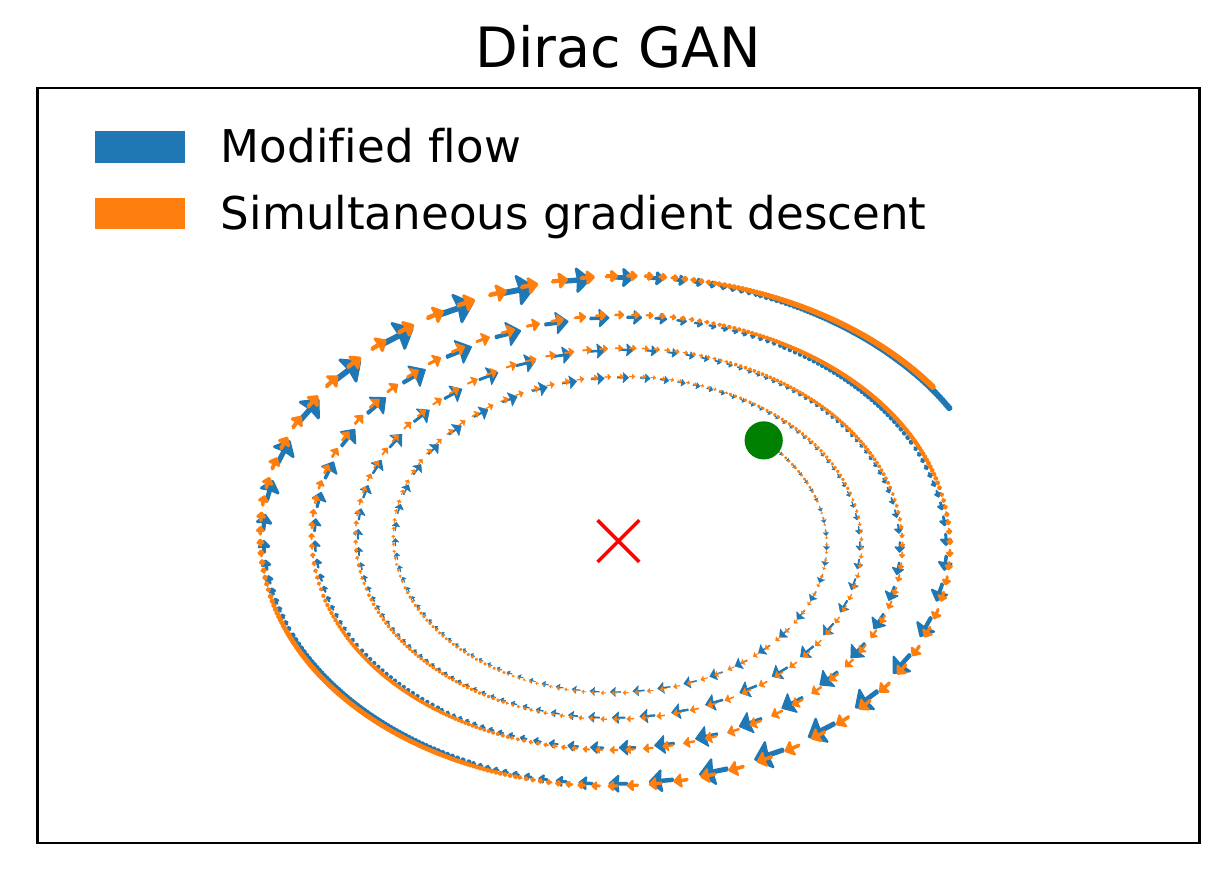}
  \includegraphics[width=0.45\columnwidth]{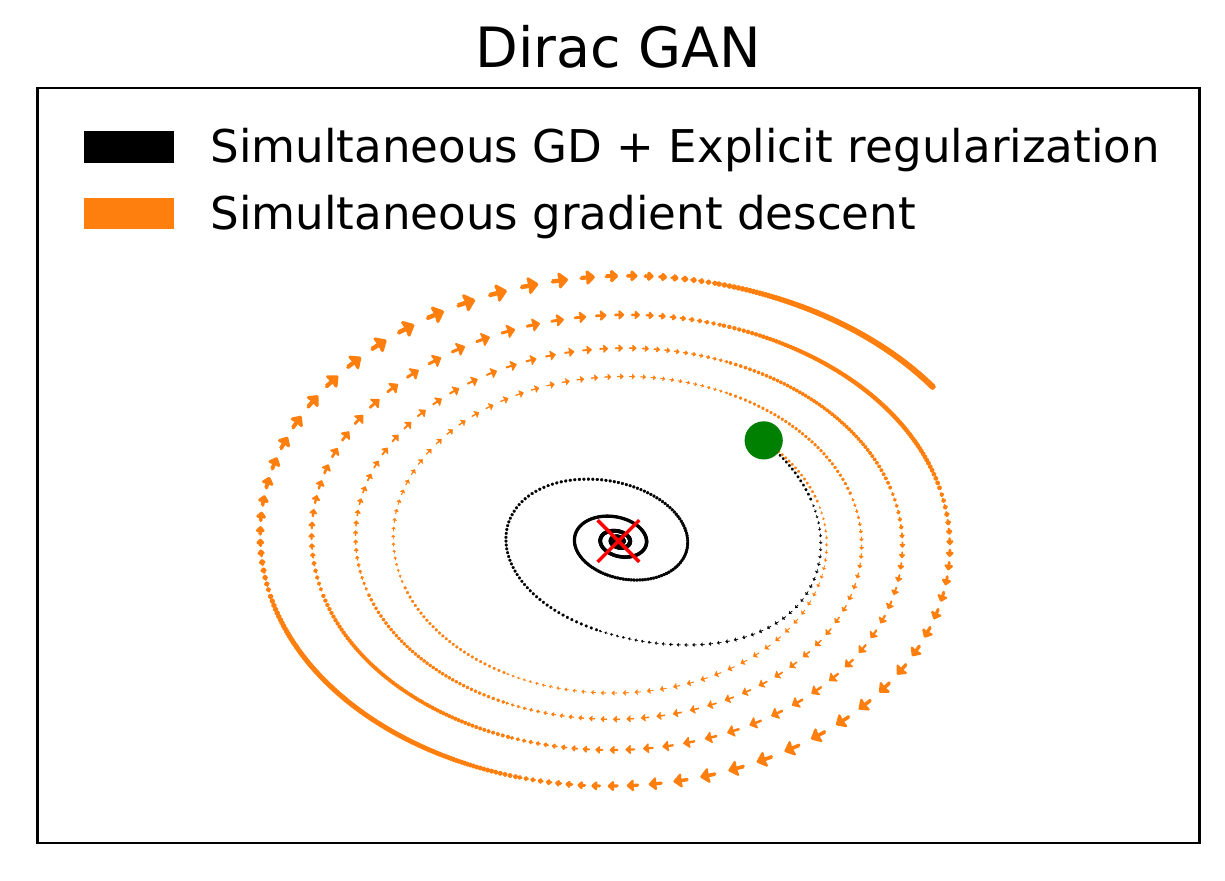}
  \caption{DiracGAN. Left: The dynamics of simultaneous gradient descent updates match the continuous dynamics given by backward error analysis for DiracGAN. Right: Explicit regularization canceling the interaction terms of DD stabilizes the DiracGAN trained with simultaneous gradient descent.}
  \label{fig:dirac_gan_sup}
\end{figure}

\subsection{DD changes the convergence behavior of Dirac-GAN}
\label{sec:dirac-gan-stability}
The Jacobian of the unmodified Dirac-GAN evaluated at the equilibrium solution given by $\phi = 0$, $\theta = 0$ is given by
\begin{gather}
 J
 =
  \begin{bmatrix}
    \nabla_{\phi,\phi}E & \nabla_{\theta,\phi}E\\
   - \nabla_{\phi,\theta}E & - \nabla_{\theta,\theta}E
   \end{bmatrix}
   =
  \begin{bmatrix}
    0&  l'(0) \\
    -l'(0) & 0
   \end{bmatrix}
\end{gather}
We see that $\Tr(J)=0$ and the determinant $|J| = l'(0)^2 $. Therefore, the eigenvalues of this Jacobian are $\lambda_{\pm} = \Tr(J)/2 \pm \sqrt{\Tr(J)^2 - 4|J|}/2 = \pm  il'(0) $ (Reproduced from \citet{mescheder2018training}). This is an example of a stable \emph{center equilibrium}.

Next we calculate the Jacobian of the modified ODEs given by DiracGAN, evaluated at an equilibrium solution and find $\widetilde{J} = J - h\Delta/2$, where
\begin{align*}
 \Delta
 =&
  \begin{bmatrix}
    \alpha ( \nabla_{\phi, \phi} E)^2 -  \alpha \nabla_{\phi,\theta}E\nabla_{\theta, \phi}E   & \alpha \nabla_{\theta, \phi}E\nabla_{\phi, \phi}E -  \alpha \nabla_{\theta,\theta}E\nabla_{\theta,\phi}E\\
   \lambda \nabla_{\phi,\theta}E\nabla_{\theta,\theta}E -  \lambda \nabla_{\phi,\phi}E\nabla_{\phi,\theta}E&
   \lambda( \nabla_{\theta,\theta} E)^2 -  \lambda \nabla_{\theta,\phi}E\nabla_{\phi,\theta}E
   \end{bmatrix}
   \\
   =&
     \begin{bmatrix}
     -  \alpha \nabla_{\phi,\theta}E\nabla_{\theta, \phi}E   & 0\\
   0& -  \lambda \nabla_{\theta,\phi}E\nabla_{\phi,\theta}E
   \end{bmatrix}
   \\
   =&
  \begin{bmatrix}
    -\alpha l'(0)^2 &  0 \\
    0 & -\lambda l'(0)^2
   \end{bmatrix}
\end{align*}
so
\begin{gather}
 \widetilde{J}
 =
  \begin{bmatrix}
    h\alpha l'(0)^2/2&  l'(0) \\
    -l'(0) & h\lambda l'(0)^2/2
   \end{bmatrix}
\end{gather}

Now, we see that the trace of the modified Jacobian for the Dirac-GAN is $\Tr(\widetilde{J})  = (h/2)(\alpha + \lambda ) l'(0)^2 > 0$, so the modified ODEs induced by gradient descent in DiracGAN are unstable.

\subsection{Explicit regularization stabilizes Dirac-GAN}
Here, we show that we can use our stability analysis to identify forms of explicit regularization that can counteract the destabilizing impact of DD. Consider the Dirac-GAN with explicit regularization of the following form:  $L_1  = - E + u{\| \nabla_{\theta} E\|}^2$ and $L_2  = E + \nu {\| \nabla_{\phi} E\|}^2$ where $\phi_t = \phi_{t-1}  -  \alpha h \nabla_{\phi} L_1 $ and  $\theta_t = \theta_{t-1}  -  \lambda h \nabla_{\theta} L_2  $ and with $u, \nu \sim \mathcal{O}(h)$. The modified Jacobian for this system is given by
\begin{align*}
 \widetilde{J}
   =&
  \begin{bmatrix}
    h\alpha/2 \nabla_{\phi,\theta}E\nabla_{\theta, \phi}E - u \nabla_{\phi, \phi} \norm{ \nabla_{\theta}E}^2  & \nabla_{\theta,\phi}E\\
   -\nabla_{\phi,\theta}E & h\lambda/2 \nabla_{\theta, \phi}E\nabla_{\phi, \theta}E -\nu \nabla_{\theta, \theta} \norm{ \nabla_{\phi}E}^2
   \end{bmatrix}
   \\
   =&
  \begin{bmatrix}
    (h\alpha/2 -2u)l'(0)^2  &  l'(0) \\
    -l'(0) &  (h\lambda/2 -2\nu)l'(0)^2
   \end{bmatrix}
\end{align*}
The determinant of the modified Jacobian is $|\widetilde{J}| =  (h\alpha/2 -2u)(h\lambda/2 -2\nu)l'(0)^4 + l'(0)^2$ and the trace is $\Tr(\widetilde{J})  = (h\alpha/2 -2u)l'(0)^2 + (h\lambda/2 -2\nu)l'(0)^2$. A necessary and sufficient condition for asymptotic stability is  $|\widetilde{J}|>0 $ and $\Tr(\widetilde{J}) <0$ (since this guarantees that the eigenvalues of the modified Jacobian have negative real part). Therefore, if $u > h\alpha /4$ and $\nu > h\lambda /4$, the system is asymptotically stable. We note however that in practice, when using discrete updates, the exact threshold for stability will have a $\mathcal{O}(h^3)$ correction, arising from the $\mathcal{O}(h^3)$ error in our backward error analysis. Also, we see that when  $u = h\alpha /4$ and $\nu = h\lambda /4$, the contribution of the cross-terms is cancelled out (up to an $\mathcal{O}(h^3)$ correction). In Figure~\ref{fig:dirac_gan_sup}, we see an example where this explicit regularization stabilizes the DiracGAN, so that it converges toward the equilibrium solution.

\section{Additional experimental results}

\subsection{Additional results using zero-sum GANs}

We show additional results showcasing the effect of DD on zero-sum games in Figure~\ref{fig:supp_idd_zero_sum_games}. We see that not only do simultaneous updates perform worse than alternating updates when using the best hyperparameters, but that simultaneous updating is much more sensitive to hyperparameter choice. We also see that multiple updates can improve the stability of a GAN trained using zero-sum losses, but this strongly depends on the choice of learning rate.

 \begin{figure}[t]
 \centering
  \includegraphics[width=0.35\columnwidth]{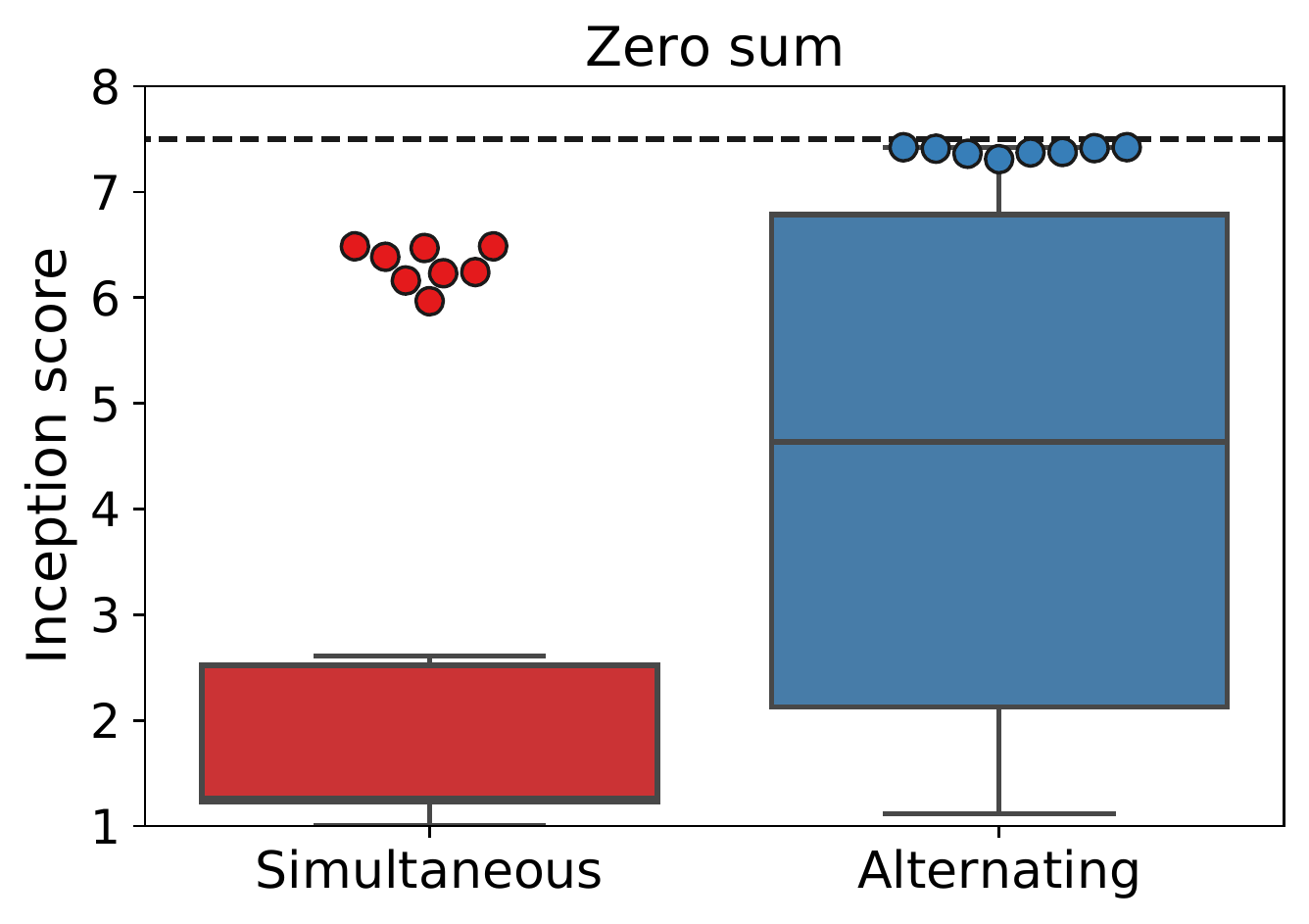}
  \includegraphics[width=0.35\columnwidth]{best_sim_vs_alt_zero_sum_is}\\
  \includegraphics[width=0.35\columnwidth]{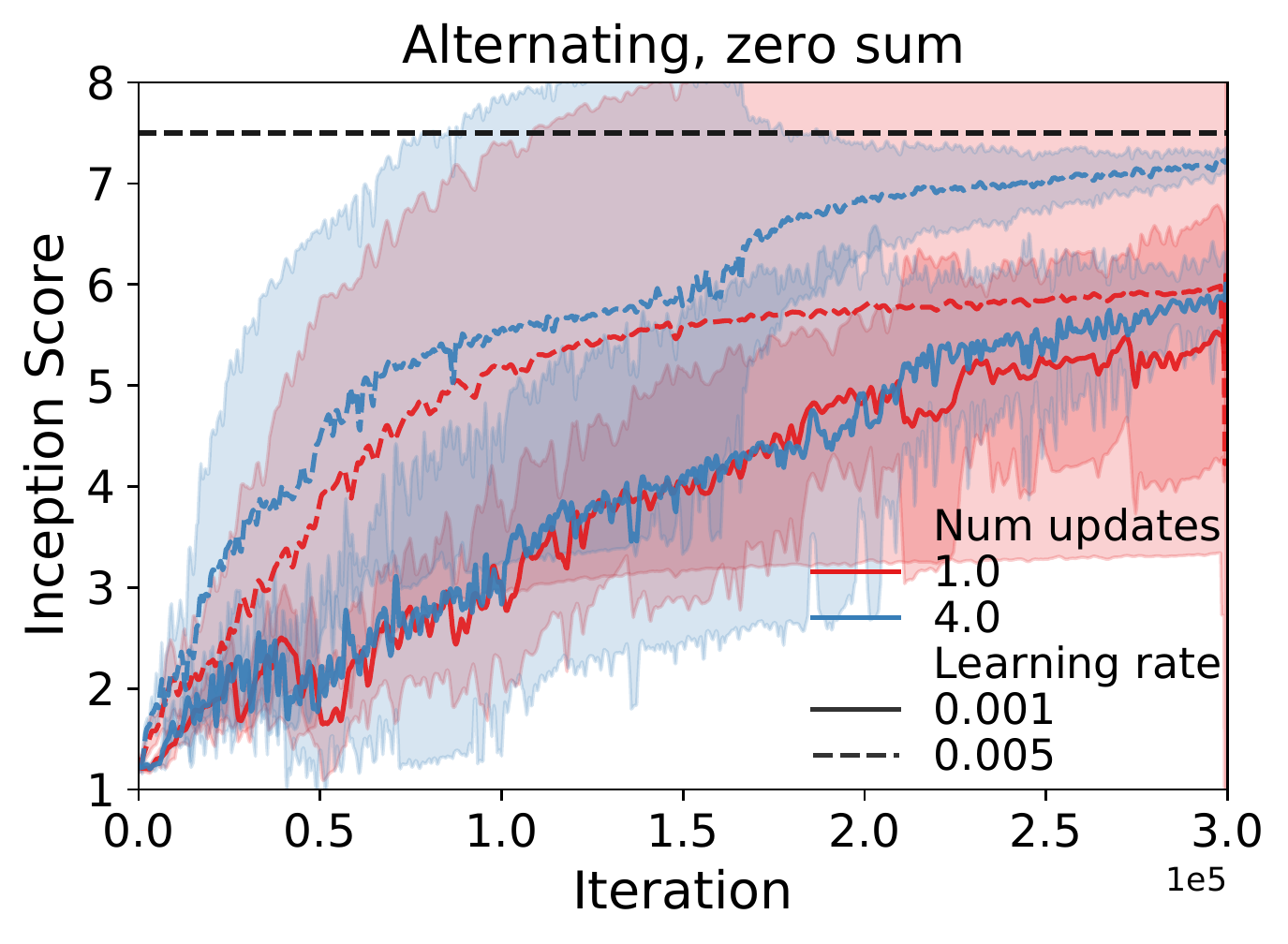}
  \includegraphics[width=0.35\columnwidth]{rk_euler_comparison_saturating_no_regularization}
    \caption{The effect of discretization drift on zero-sum games.}
   \label{fig:supp_idd_zero_sum_games}
\end{figure}

\subsubsection{Least squares GANs}

In order to assess the robustness of our results independent of the GAN loss used, we perform additional experimental results using GANs trained with a least square loss (LS-GAN~\citep{ls_gan}). We show results in Figure~\ref{fig:idd_zero_sum_games_ls_gan}, where we see that for the least square loss too, the learning rate ratios for which the generator drift does not maximize the discriminator norm (learning rate ratios above or equal to 0.5) perform best and exhibit less variance.

 \begin{figure}[t]
 \centering
  \includegraphics[width=0.3\columnwidth]{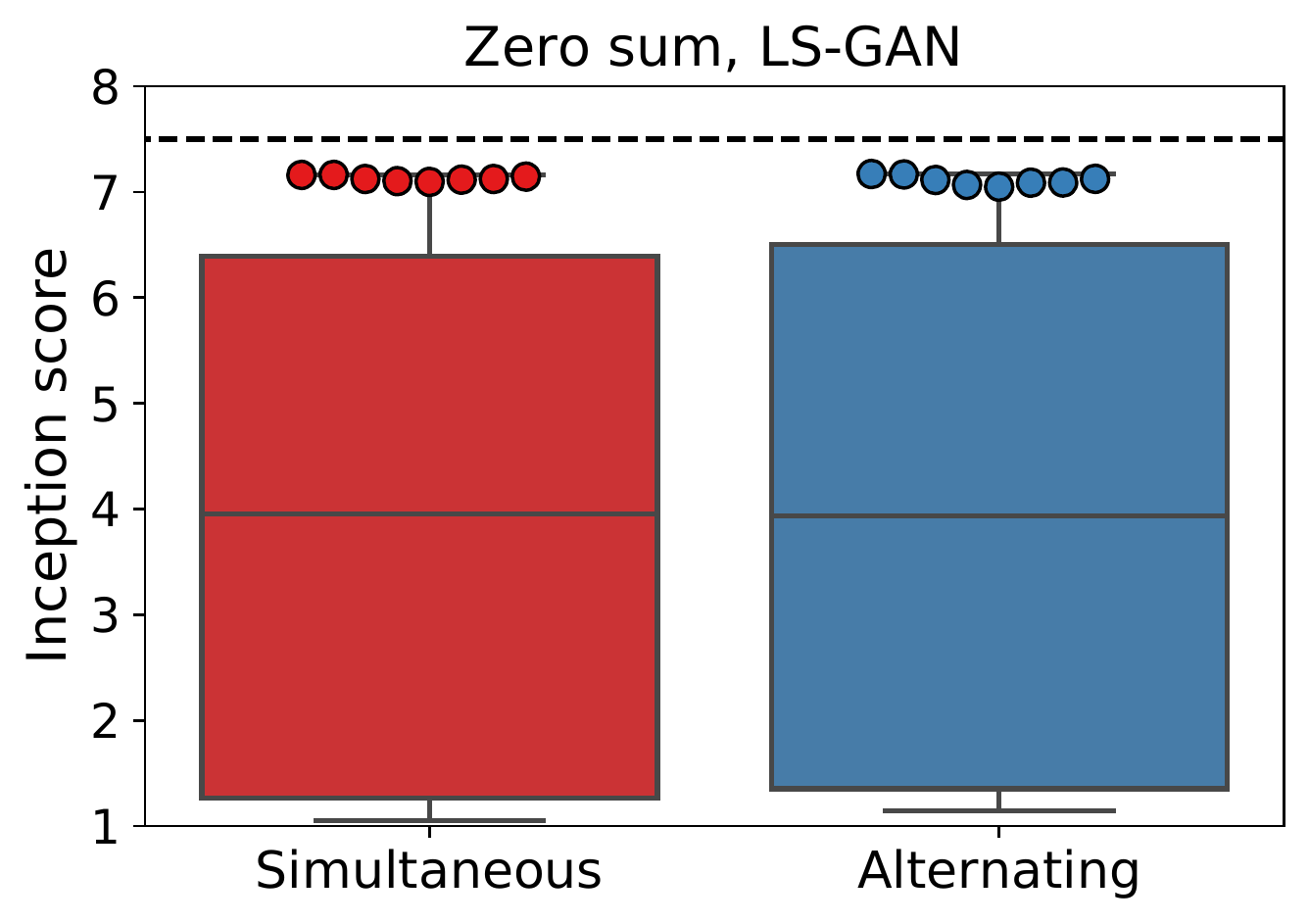}
  \includegraphics[width=0.3\columnwidth]{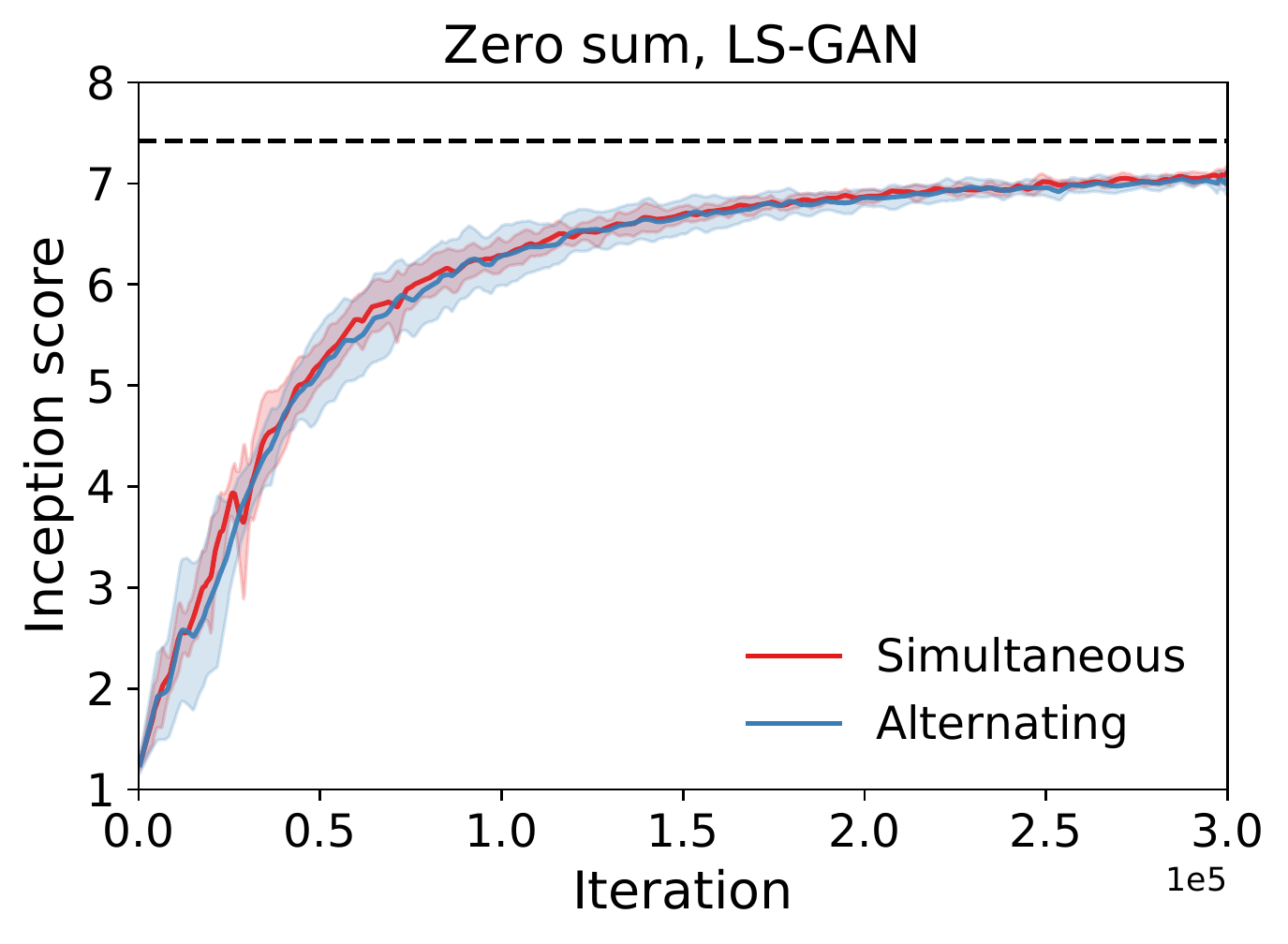}
  \includegraphics[width=0.3\columnwidth]{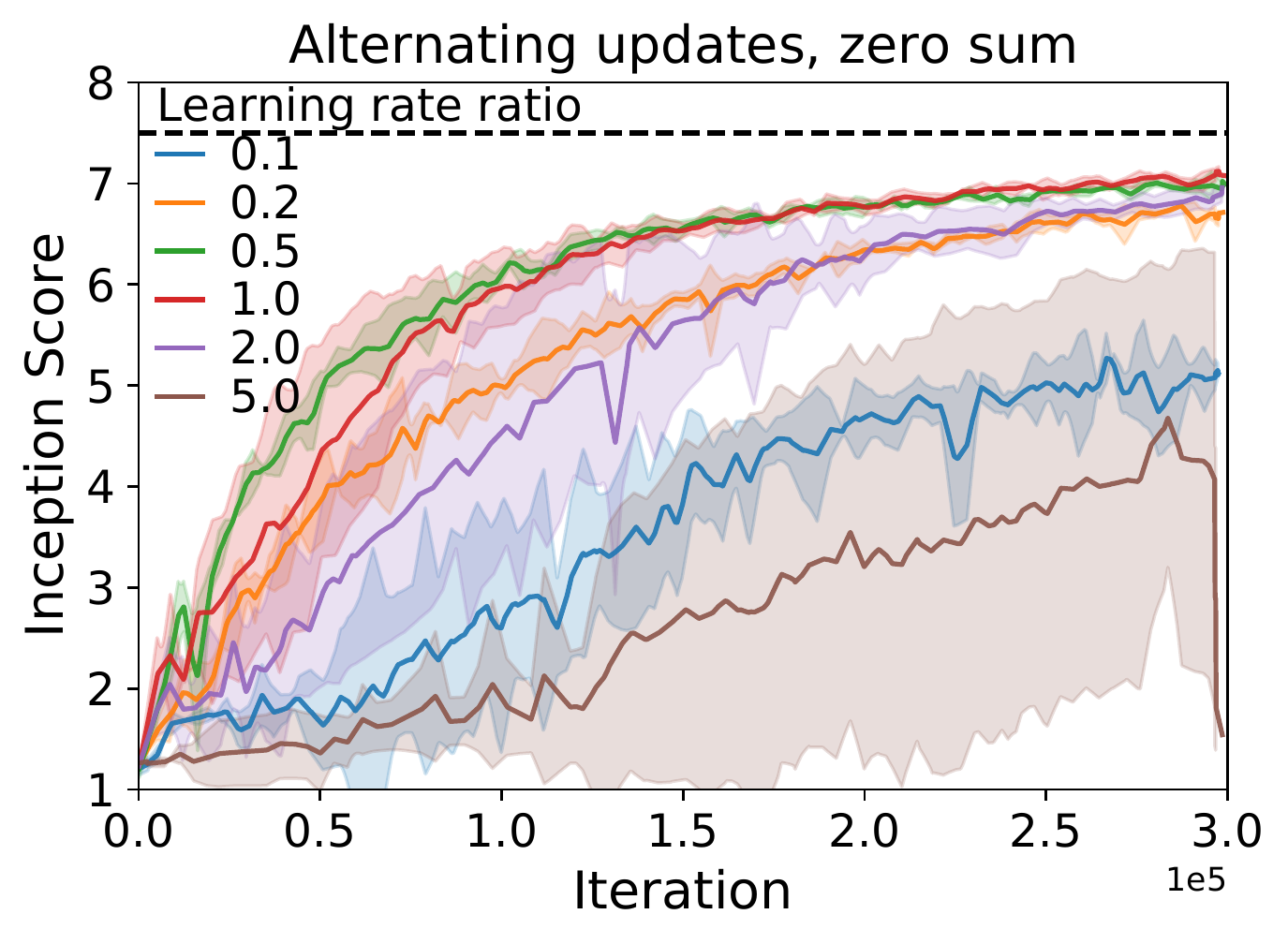}
  \caption{Least squares GAN: The effect of discretization drift on zero-sum games.}
   \label{fig:idd_zero_sum_games_ls_gan}
\end{figure}

\subsection{GANs using the non-saturating loss}

Next, we explore how the strength of the DD depends on the game dynamics. We do this by comparing the \textit{relative effect} that numerical integration schemes have across different games. To this end, we consider the non-saturating loss introduced in the original GAN paper ($- \log D_{\vphi}(G_{\vtheta}(\vz))$). This loss has been extensively used since it helps to avoid problematic gradients early in training. When using this loss, we see that there is little difference between simultaneous and alternating updates (Figure~\ref{fig:supp_non_saturating}), unlike the saturating loss case. These results demonstrate that since DD depends on the underlying dynamics, it is difficult to make general game-independent predictions about which numerical integrator will perform best.

\begin{figure*}[t]
  \centering
  \includegraphics[width=0.35\columnwidth]{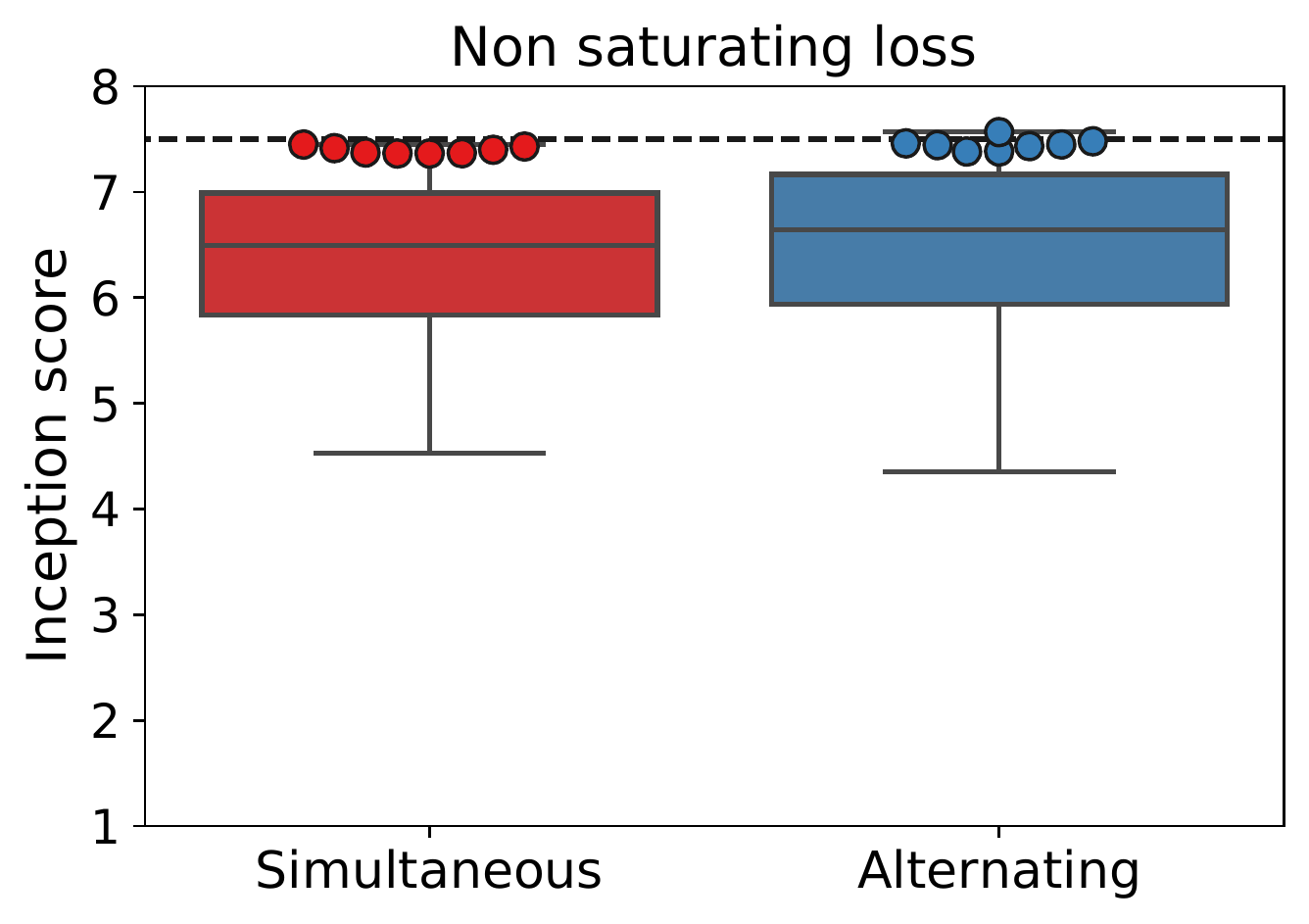}
  \includegraphics[width=0.35\columnwidth]{best_sim_vs_alt_zero_non_satuarting_is}\\
  \includegraphics[width=0.35\columnwidth]{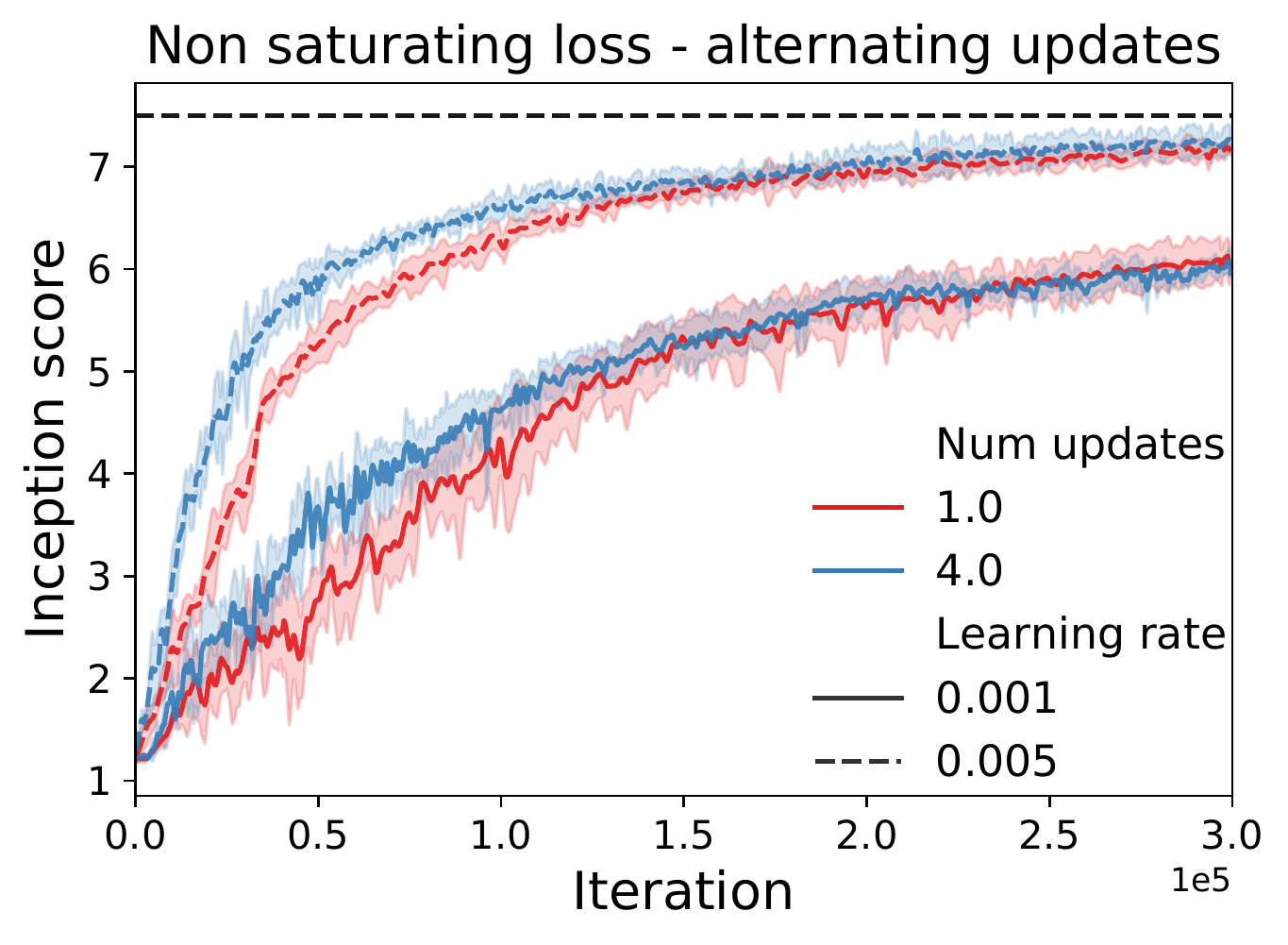}
  \includegraphics[width=0.35\columnwidth]{rk_euler_comparison_non_saturating_no_regularization}
  \caption{The effect of discretization drift depends on the game: with the non saturating loss, the relative performance of different numerical estimators is different compared to the saturating loss, and the effect of explicit regularization is also vastly different.}
  \label{fig:supp_non_saturating}
\end{figure*}

\subsection{Explicit regularization in zero-sum games trained using simultaneous gradient descent}

We now show additional experimental results and visualizations obtained using explicit regularization obtained using the original zero sum GAN objective, as presented in~\citet{goodfellow2014generative}.

We show the improvement that can be obtained compared to gradient descent with simultaneous updates by canceling the interaction terms in Figure~\ref{fig:sgd_vs_cancel_drift_interaction}.
We additionally show results obtained from strengthening the self terms in Figure~\ref{fig:sgd_vs_cancel_drift_interaction_all_types_reg}.

In Figure~\ref{fig:adam_comparison}, we show that by cancelling interaction terms, SGD becomes a competitive optimization algorithm when training the original GAN. We also note that in the case of Adam, while convergence is substantially faster than with SGD, we notice a degrade in performance later in training. This is something that has been observed in other works as well (e.g.~\citep{odegan}).

\begin{figure}[t]
 \centering
  \begin{subfigure}[Inception Score ($\uparrow$).]{
  \includegraphics[width=0.31\columnwidth]{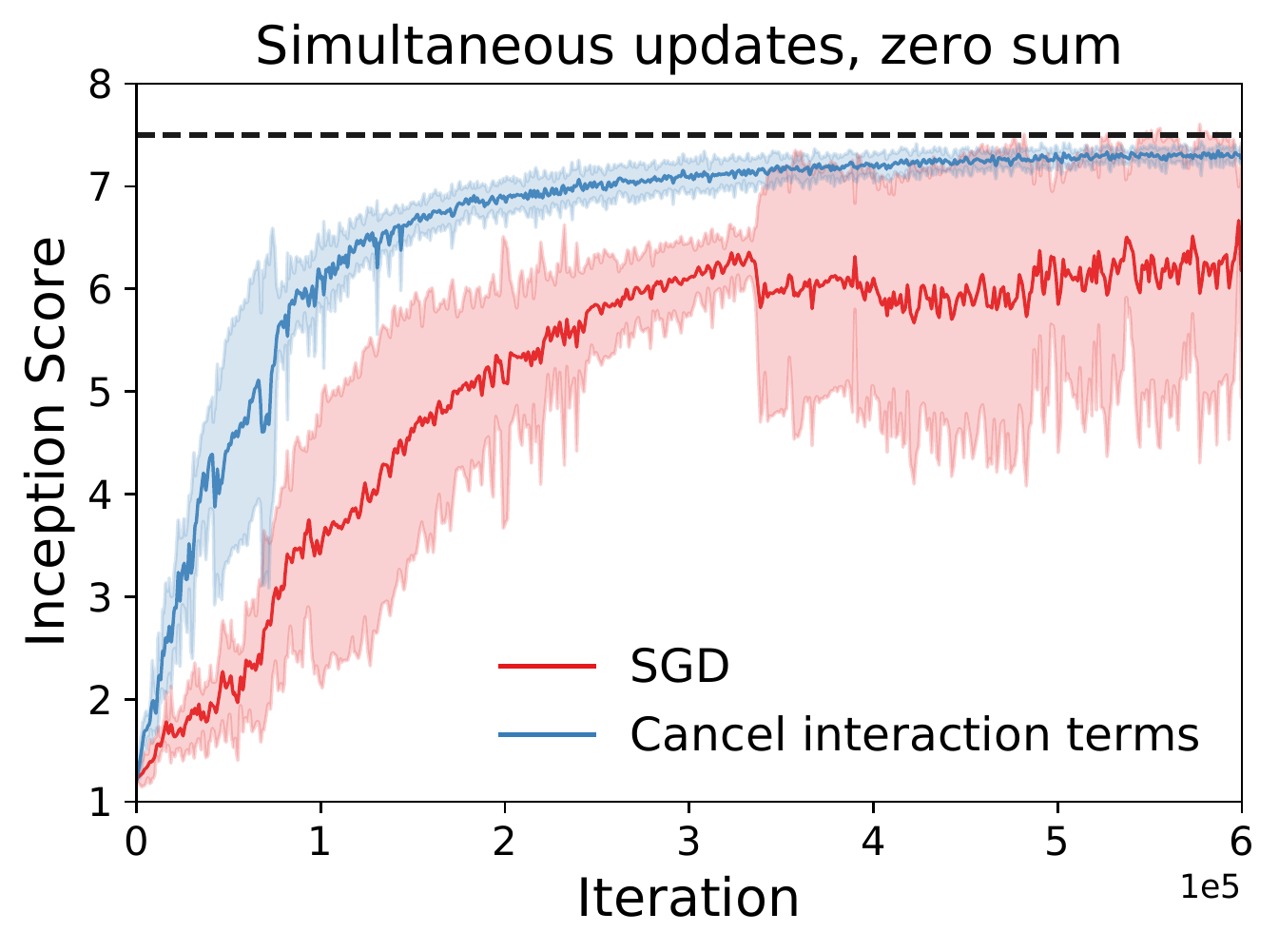}
} \end{subfigure}
 \begin{subfigure}[Frechet Inception Distance ($\downarrow$).]{
  \includegraphics[width=0.31\columnwidth]{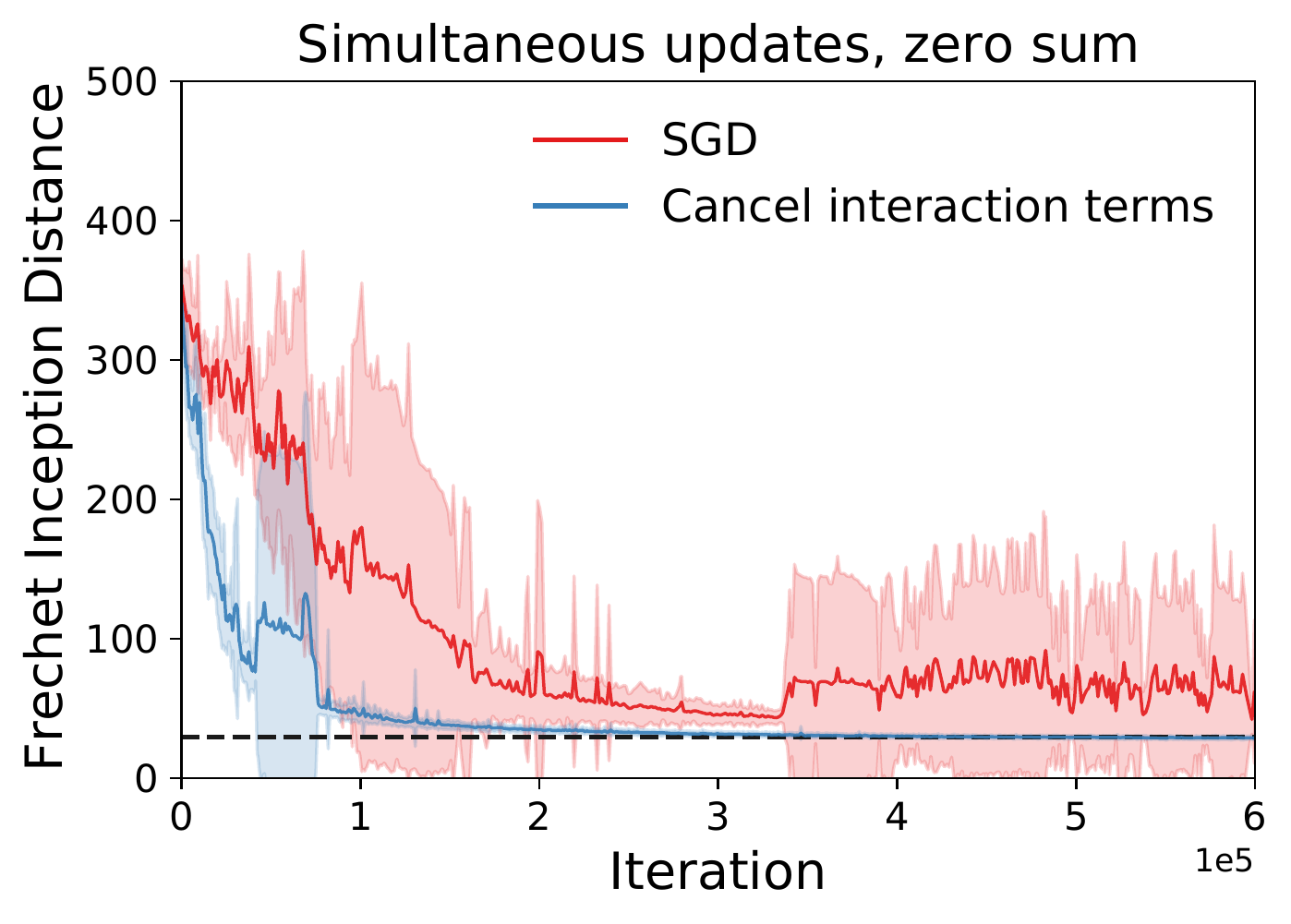}
} \end{subfigure}
  \caption{Using explicit regularization to cancel the effect of the interaction components of drift eads to a substantial improvement compared to SGD without explicit regularization.}
  \label{fig:sgd_vs_cancel_drift_interaction}
\end{figure}

\begin{figure}[t]
 \centering
  \begin{subfigure}[Inception Score ($\uparrow$).]{
  \includegraphics[width=0.31\columnwidth]{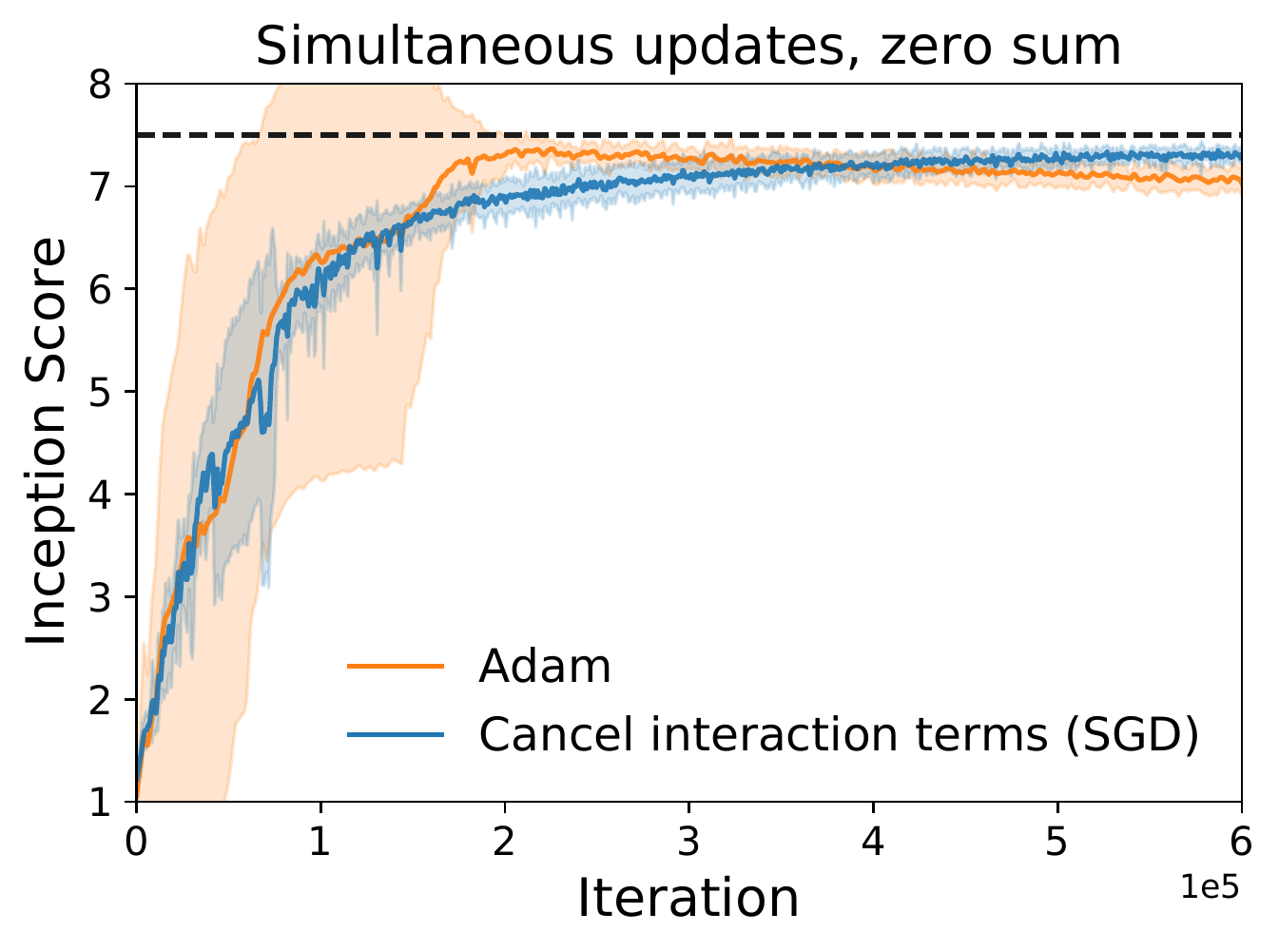}
} \end{subfigure}
 \begin{subfigure}[Frechet Inception Distance ($\downarrow$).]{
  \includegraphics[width=0.31\columnwidth]{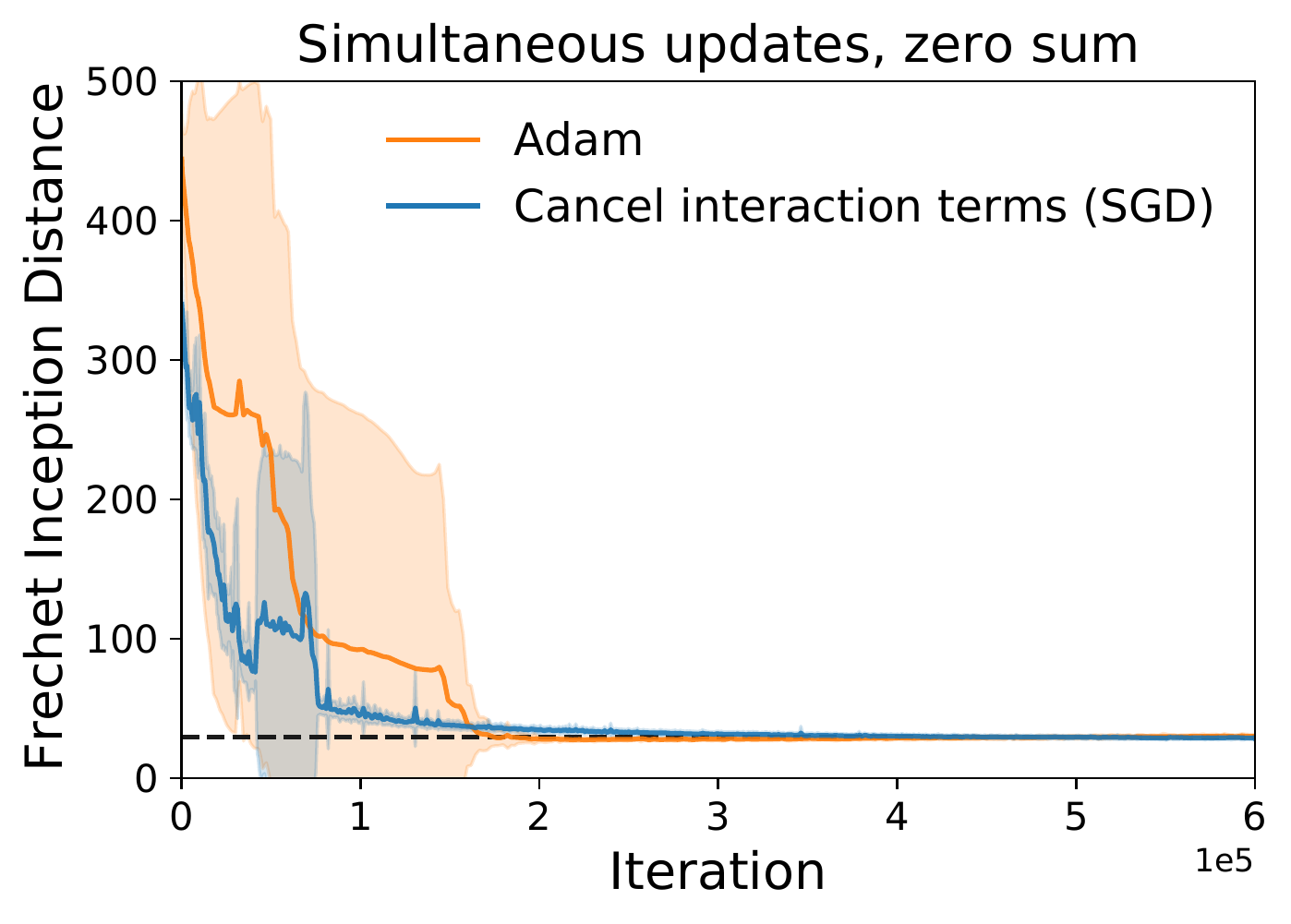}
} \end{subfigure}
  \caption{Using explicit regularization to cancel the effect of the interaction components of drift allows us to obtain the same peak performance as Adam using SGD without momentum.}
  \label{fig:adam_comparison}
\end{figure}

\begin{figure}[t]
 \centering
  \begin{subfigure}[Inception Score ($\uparrow$).]{
  \includegraphics[width=0.31\columnwidth]{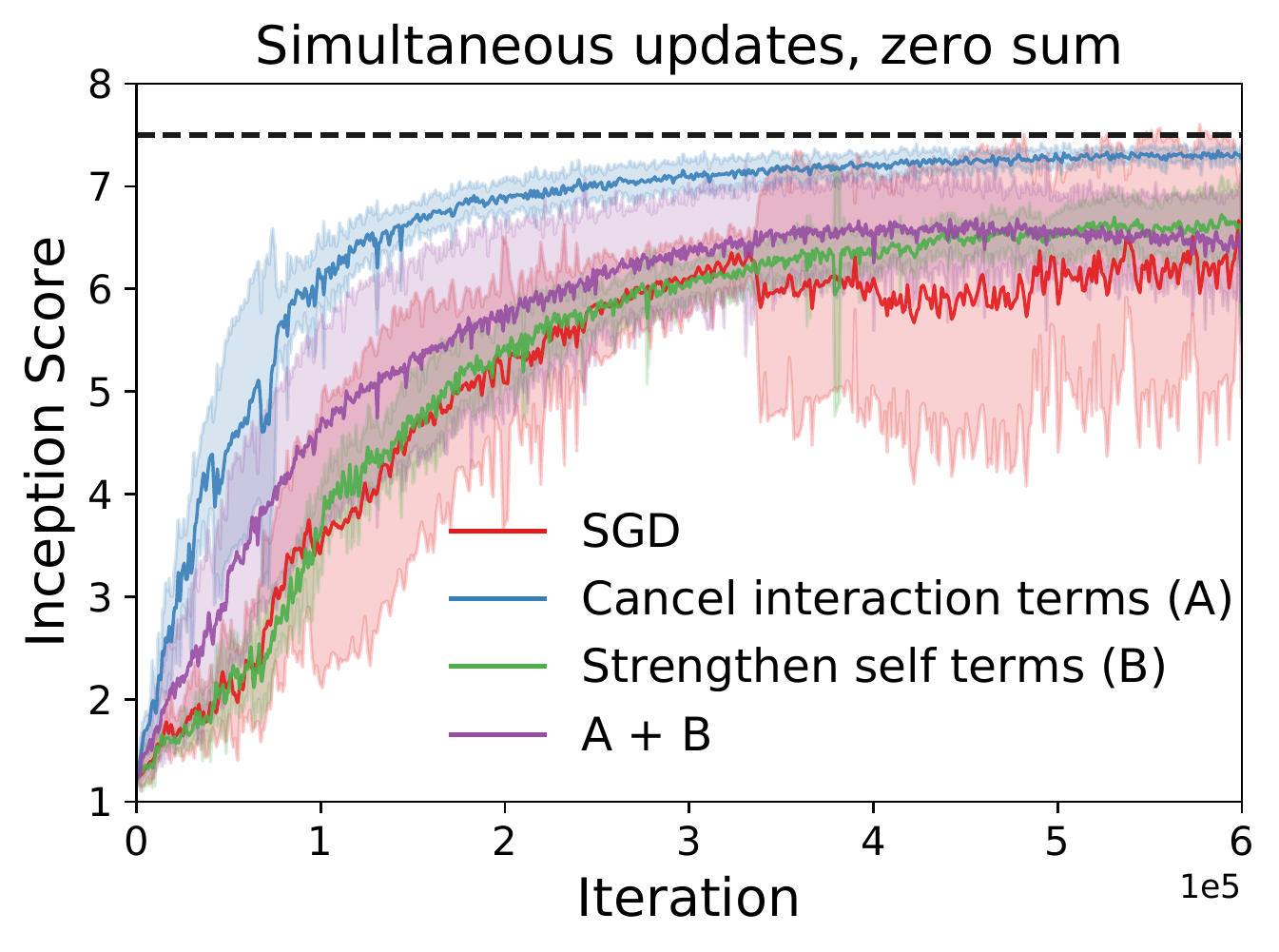}
} \end{subfigure}
 \begin{subfigure}[Frechet Inception Distance ($\downarrow$).]{
  \includegraphics[width=0.31\columnwidth]{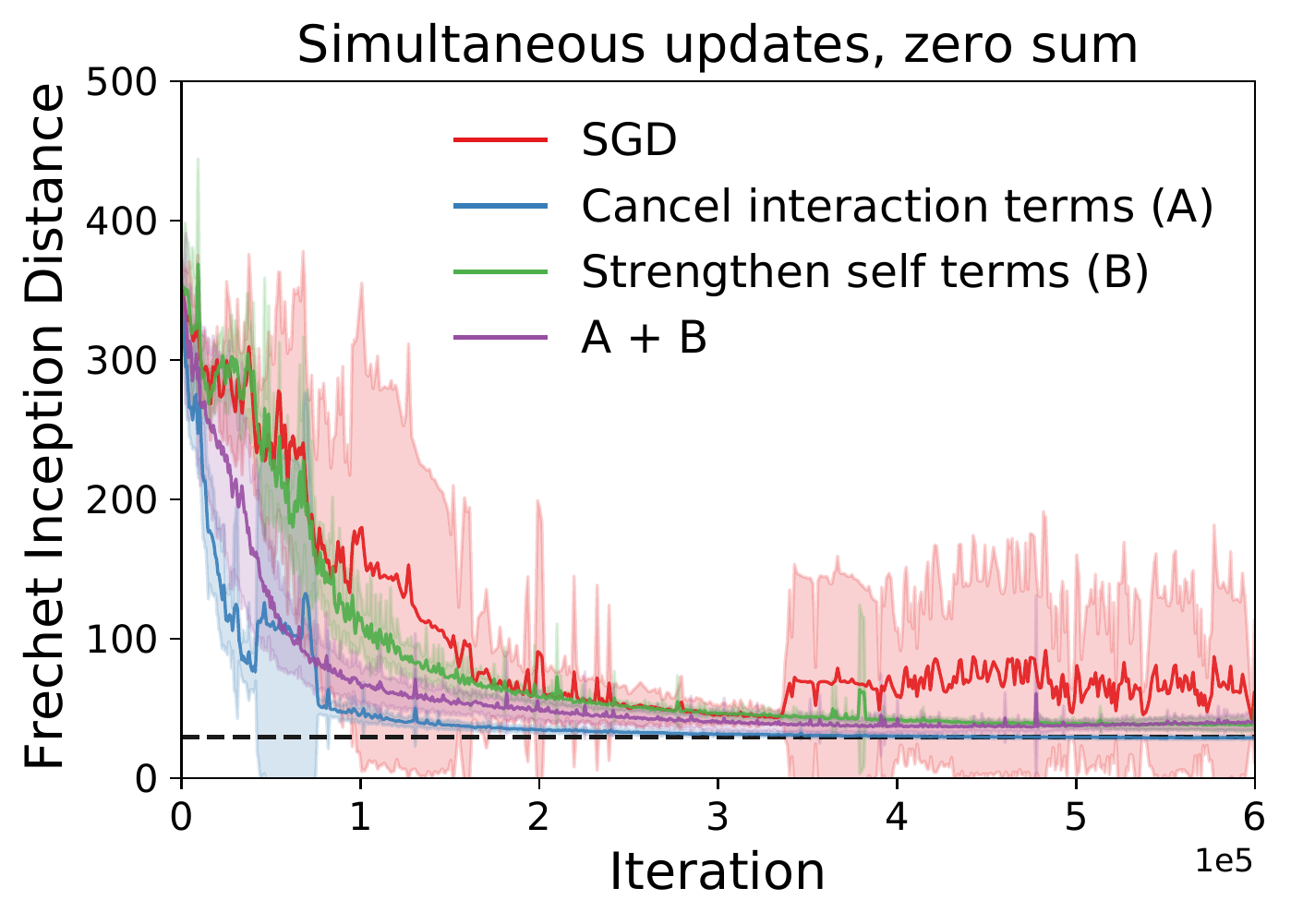}
} \end{subfigure}
  \caption{Using explicit regularization to cancel the effect of the drift eads to a substantial improvement compared to SGD without explicit regularization. Strengthening the self terms -- the terms which minimize the player's own norm -- does not lead to a substantial improvement; this is somewhat expected since while the modified ODEs give us the exact coefficients required to \textit{cancel} the drift, they do not tell us how to strengthen it, and our choice of exact coefficients from the drift might not be optimal.}
  \label{fig:sgd_vs_cancel_drift_interaction_all_types_reg}
\end{figure}

\subsubsection{More percentiles}

Throughout the main paper, we displayed the best $10\%$ performing models for each optimization algorithm used. We now expand that to show performance results across the best $20\%$ and $30\%$ of models in Figure~\ref{fig:multiple_percentages_sgd_cancel_interaction_terms}. We observe a consistent increase in performance obtained by canceling the interaction terms.

\begin{figure}[t]
 \centering
  \begin{subfigure}[Top 10\% models.]{
  \includegraphics[width=0.31\columnwidth]{cancel_interaction_sgd_comparison_simultaneous_is_10}
} \end{subfigure}
 \begin{subfigure}[Top 20\% models.]{
  \includegraphics[width=0.31\columnwidth]{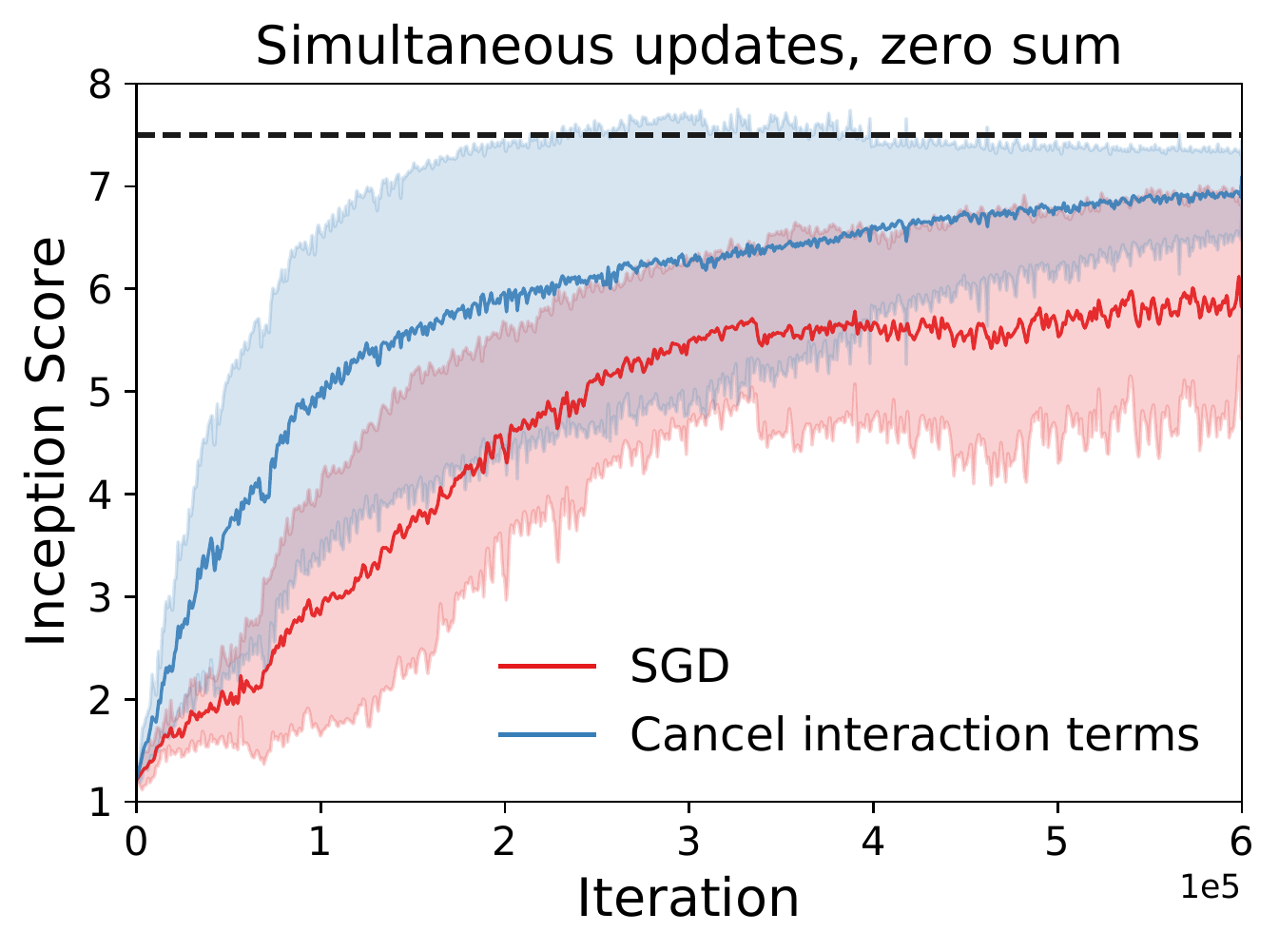}
} \end{subfigure}
\begin{subfigure}[Top 30\% models.]{
  \includegraphics[width=0.31\columnwidth]{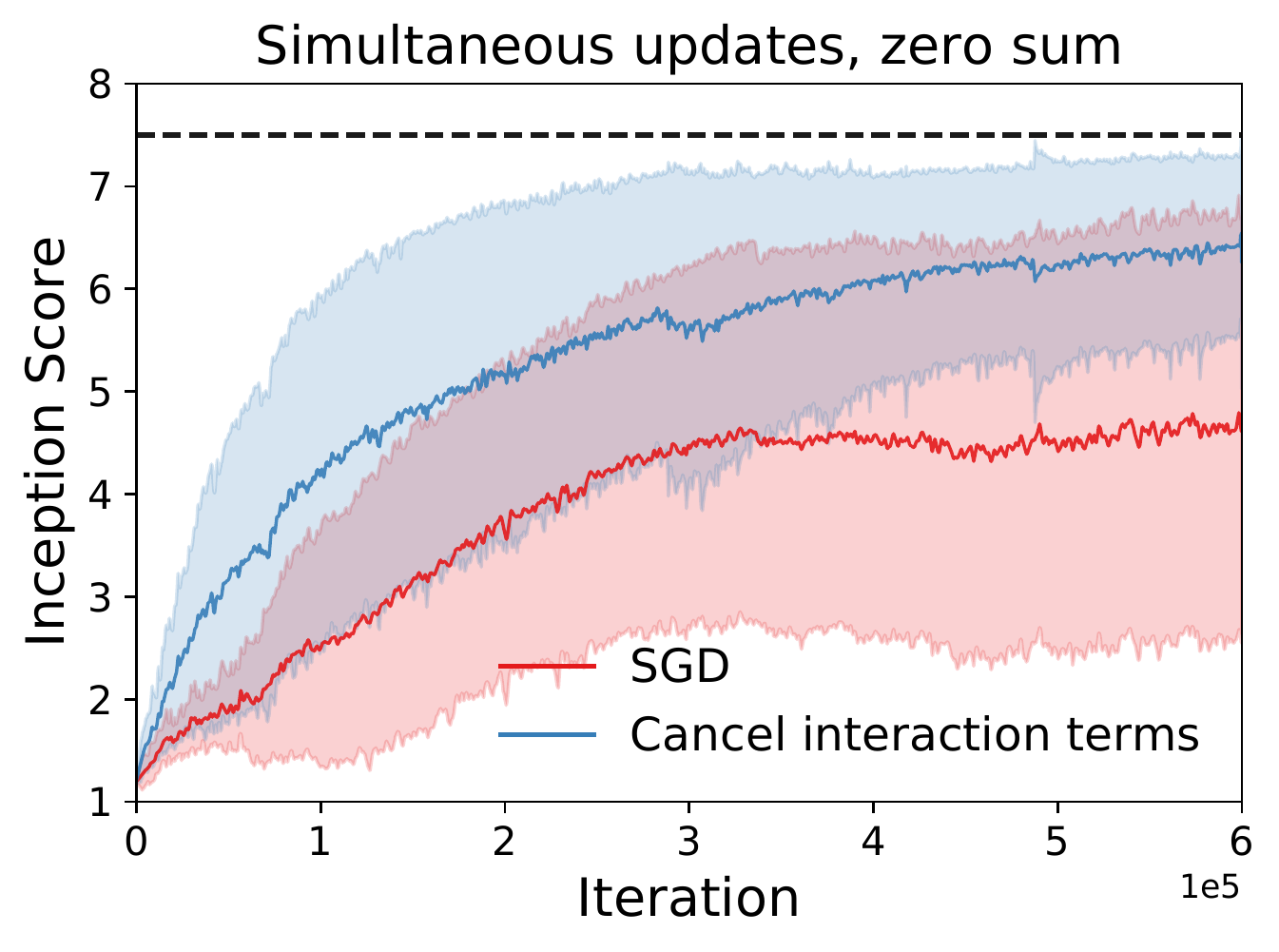}
} \end{subfigure}
  \caption{Performance across the top performing models for vanilla SGD and with canceling the interaction terms. Across all percentages, canceling the interaction terms improves performance.}
  \label{fig:multiple_percentages_sgd_cancel_interaction_terms}
\end{figure}

\subsubsection{Batch size comparison}

We now show that the results which show the efficacy of canceling the interaction terms are resiliant to changes in batch size in Figure~\ref{fig:batch_size_comparison}.

\begin{figure}[t]
 \centering
  \begin{subfigure}[Batch size 64.]{
  \includegraphics[width=0.31\columnwidth]{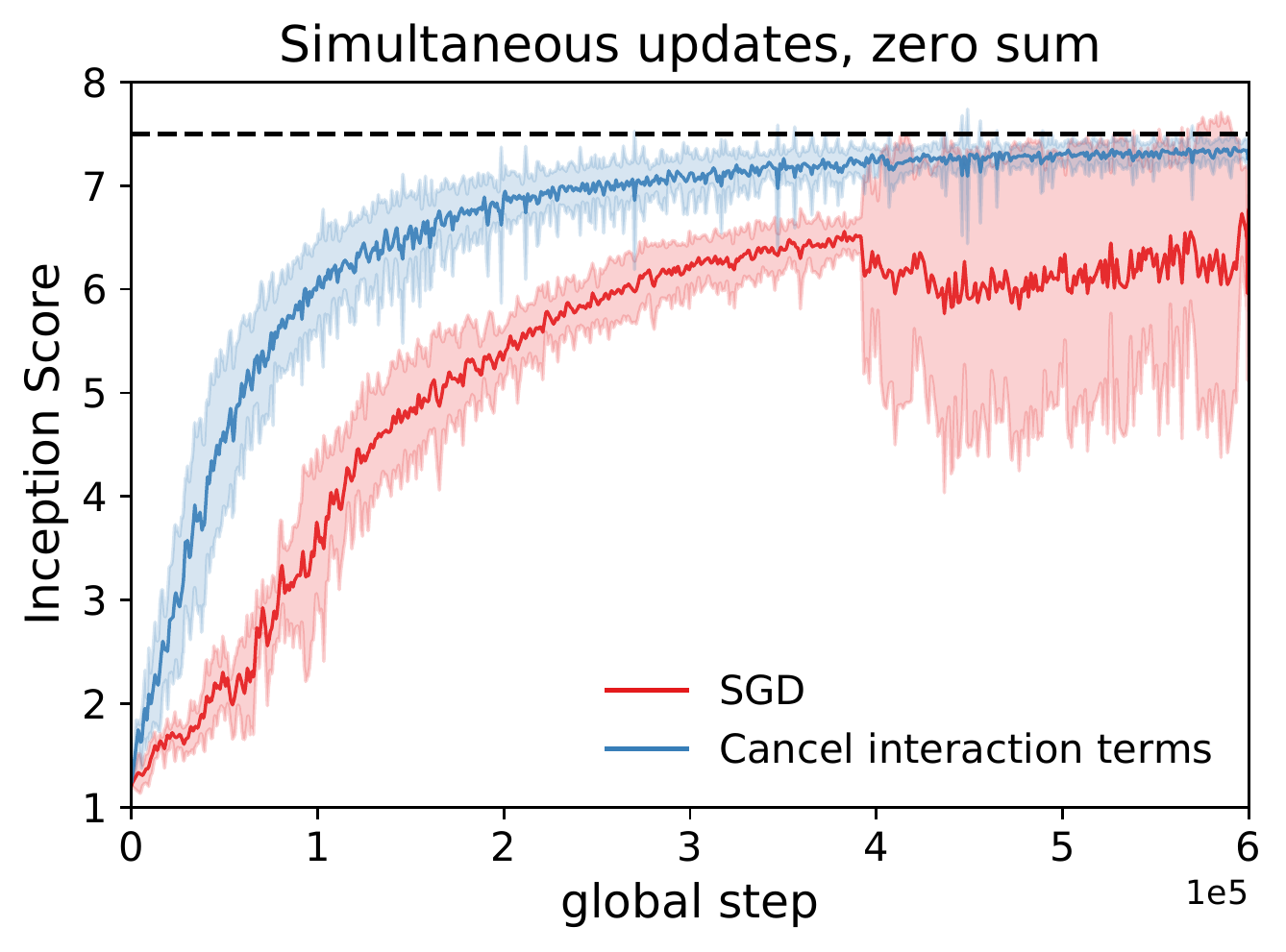}
} \end{subfigure}
 \begin{subfigure}[Batch size 128.]{
  \includegraphics[width=0.31\columnwidth]{cancel_interaction_sgd_comparison_simultaneous_is_10}
} \end{subfigure}
\begin{subfigure}[Batch size 256.]{
  \includegraphics[width=0.31\columnwidth]{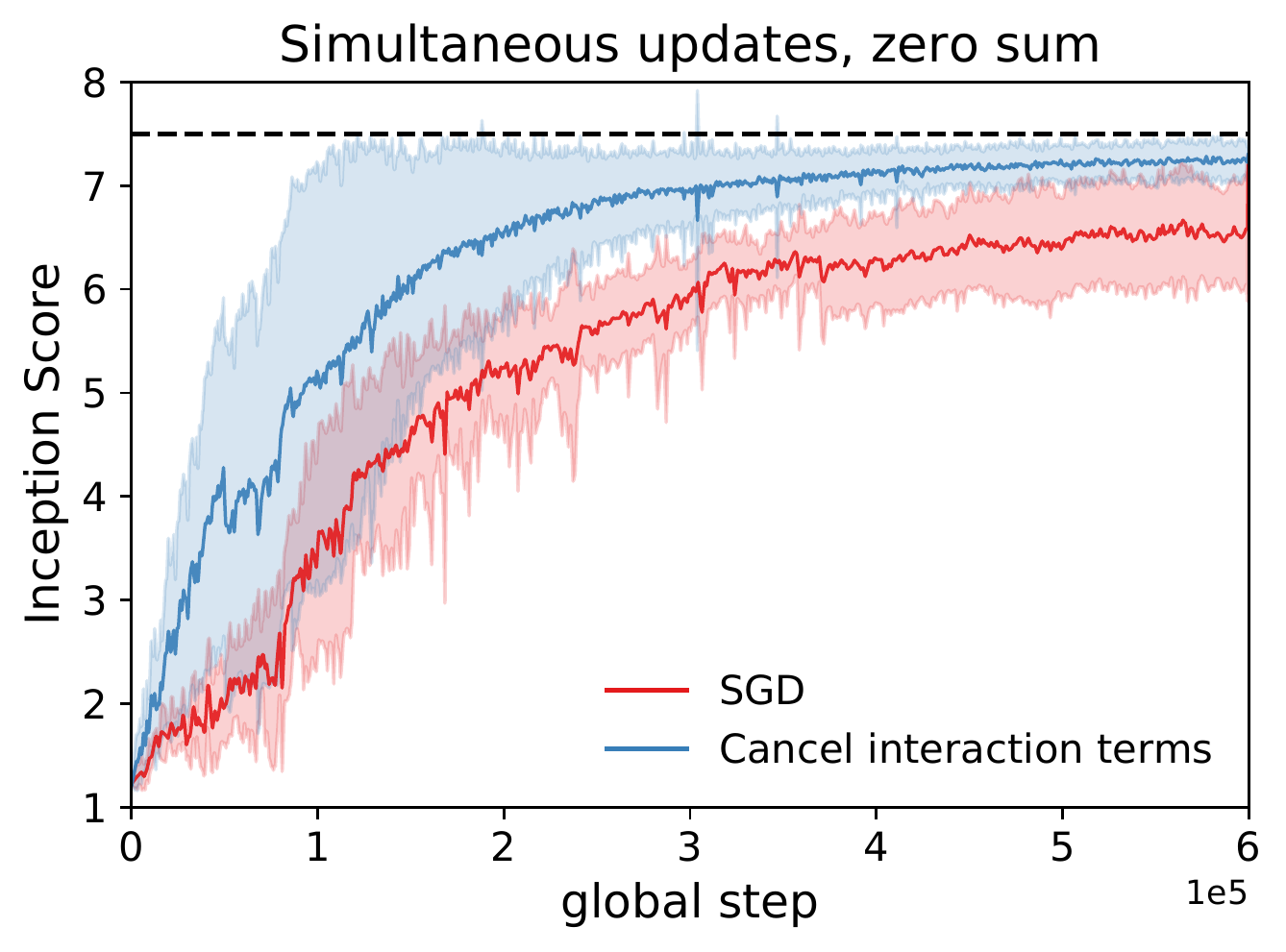}
} \end{subfigure}
  \caption{Performance when changing the batch size. We consistently see that canceling te interaction terms improves performance.}
  \label{fig:batch_size_comparison}
\end{figure}

\subsubsection{Comparison with Symplectic Gradient Adjustment}

We show results comparing with Symplectic Gradient Adjustment (SGA) ~\citep{balduzzi2018mechanics} in Figure~\ref{fig:supp_sga_comp} (best performing models) and Figure~\ref{fig:sga_box_plots} (quantiles showing performance across all hyperparameters and seeds). We observe that despite having the same functional form as SGA, canceling the interaction terms of DD performs better; this is due to the choice of regularization coefficients, which in the case of canceling the interaction terms is provided by Corollary~\ref{cor:zs-sim}.

\begin{figure}[t]
 \centering
  \begin{subfigure}[Inception Score ($\uparrow$).]{
  \includegraphics[width=0.31\columnwidth]{explicit_regularization_sim_updates_comparison_with_sga_is}
} \end{subfigure}
 \begin{subfigure}[Frechet Inception Distance ($\downarrow$).]{
  \includegraphics[width=0.31\columnwidth]{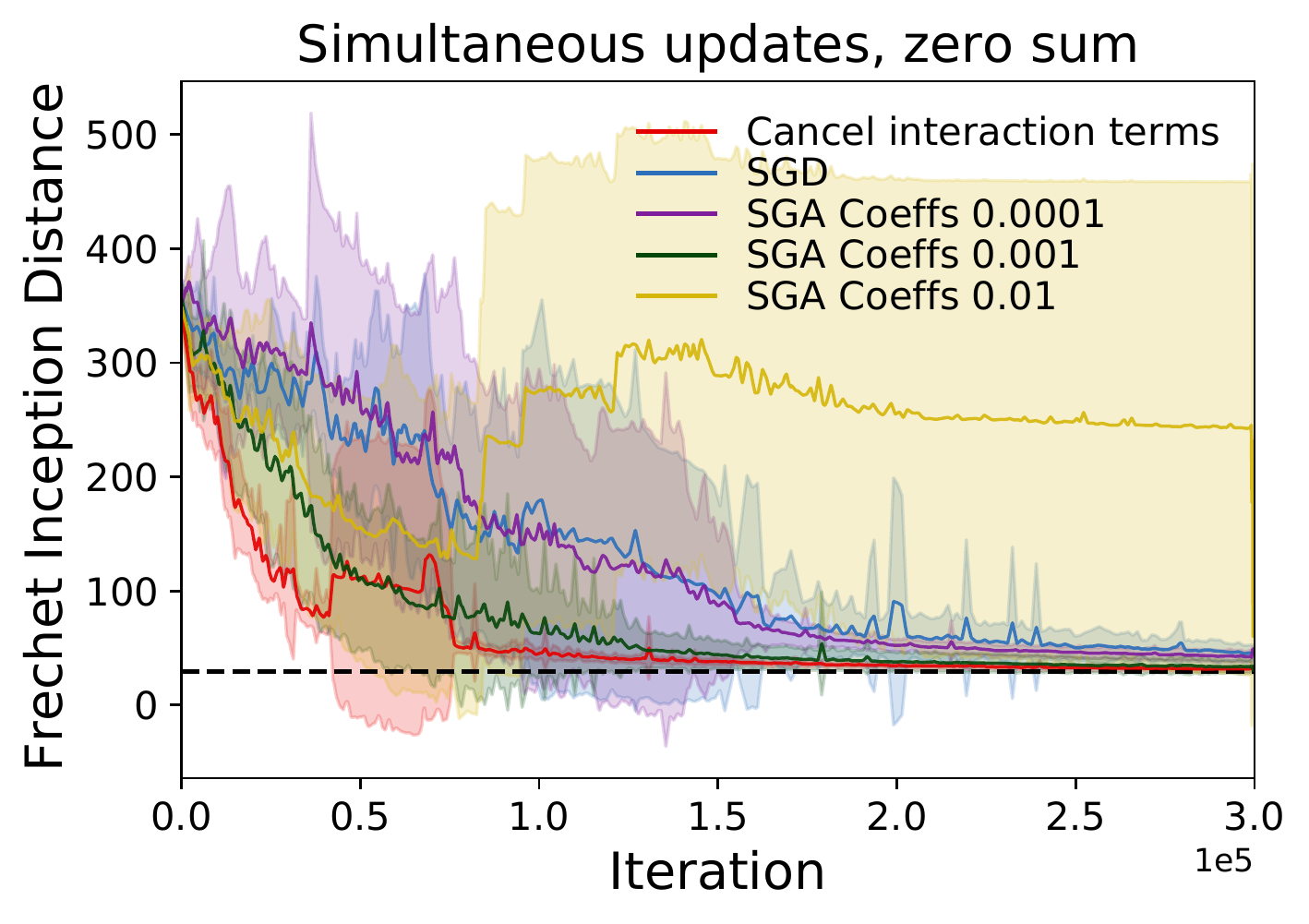}
} \end{subfigure}
  \caption{Comparison with Symplectic Gradient Adjustment (SGA). Canceling the interaction terms results in similar performance, but with less variance. The performance of SGA heavily depends on the strength of regularization, adding another hyperparmeter to the hyperparameter sweep, while canceling the interaction terms of the drift requires no other hyperparameters, since the explicit regularization coefficients strictly depend on learning rates.}
  \label{fig:supp_sga_comp}
\end{figure}

\begin{figure}[t]
 \centering
  \begin{subfigure}{
  \includegraphics[width=0.8\columnwidth]{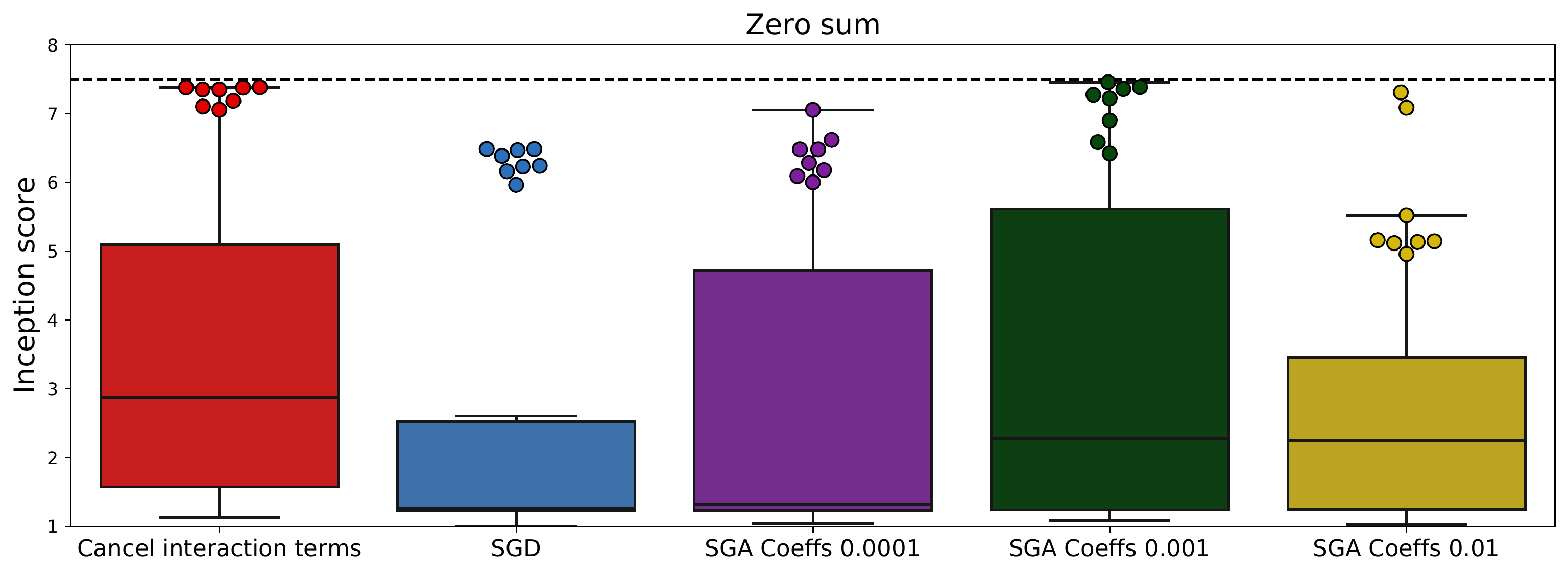}
} \end{subfigure}
  \caption{Comparison with Symplectic Gradient Adjustment (SGA),  obtained from \textit{all models in the sweep}. Without requiring an additional sweep over the regularization coefficient, canceling the interaction terms results in better performance across the learning rate sweep and less sensitivity to hyperparameters.}
  \label{fig:sga_box_plots}
\end{figure}

\subsubsection{Comparison with Consensus Optimization}

We show results comparing with Consensus Optimization (CO)~\citep{mescheder2017numerics} in Figure~\ref{fig:consensus_opt_com} (best performing models) and Figure~\ref{fig:consensus_opt_box_plots} (quantiles showing performance across all hyperparameters and seeds). We observe that canceling the interaction terms performs best, and that additionally strengthening the self terms does not provide a performance benefit.

\begin{figure}[t]
 \centering
  \begin{subfigure}[Inception Score ($\uparrow$).]{
  \includegraphics[width=0.31\columnwidth]{consensus_opt_comp_is}
} \end{subfigure}
 \begin{subfigure}[Frechet Inception Distance ($\downarrow$).]{
  \includegraphics[width=0.31\columnwidth]{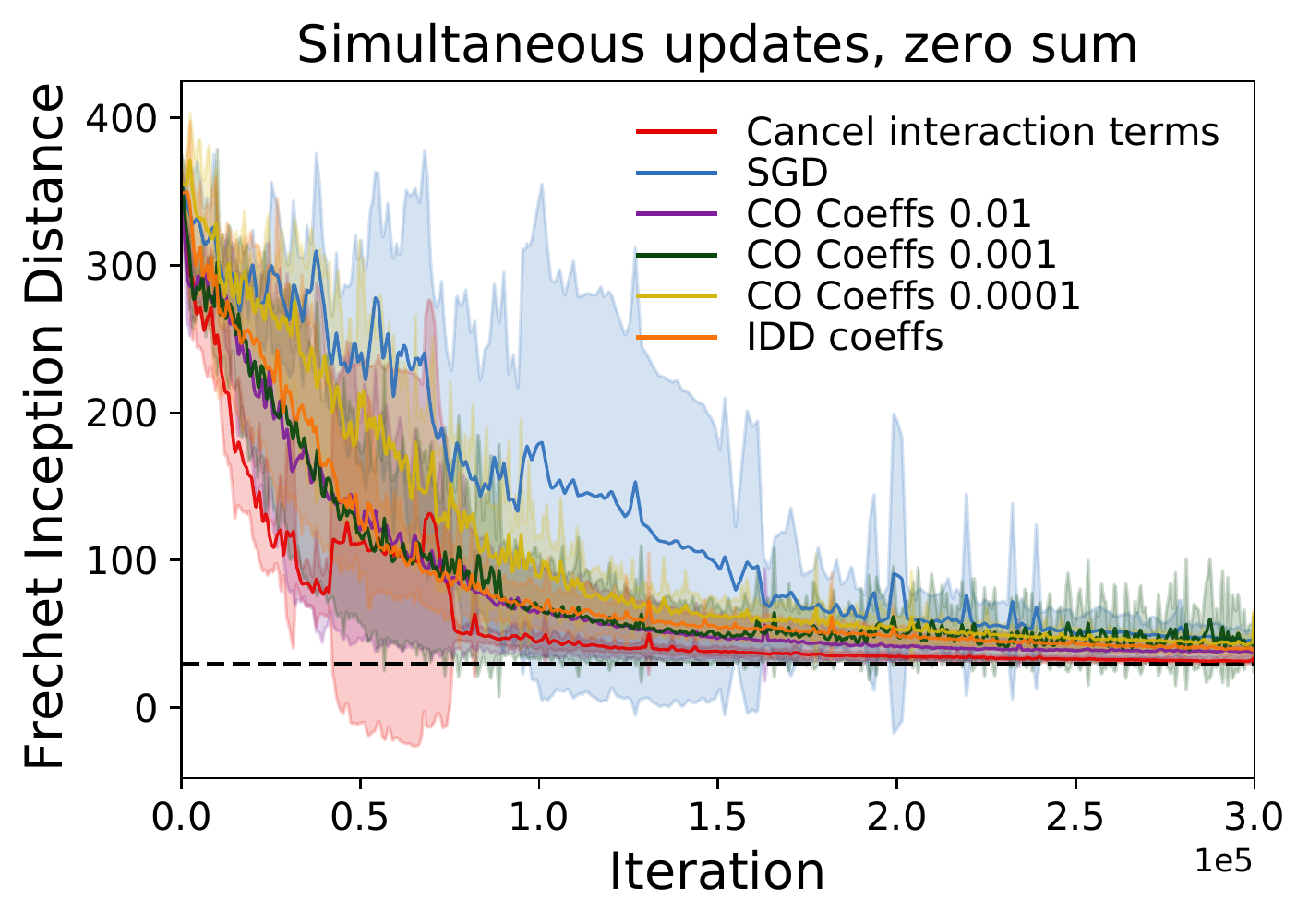}
} \end{subfigure}
  \caption{Comparison with Consensus Optimization (CO). Despite not requiring additional hyperparameters compared to the standard SGD learning rate sweep, canceling the interaction terms of the drift performs better than consensus optimization. Using consensus optimization with a fixed coefficient can perform better than using the drift coefficients when we use them to the strengthen the norms -- this is somewhat expected since while the modified ODEs give us the exact coefficients required to \textit{cancel} the drift, they do not tell us how to strengthen it, and our choice of exact coefficients from the drift might not be optimal.}
  \label{fig:consensus_opt_com}
\end{figure}

\begin{figure}[t]
 \centering
  \begin{subfigure}{
  \includegraphics[width=0.8\columnwidth]{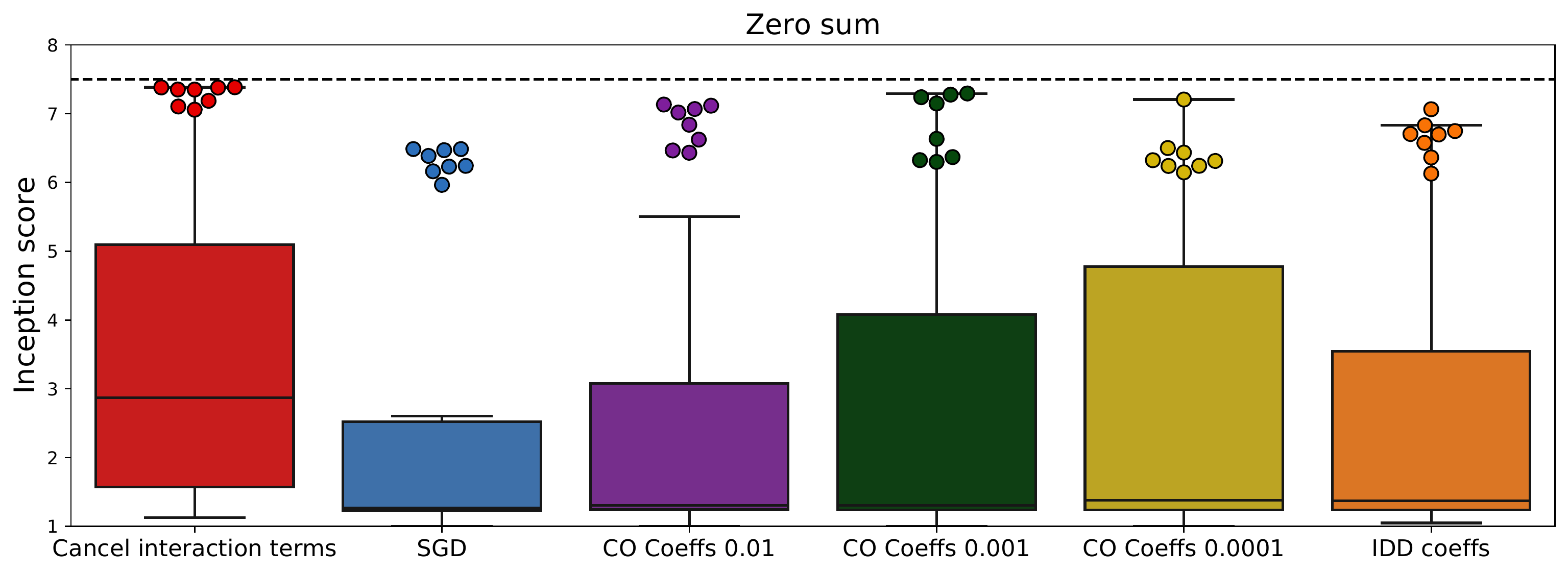}
} \end{subfigure}
  \caption{Comparison with Consensus Optimization (CO), obtained from \textit{all models in the sweep}. Without requiring an additional sweep over the regularization coefficient, canceling the interaction terms results in better performance across the learning rate sweep and less sensitivity to hyperparameters.}
  \label{fig:consensus_opt_box_plots}
\end{figure}

\subsubsection{Variance across seeds}

We have mentioned in the main text the challenge with variance across seeds observed when training GANs with SGD, especially in the case of simultaneous updates in zero sum games. We first notice that performance of simultaneous updates depends strongly on learning rates, with most models not learning. We also notice variance across seeds, both in vanilla SGD and when using explicit regularization to cancel interaction terms. In order to investigate this effect, we ran a sweep of 50 seeds for the best learning rates we obtain when canceling interaction terms in simultaneous updates, namely a discriminator learning rate of $0.01$ and a generator learning rate of $0.005$, we obtain Inception Score results with mean 5.61, but a very large standard deviation of 2.34. Indeed, as shown in Figure~\ref{fig:seed_performance_cancel_interaction_terms}, more than 50\% of the seeds converge to an IS grater than 7.
To investigate the reason for the variability, we repeat the same experiment, but clip the gradient value for each parameter to be in $[-0.1, 0.1]$, and show results in Figure~\ref{fig:seed_performance_cancel_interaction_terms_clip}. We notice that another $10\%$ of jobs converge to an IS score grater than 7, and $10\%$ drop in the number of jobs that do not manage to learn. This makes us postulate that the reason for the variability is due to large gradients, perhaps early in training. We contrast this variability across seeds with the consistent performance we obtain by looking at the best performing \textit{models} across learning rates, where as we have shown in the main paper and throughout the Supplementary Material, we obtain consistent performance which consistently leads to a substantial improvement compared to SGD without explicit regularization, and obtains performance comparable with Adam.

\begin{figure}[t]
 \centering
  \begin{subfigure}[Learning curves.]{
  \includegraphics[width=0.31\columnwidth]{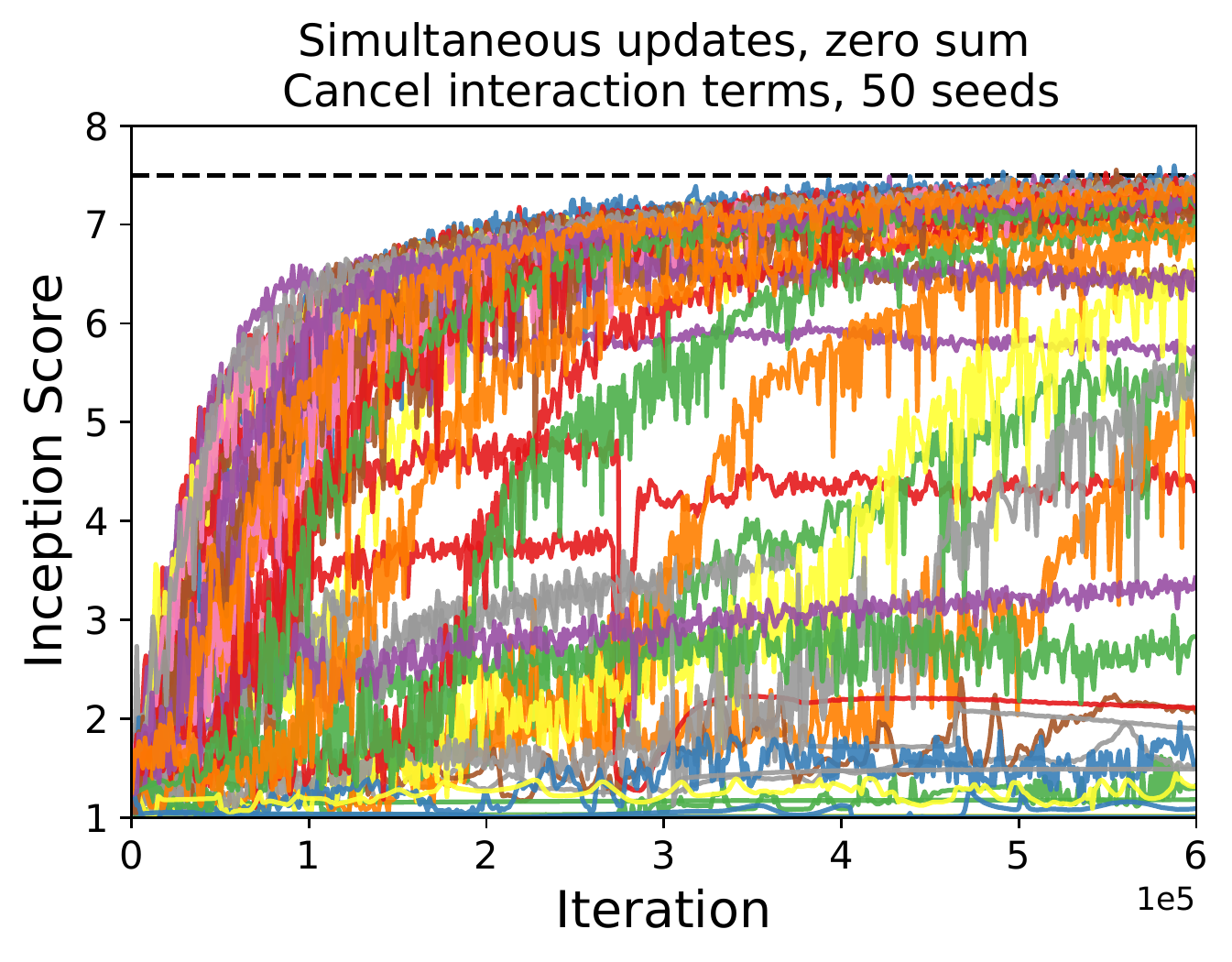}
} \end{subfigure}
 \begin{subfigure}[Performance Histogram.]{
  \includegraphics[width=0.31\columnwidth]{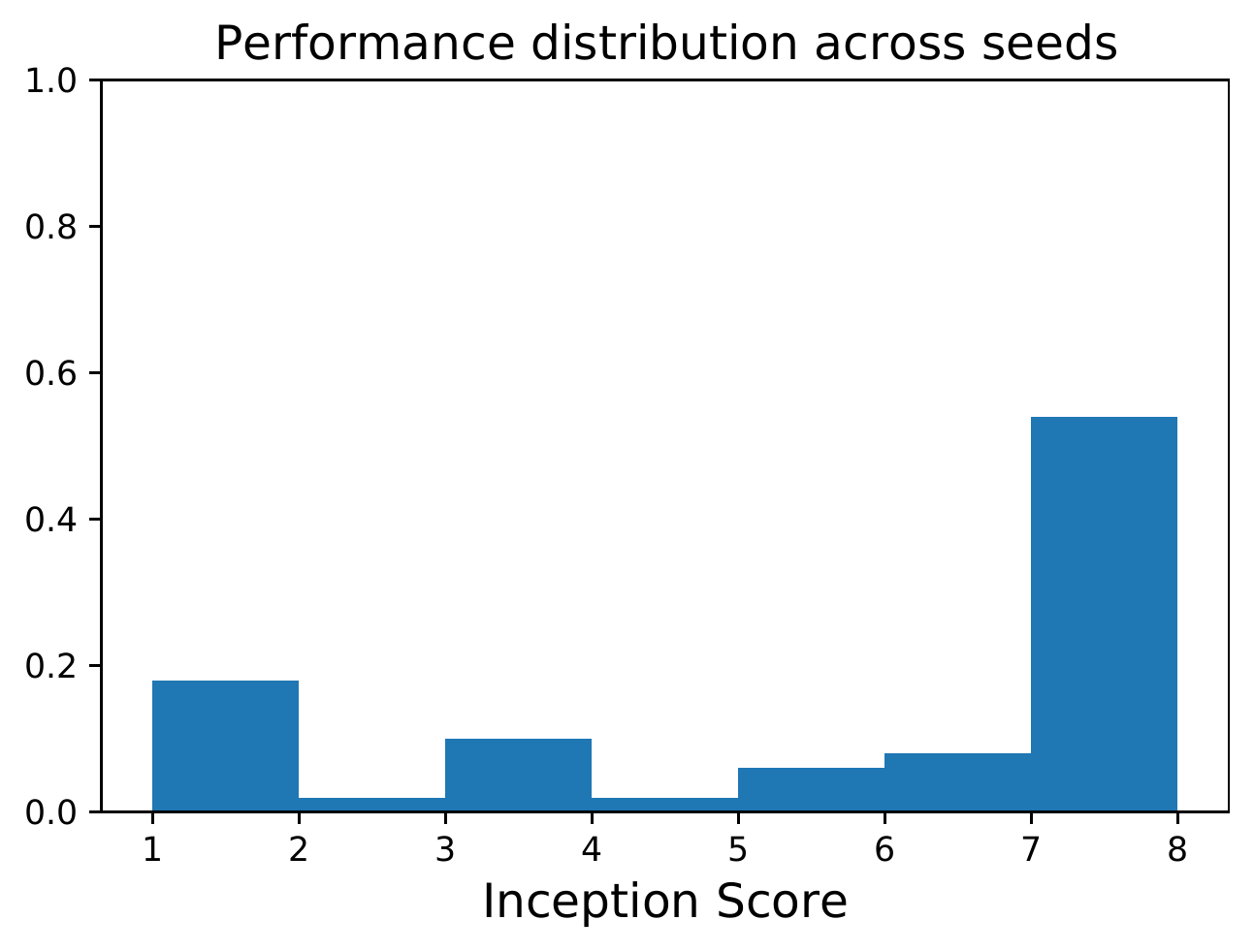}
} \end{subfigure}
  \caption{Variability across seeds for the best performing hyperparameters, when canceling interaction terms in simultaneous updates for the original GAN, with a zero sum loss.}
  \label{fig:seed_performance_cancel_interaction_terms}
\end{figure}

\begin{figure}[t]
 \centering
  \begin{subfigure}[Learning curves.]{
  \includegraphics[width=0.31\columnwidth]{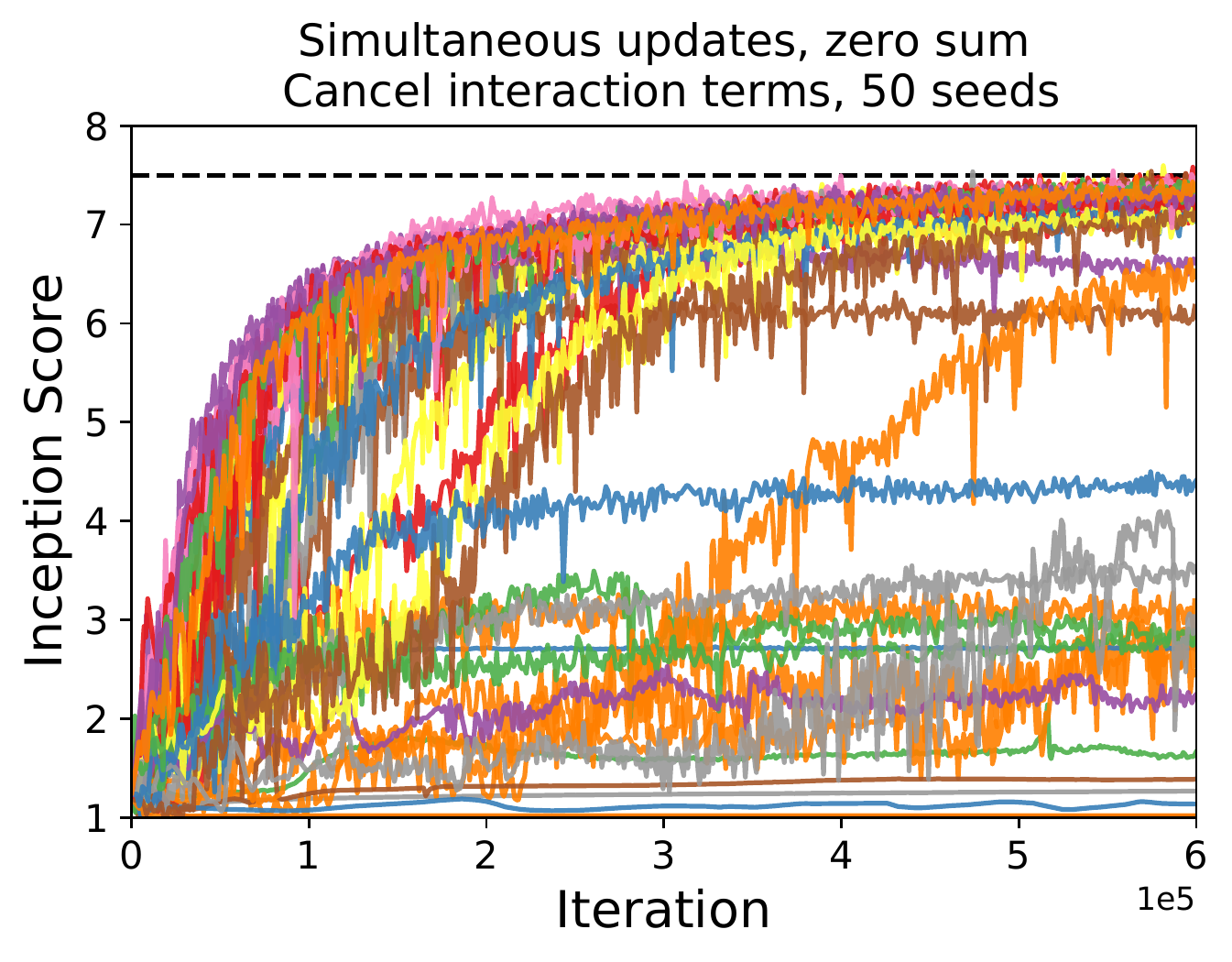}
} \end{subfigure}
 \begin{subfigure}[Performance Histogram.]{
  \includegraphics[width=0.31\columnwidth]{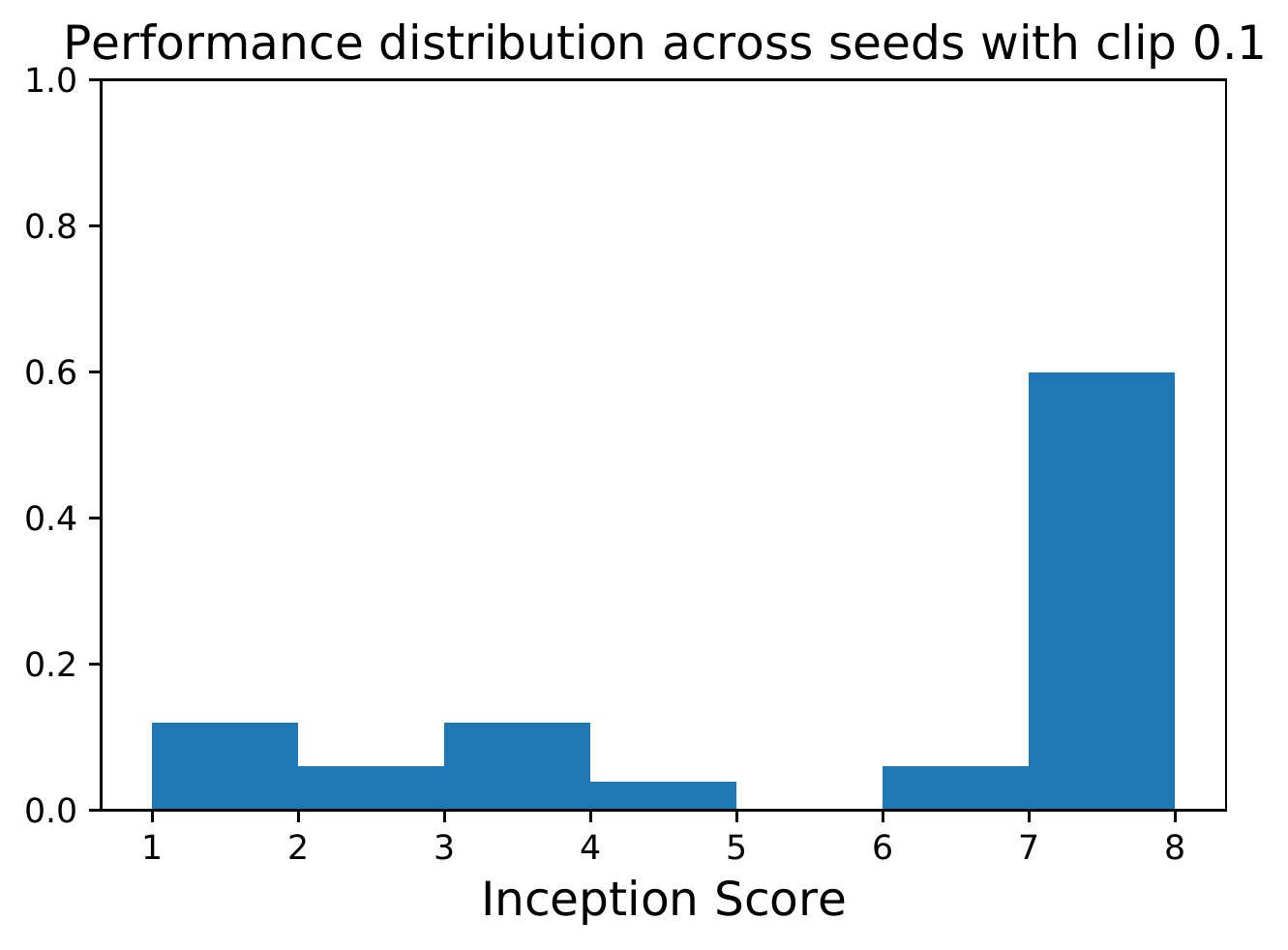}
} \end{subfigure}
  \caption{\textit{With gradient clipping}. Gradient clipping can reduce variability. This suggests that the instabilities observed in gradient descent are caused by large gradient updates.}
  \label{fig:seed_performance_cancel_interaction_terms_clip}
\end{figure}

\subsection{Explicit regularization in zero-sum games trained using alternating gradient descent}

We perform the same experiments as done for simultaneous updates also with alternating updates. To do so, we cancel the effect of DD using the same explicit regularization functional form, but updating the coefficients to be those of alternating updates. We show results in Figure~\ref{fig:sgd_vs_cancel_drift_interaction_alternating}, where we see very little difference in the results compared to vanilla SGD, perhaps apart from less instability early on in training. We postulate that this could be due the effect of DD in alternating updates can can be beneficial, especially for learning rate ratios for which the second player also minimizes the gradient norm of the first player. We additionally show results obtained from strengthening the self terms in Figure~\ref{fig:sgd_vs_cancel_drift_interaction_all_types_reg_alt}.

\subsubsection{More percentiles}

Throughout the main paper, we displayed the best $10\%$ performing models for each optimization algorithm used. We now expand that to show performance results across the $20\%$ and $30\%$ jobs in Figure~\ref{fig:multiple_percentages_sgd_cancel_interaction_terms_alternating}. We observe a consistent increase in performance obtained by canceling the interaction terms.

\begin{figure}[t]
 \centering
  \begin{subfigure}[Inception Score ($\uparrow$).]{
  \includegraphics[width=0.31\columnwidth]{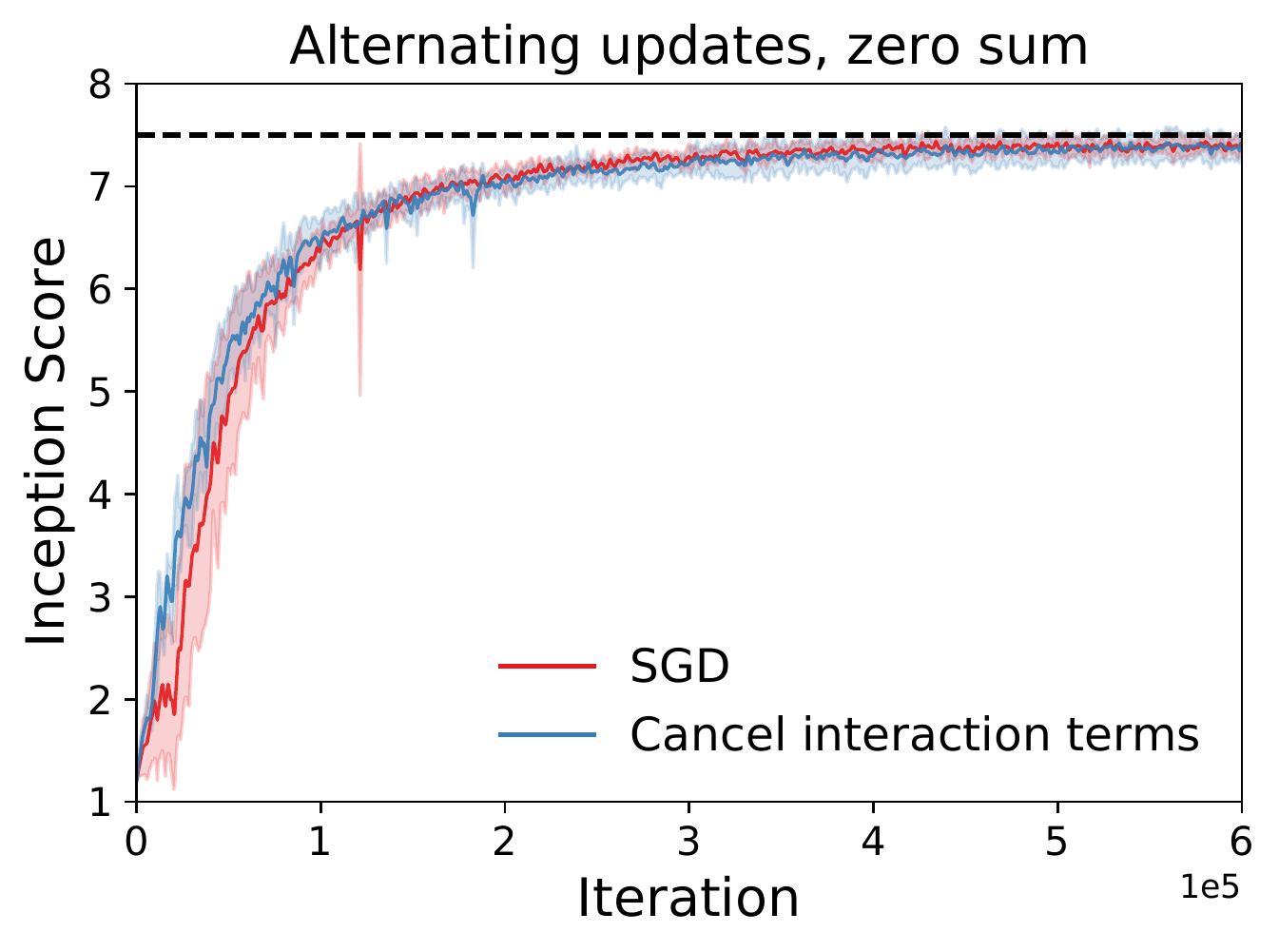}
} \end{subfigure}
 \begin{subfigure}[Frechet Inception Distance ($\downarrow$).]{
  \includegraphics[width=0.31\columnwidth]{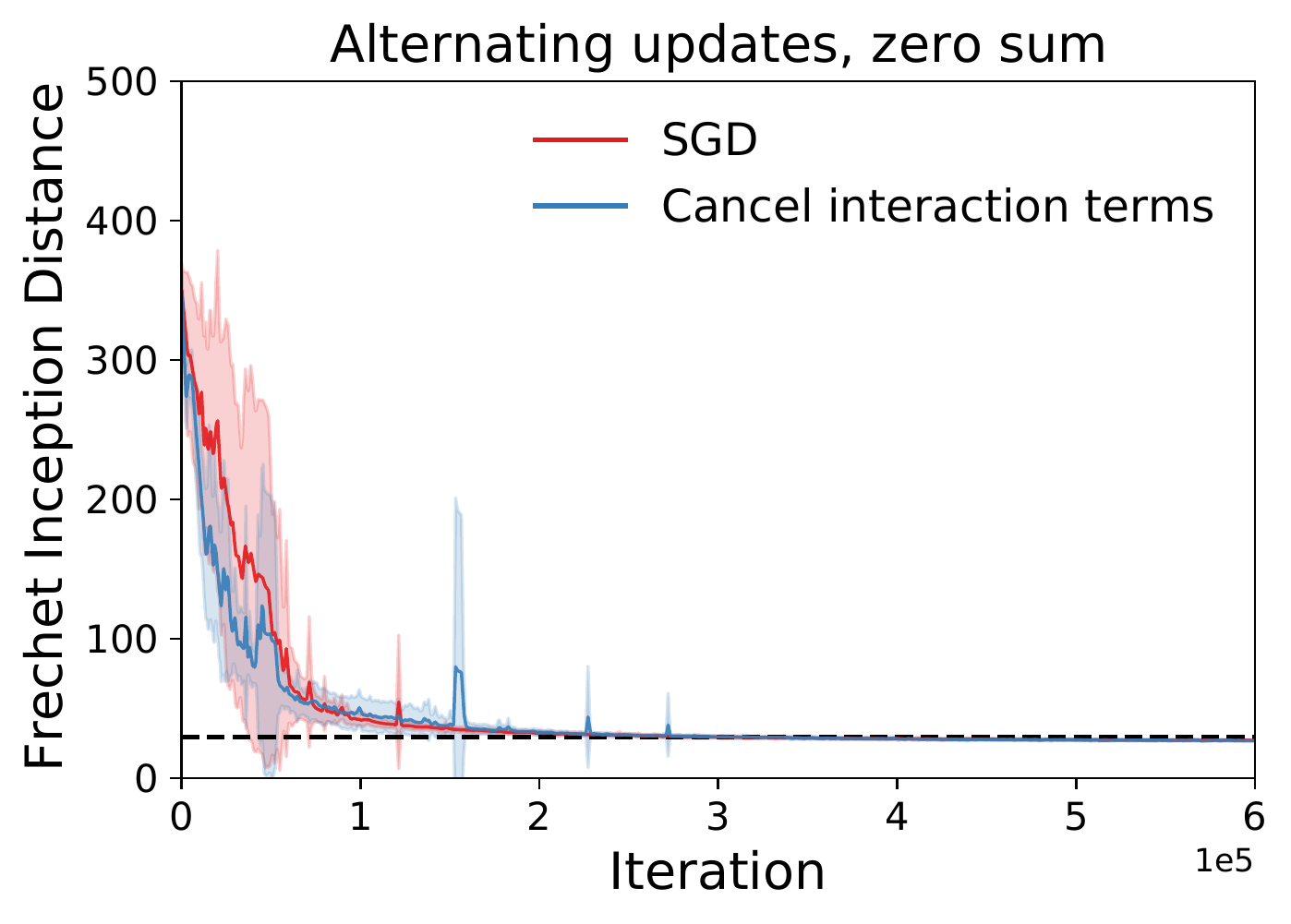}
} \end{subfigure}
  \caption{In alternating updates, using explicit regularization to cancel the effect of the interaction components of drift does not substantially improve performance compared to SGD, but can reduce variance. This is expected, given that the interaction terms for the second player in the case of alternating updates can have a beneficial regularization effect.}
  \label{fig:sgd_vs_cancel_drift_interaction_alternating}
\end{figure}

\begin{figure}[t]
 \centering
  \begin{subfigure}[Inception Score ($\uparrow$).]{
  \includegraphics[width=0.31\columnwidth]{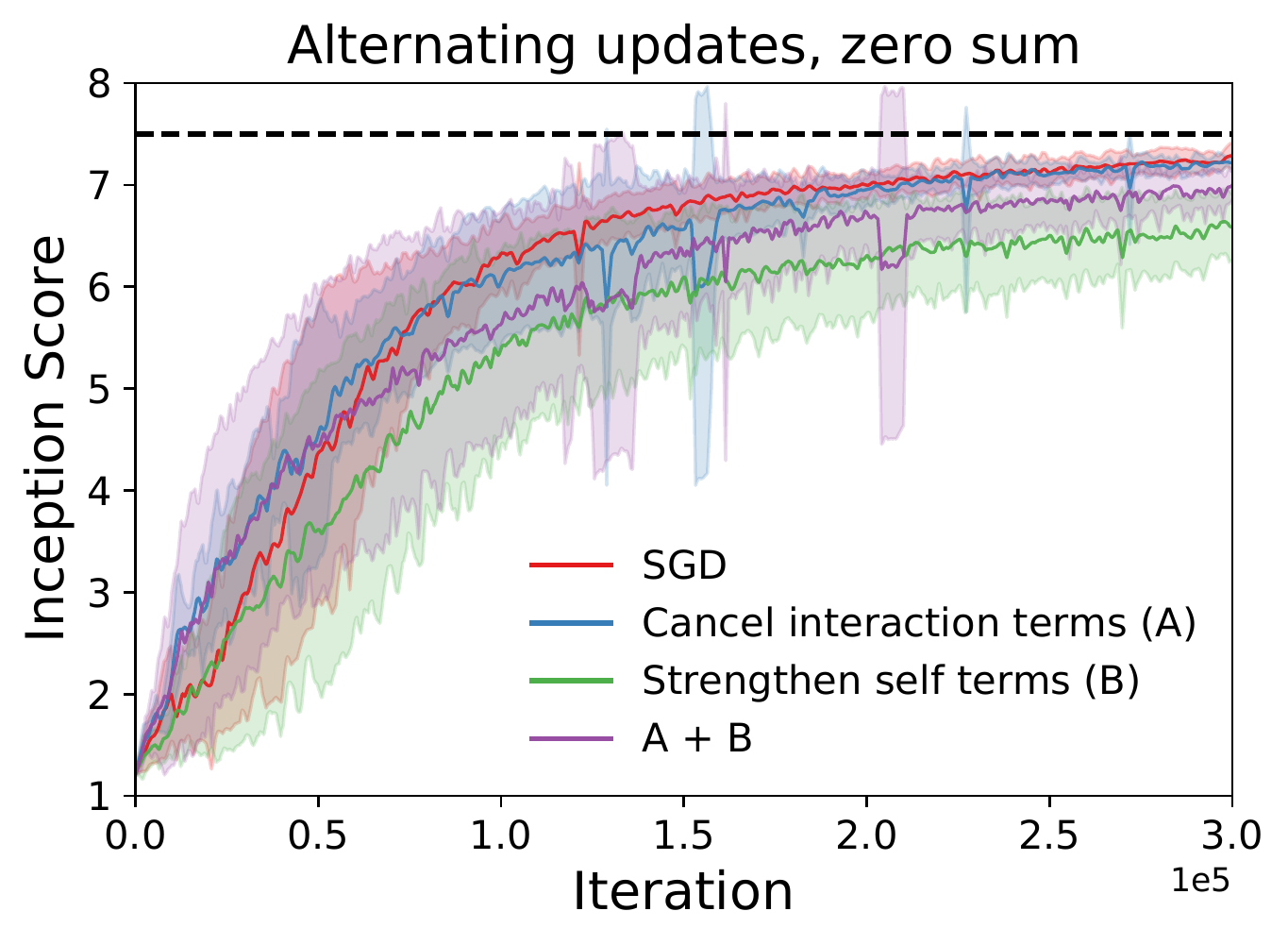}
} \end{subfigure}
 \begin{subfigure}[Frechet Inception Distance ($\downarrow$).]{
  \includegraphics[width=0.31\columnwidth]{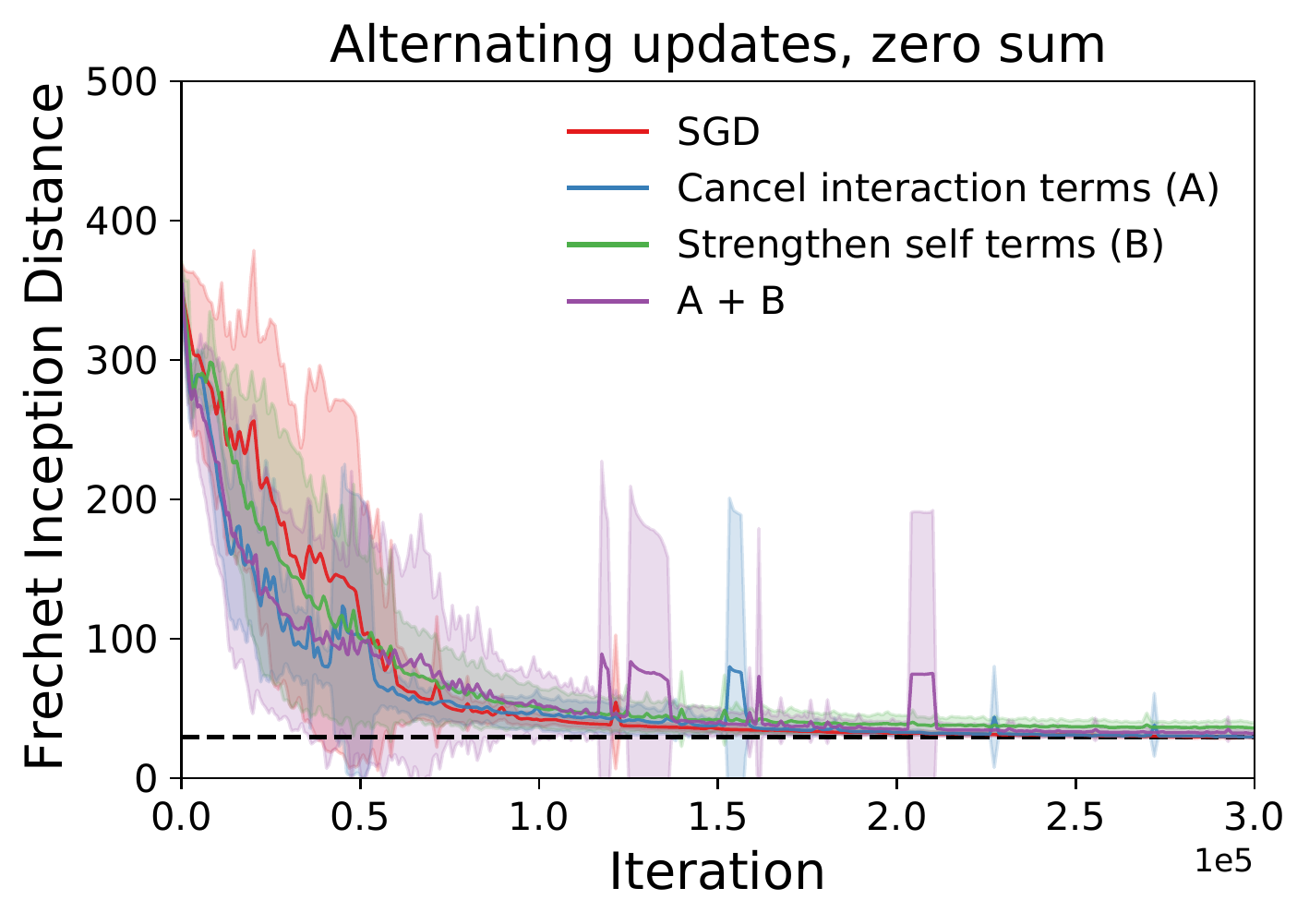}
} \end{subfigure}
  \caption{In alternating updates, using explicit regularization to cancel the effect of the interaction components of drift does not substantially improve performance compared to SGD, but can reduce variance. This is expected, given that the interaction terms for the second player in the case of alternating updates can have a beneficial regularization effect. Strengthening the self terms -- the terms which minimize the player's own norm -- does not lead to a substantial improvement; this is somewhat expected since while the modified ODEs give us the exact coefficients required to \textit{cancel} the drift, they do not tell us how to strengthen it, and our choice of exact coefficients from the drift might not be optimal.}
  \label{fig:sgd_vs_cancel_drift_interaction_all_types_reg_alt}
\end{figure}

\begin{figure}[t]
 \centering
  \begin{subfigure}[Top 10\% models.]{
  \includegraphics[width=0.31\columnwidth]{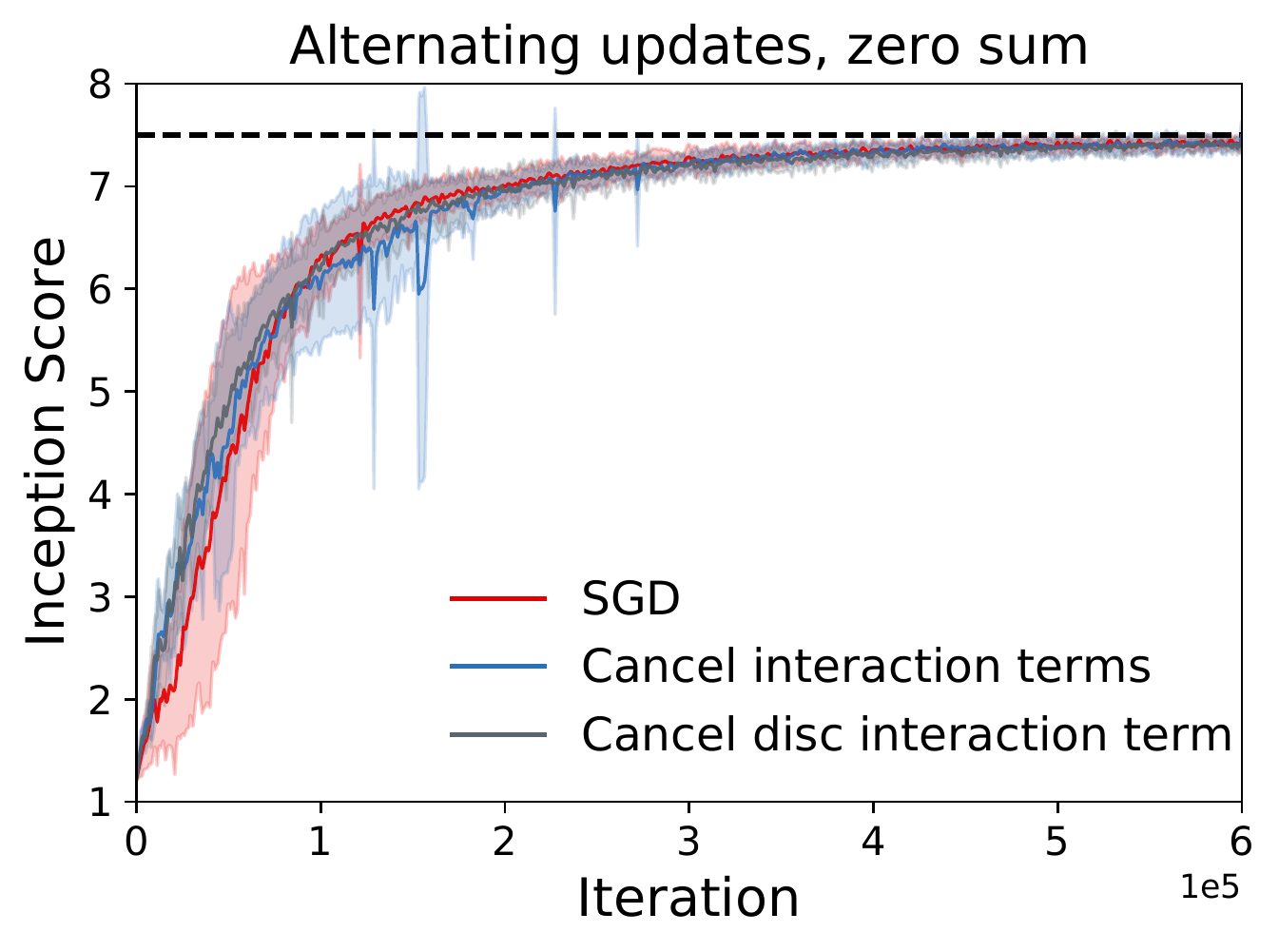}
} \end{subfigure}
 \begin{subfigure}[Top 20\% models.]{
  \includegraphics[width=0.31\columnwidth]{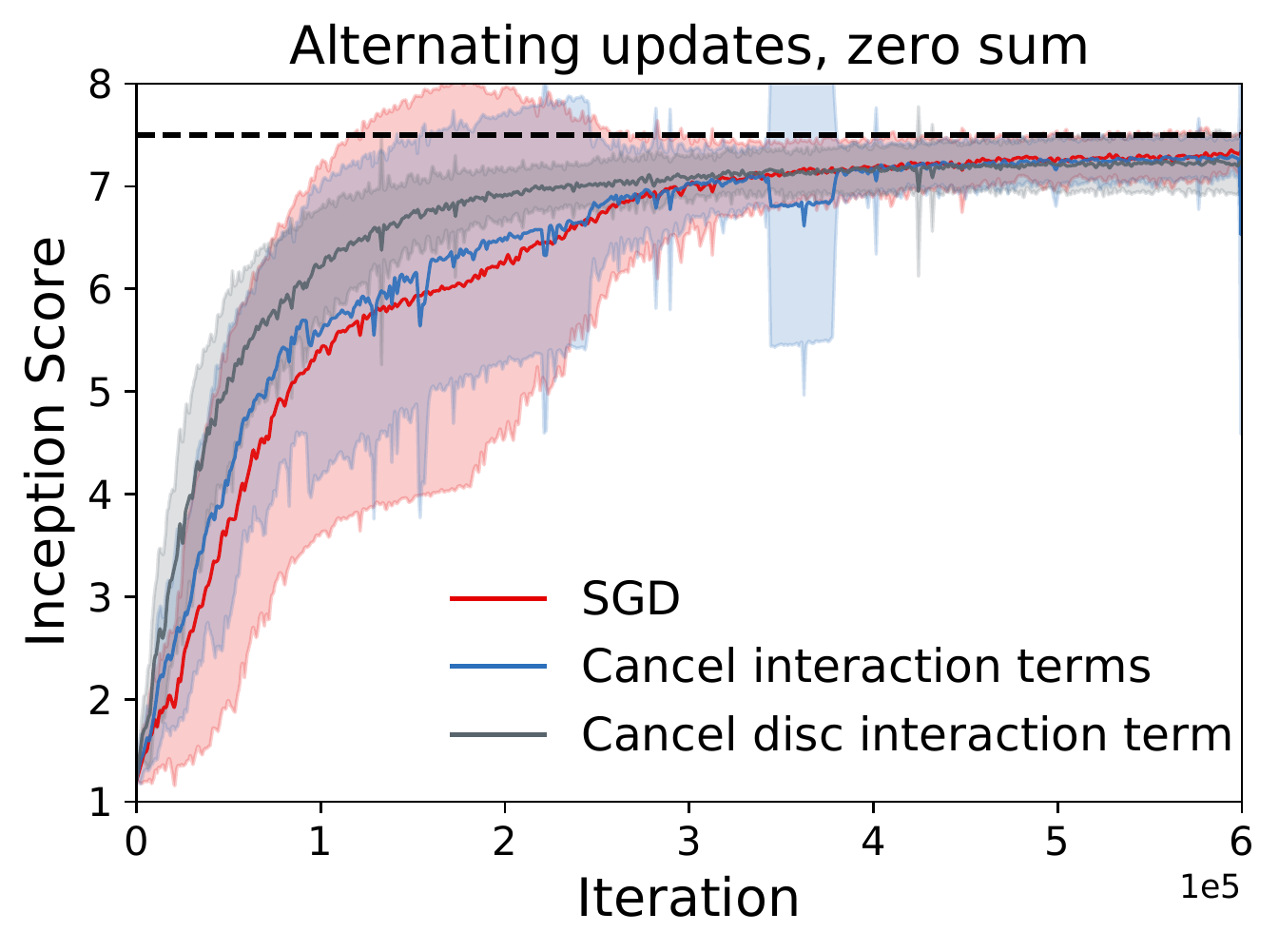}
} \end{subfigure}
\begin{subfigure}[Top 30\% models.]{
  \includegraphics[width=0.31\columnwidth]{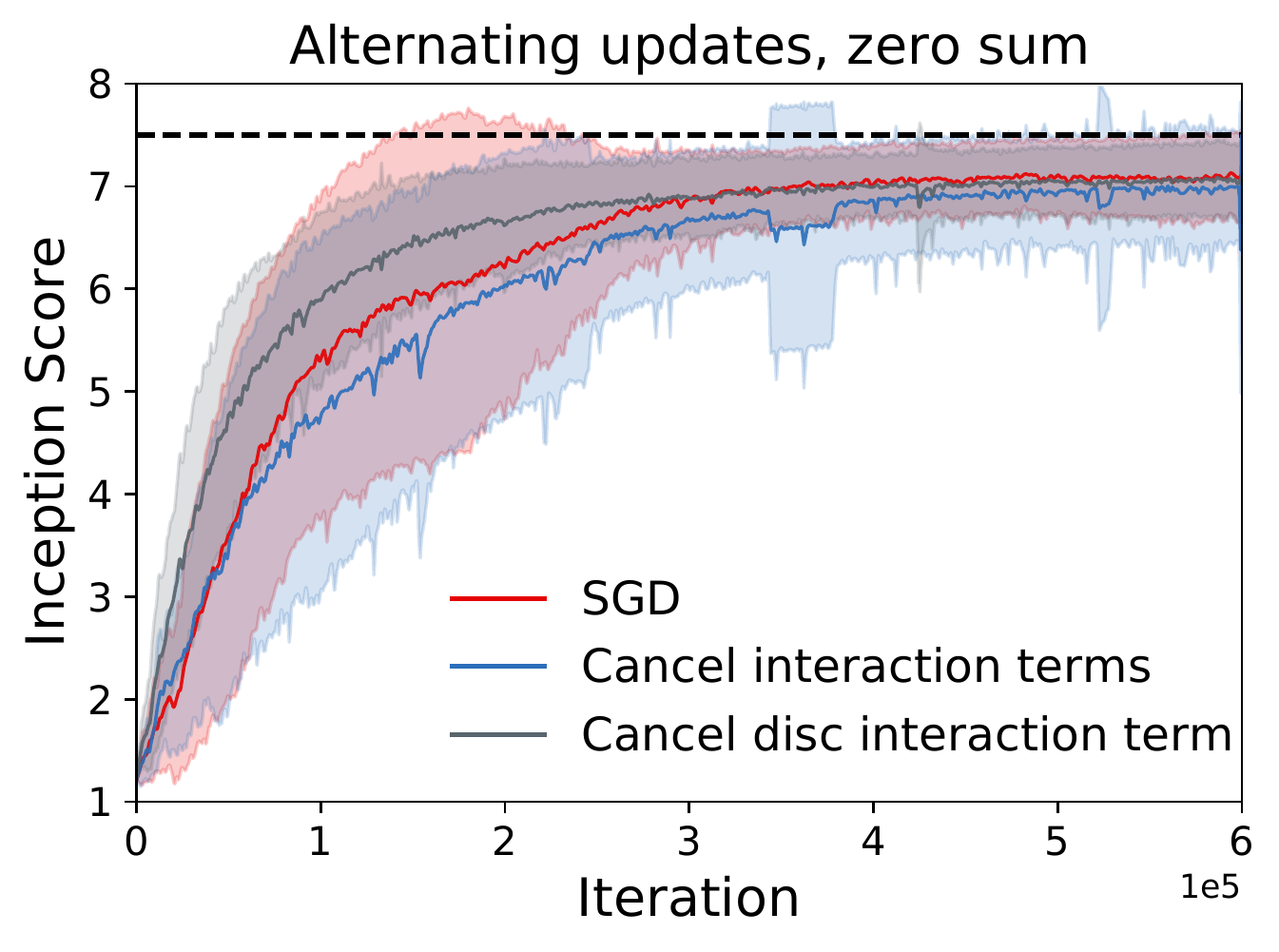}
} \end{subfigure}
  \caption{Performance across the top performing models for vanilla SGD and with canceling the interaction terms and only cancelling the discriminator interaction terms. We notice that canceling only the discriminator interaction terms can result in higher performance across more models, and that the interaction term of the generator can play a positive role, likely due to the smaller strength of the generator interaction term compared to simultaneous updates.}
  \label{fig:multiple_percentages_sgd_cancel_interaction_terms_alternating}
\end{figure}

\clearpage
\section{Experimental details}

\subsection{Classification experiments}

For the MNIST classification results with alternating updates, we use an MLP with layers of size $[100, 100, 100, 10]$ and a learning rate of $0.08$.
The batch size used is $50$ and models are trained for 1000 iterations.
 Error bars are obtained from 5 different seeds.

\subsection{GAN experiments}

\textbf{SGD results}:
All results are obtained using sweep over $\{0.01, 0.005, 0.001, 0.0005\}$ for the discriminator and for the generator learning rates. We restrict the ratio between the two learning rates to be in the interval $[0.1, 10]$ to ensure the validity of our approximations. All experiments use a batch size of 128.

\textbf{Learning rate ratios}: In the learning rate ratios experiments, we control for the number of experiments which have the same learning rate ratio. To do so, we obtain 5 learning rates uniformly sampled from the interval $[0.001, 0.01]$ which we use for the discriminator, we fix the learning rate ratios to be in $\{0.1, 0.2, 0.5, 1., 2., 5.\}$ and we obtain the generator learning rate from the discriminator learning rate and the learning rate ratio.

\textbf{Adam results}:
For Adam, we use the learning rate sweep $\{10^{-4}, 2 \times 10^{-4}, 3 \times 10^{-4}, 4 \times 10^{-4}\}$ for the discriminator and the same for the generator, with $\beta_1 = 0.5$ and $\beta_2 = 0.99$.

\textbf{Explicit regularization coefficients}: For all the experiments where we cancel the interaction terms of the drift, we use the coefficients given by DD. For SGA and consensus optimization, we do a sweep over coefficients in $\{0.01, 0.001, 0.0001\}$.

\textbf{Model Architectures}: All GAN experiments use the CIFAR-10 convolutional SN-GAN architectures -- see Table 3 in Section B.4 in~\citet{miyato2018spectral}.

\textbf{Libraries}: We use JAX~\citep{jax2018github} to implement our models, with Haiku~\citep{haiku2020github} as the neural network library, and Optax~\citep{optax2020github} for optimization.

\textbf{Computer Architectures}: All models are trained on NVIDIA V100 GPUs. Each model is trained on 4 devices.

\subsection{Implementing explicit regularization}

The loss functions we are interested are of the form
\begin{align*}
  L = \mathbb{E}_{p(\vx)} f_{\vtheta}(\vx)
\end{align*}

We then have
\begin{align*}
  \nabla_{\vtheta_i} L &= \mathbb{E}_{p(\vx)}  \nabla_{\vtheta_i} f_{\vtheta}(\vx) \\
\norm{\nabla_{\vtheta} L}^2 &= \sum_{i=1}^{i = |\vtheta|} \left(\nabla_{\vtheta_i} L\right)^2 =  \sum_{i=1}^{i = |\vtheta|} \left(\mathbb{E}_{p(\vx)}  \nabla_{\vtheta_i} f_{\vtheta}(\vx)\right)^2
\end{align*}

to obtain an unbiased estimate of the above using samples, we have that:
\begin{align*}
\norm{\nabla_{\vtheta} L}^2 &=  \sum_{i=1}^{i = |\vtheta|} \left(\mathbb{E}_{p(\vx)}  \nabla_{\vtheta_i} f_{\vtheta}(\vx)\right)^2 \\
& = \sum_{i=1}^{i = |\vtheta|} (\mathbb{E}_{p(\vx)}  \nabla_{\vtheta_i} f_{\vtheta}(\vx)) (\mathbb{E}_{p(\vx)}  \nabla_{\vtheta_i} f_{\vtheta}(\vx)) \\
& \approx \sum_{i=1}^{i = |\vtheta|} \left(\frac{1}{N} \sum_{k=1}^{N} \nabla_{\vtheta_i} f_{\vtheta}(\widehat{\vx_{1,k}})\right) \left(\frac{1}{N} \sum_{j=1}^{N}\nabla_{\vtheta_i} f_{\vtheta}(\widehat{\vx_{2,j}})\right)
\end{align*}

so we have to use two sets of samples $\vx_{1,k} \sim p(\vx)$ and $\vx_{2,j} \sim p(\vx)$ from the true distribution (by splitting the batch into two or using a separate batch) to obtain the correct norm. To compute an estimator for $\nabla_{\vphi} \norm{\nabla_{\vtheta} L}^2$, we can compute the gradient of the above unbiased estimator of $\norm{\nabla_{\vtheta} L}^2$. However, to avoid computing gradients for two sets of samples, we derive another unbiased gradient estimator, which we use in all our experiments:
\begin{align}
\frac{2}{N} \sum_{i=1}^{i = |\vtheta|} \sum_{k=1}^N \nabla_{\vphi} \nabla_{\vtheta_i} f_{\vtheta}(\widehat{\vx_{1, j}}) \nabla_{\vtheta_i} f_{\vtheta}(\widehat{\vx_{2, j}})
\label{eq:unbiased_grads_one_backprop_N_terms}
\end{align}

\end{document}